%% file: main.tex
\documentclass{article}

\PassOptionsToPackage{numbers, compress}{natbib}

\usepackage[final,main]{neurips_2026}

\usepackage{amsmath}
\usepackage{mathtools}
\usepackage{amsthm,enumitem}
\usepackage{amssymb}  
\usepackage[utf8]{inputenc} 
\usepackage[T1]{fontenc}    
\usepackage{hyperref}       
\hypersetup{hidelinks,hypertexnames=false}
\usepackage{url}            
\usepackage{booktabs}       
\usepackage{amsfonts}       
\usepackage{nicefrac}       
\usepackage{microtype}      
\usepackage{xcolor}         
\usepackage{graphicx}
\usepackage{multirow}
\usepackage{pifont}
\usepackage{wrapfig}
\usepackage[dvipsnames]{xcolor}

\usepackage{algorithm}
\usepackage{algorithmic}

\newtheorem{remark}{Remark}
\newtheorem{proposition}{Proposition}
\newtheorem{theorem}{Theorem}
\newtheorem{lemma}{Lemma}

\newtheorem{corollary}{Corollary}
\newtheorem{definition}{Definition}
\providecommand{\bbm}[1]{\mathbf{#1}}

\newcommand{\E}{\mathbb{E}}  

\newcommand{\R}{\mathbb{R}}

\newcommand{\cI}{\mathcal{I}} 
\newcommand{\Tset}{\mathsf{T}} 
\newcommand{\Fset}{\mathsf{F}} 
\newcommand{\phase}{\mathsf{phase}}
\newcommand{\class}{\mathsf{cls}}
\newcommand{\tauF}{\tau_{\Fset}}
\newcommand{\epsb}{\varepsilon_{\mathrm{b}}}
\newcommand{\Pihat}{\widehat{\Pi}}
\newcommand{\bhat}{\widehat{b}}

\title{Persistent-Transient Policy Evaluation for Markov Chains via Minimal Peripheral Quotients}

\author{%
  Yang Xu \\
  Purdue University, USA\\
  \texttt{xu1720@purdue.edu} \\
  \And
  Vaneet Aggarwal \\
  Purdue University, USA \\
  \texttt{vaneet@purdue.edu} \\
}

\begin{document}

\maketitle
\raggedbottom

\begin{abstract}
We study fixed-policy evaluation for finite Markov chains that may be reducible and periodic. Classical evaluation methods with gain and bias decomposition are not always diagnostic: the gain records only invariant Ces\`aro averages, while persistent phase-dependent behavior is absorbed into the bias together with genuinely transient effects. We identify the real peripheral invariant subspace $\mathcal{K}(P)$ of the transition matrix $P$ as the source of this ambiguity. Quotienting by $\mathcal{K}(P)$ is the minimal exact quotient that removes all non-decaying modes and makes the remaining dynamics strictly stable. After choosing a gauge projection $\Pi$ with kernel $\mathcal{K}(P)$, the reward admits a unique decomposition $r = g_\Pi^\star + (I-P)v_\Pi^\star$, where $g_\Pi^\star$ is a persistent regime profile and $v_\Pi^\star$ is a gauge-fixed transient component. An exact comparison with classical normalized gain and bias shows that the new pair reallocates the same information into a persistent profile and a gauge-fixed representative of the transient quotient. This decomposition reconstructs finite-horizon returns, recovers statewise average reward, admits a transient-cost interpretation, and yields a stable estimator under a generative model.
\end{abstract}

\section{Introduction}

Average-reward dynamic programming and reinforcement learning evaluate sequential decisions without discounting, and classical theory uses Poisson equations, gain vectors, and bias functions to describe long-run performance \cite{puterman2014markov,bertsekas2012dynamic,meyer1975role,schweitzer1978foolproof}. Recent work has extended planning and learning guarantees beyond the unichain setting and into weakly communicating or multichain regimes \cite{zurek2024span,lee2025optimal,zurekchen2025faster}. We focus on the evaluation problem after a stationary policy has already been fixed. The resulting object is a finite-state Markov reward process with transition matrix $P$ and reward vector $r$. The central question is: what should summarize the performance of this policy from each initial state when the induced chain may be reducible and periodic?

A useful answer should separate three roles. It should recover the statewise long-term average reward, since different initial states may reach different recurrent classes. It should then describe the persistent regime that the chain eventually exhibits, including cyclic phase structure that repeats forever. It should also describe the transient cost paid before that persistent regime is reached. These roles are distinct in simple examples: a periodic chain can repeat a nonconstant reward pattern forever, a multichain model requires statewise evaluation because recurrent outcomes depend on the start state, and a gridworld random walk can make two policies with the same eventual average behave very differently over finite horizons. Policy evaluation beyond ergodicity should therefore distinguish what persists forever from what is paid only on the way there.

This distinction matters in computing systems with engineered cycles. RotorNet uses a fixed cyclic optical-connectivity schedule; its hardware implementation reports recurring transmission-error bursts at reconfiguration boundaries and rotor defects. In one TCP test, a defect occupying $2\%$ of a usable slot reduced throughput by $11\%$ relative to ideal service \cite{rotornet2024}. These recurring losses call for schedule, guard-interval, or hardware changes, whereas a one-time initialization backlog calls for a different intervention. Opus similarly reconfigures optical rails at communication-phase boundaries within distributed ML training iterations \cite{opus2026}, connecting phase-aware accounting to current large-scale ML infrastructure. An average-throughput summary cannot by itself distinguish recurring service limitations, initial phase alignment, and temporary startup cost. A finite model can retain queue or protocol state together with schedule phase: deterministic phase advancement yields a periodic Markov chain even when arrivals and service are stochastic. Distinct closed operating configurations can additionally make outcomes depend on initialization. These examples motivate our evaluation question; they are not deployments of our estimator. Appendix~\ref{app:cyclic_interpretation} gives a concrete model and explains the scope of exact versus approximate periodicity.

The classical language for this problem for the Markov chain $P$ with reward function $r$ is the gain-bias Poisson equation in the form of $r = \rho + (I-P)h$, where $\rho$ is the gain vector and $h$ is the bias function \cite{puterman2014markov,bertsekas2012dynamic,meyer1975role}. This formalism is already valid for fixed-policy evaluation in general finite chains, and recent multichain fixed-point methods also build on gain-bias structure \cite{zurekchen2025faster}. However, the gain vector stores only invariant long-run information. The classical gain is the Ces\`aro average $\rho=P^\infty r$, and therefore satisfies $P\rho=\rho$. On each closed irreducible recurrent class, every such invariant vector is constant, so the classical gain collapses cyclic phases to their time average. It can record the average of a periodic reward pattern, but not the phase-resolved pattern itself. Persistent behavior that changes from step to step, such as an alternating reward on a cycle, lasts forever but is not invariant. The only remaining place for this behavior in the classical pair is the bias $h$.

This creates the entanglement issue. In an ideal diagnostic decomposition, the long-run component would contain all behavior that persists, while the transient component would contain only effects that fade away. Classical gain and bias do not satisfy this principle once periodic persistent structure appears: the invariant gain keeps only the Ces\`aro part, and the non-invariant persistent remainder is absorbed into the bias. The resulting issue is not that returns cannot be evaluated, but that the classical decomposition does not tell us cleanly what caused them. The practical issue is therefore not loss of evaluability but loss of diagnosability. A practitioner needs to know whether poor finite-horizon behavior is caused by the eventual regime or by the path to that regime, because these imply different interventions. A deterministic cycle makes the point sharp. In such a chain, the process starts inside its recurrent regime and cycles forever. Nevertheless, the invariant gain collapses the phase pattern into an average, so the remaining oscillation must be represented in the bias. Thus a large classical bias can arise even when there is no genuine transient burden to diagnose. 

Our answer is to organize evaluation by dynamical persistence. Let $\mathcal{K}(P)$ be the real peripheral invariant subspace of $P$, containing exactly the non-decaying modes of the chain, which includes the invariant recurrent-class directions as well as periodic phase directions. The quotient by $\mathcal{K}(P)$ is not a design choice. Dynamically, it is the smallest invariant subspace that must be removed before strict contraction is possible. Semantically, it is exactly the collection of modes whose effects remain visible at arbitrarily long horizons and therefore should not be called transient. We prove this minimality statement and then use the quotient to define a unique decomposition $r = g_\Pi^\star + (I-P)v_\Pi^\star$ after a gauge projection $\Pi$ with $\ker(\Pi)=\mathcal{K}(P)$ is fixed. The vector $g_\Pi^\star$ is the persistent regime profile: its invariant part recovers the classical statewise average reward, while its orbit $\{P^t g_\Pi^\star\}_{t\ge 0}$ records recurrent-class and phase behavior that the invariant gain discards. The vector $v_\Pi^\star$ represents the quotient in which all non-decaying modes have been removed. Under the anchor gauge used in this paper, it has an explicit transient-to-regime cost interpretation. A key feature of this viewpoint is that it reorganizes the same Poisson information as classical approaches into two parts with different meanings. We prove that the invariant part of $g_\Pi^\star$ is exactly the classical gain, and that the classical normalized bias equals the new transient component plus a unique non-invariant peripheral correction. Thus the decomposition makes persistent directions explicit while representing the transient quotient in a chosen gauge. In aperiodic multichain models, where no non-invariant peripheral modes exist, the distinction disappears and the framework reduces to the classical picture up to gauge. The structural viewpoint also has an algorithmic payoff. Because the peripheral quotient restores contraction, it gives a stable target for sample-based estimation under a generative model. Our estimator learns recurrent classes, cyclic phases, anchors, and phase-offset absorption weights; constructs an anchor gauge; solves the projected quotient fixed-point equation for $v_\Pi^\star$; and reconstructs $g_\Pi^\star$ from anchor residuals. The output recovers average reward, exposes persistent regime behavior, and quantifies transient-to-regime cost. In this sense, the contribution is not a new way to compute returns. It is a new way to attribute them. Notably, a detailed comparison of our work with existing literature is given in Appendix~\ref{app:related_work}.

\paragraph{Contributions.}
We introduce a policy-evaluation framework for general finite Markov chains that removes the classical entanglement between persistent non-invariant behavior and transient effects. 
\begin{itemize}[leftmargin=*,itemsep=1pt,topsep=2pt,parsep=0pt]
    \item We identify $\mathcal{K}(P)$ as the minimal real $P$-invariant subspace whose quotient dynamics are strictly contractive.
    \item We define the gauge-fixed decomposition $r=g_\Pi^\star+(I-P)v_\Pi^\star$ into a persistent regime profile and a transient component.
    \item We prove an exact comparison with classical normalized gain and bias, showing that the pair reorganizes rather than augments Poisson information.
    \item We show that the decomposition reconstructs finite-horizon returns, recovers long-run average reward, yields an anchor-gauge transient-cost interpretation, and can be estimated from generative-model samples.
\end{itemize}

\section{Formulation}\label{sec:formulation}

We fix a stationary policy in a finite Markov decision process and evaluate the induced Markov reward process. The state space is $\mathcal S=\{1,2,\dots,n\}$, the transition matrix is $P(s'| s)=\sum_{a\in\mathcal A}\pi(a| s)\mathcal P(s'| s,a)$, and the one step reward is $r(s)=\sum_{a\in\mathcal A}\pi(a| s)r_{\mathrm{mdp}}(s,a)$. Thus every object below depends only on the pair $(P,r)$. For an initial state $s$ and a horizon $T\ge 1$, let $J_T(s)=\mathbb E_s[\sum_{t=0}^{T-1}r(X_t)]$, or in vector form $J_T=\sum_{t=0}^{T-1}P^t r$. The statewise long run average reward profile is
\begin{equation}\label{eq:gdef}
g(s) \coloneqq \lim_{T\to\infty}\frac{1}{T}J_T(s),
\end{equation}
which exists for every finite Markov chain \cite{puterman2014markov,bertsekas2012dynamic}. The question is how to represent the same fixed policy in a way that distinguishes persistent regime behavior from transient cost. This distinction is about representation rather than about existence of returns. The sequence $J_T$ is always defined, and the average reward profile $g$ is always defined in the finite state setting. What is not automatic is a decomposition whose two pieces answer two different diagnostic questions: what regime behavior remains visible at long horizons, and what finite accumulated effect is caused by the path before that regime is reached.

\subsection{Classical normalized gain and bias}

The classical representation uses the Ces\`aro limiting projection $P^\infty=\lim_{T\to\infty}T^{-1}\sum_{t=0}^{T-1}P^t$, whose existence is standard for finite Markov chains \cite{meyer1975role,puterman2014markov}, and the gain $\rho=P^\infty r$. A normalized bias $h$ is any solution of
\begin{equation}\label{eq:poisson}
r=\rho+(I-P)h,\qquad P^\infty h=0.
\end{equation}
This pair is well defined in finite state models and is the standard language of average reward evaluation. Its limitation is semantic rather than algebraic. The gain $\rho$ is invariant, so it can only record quantities that remain unchanged after one transition. On a periodic recurrent class, however, the policy may exhibit a reward pattern that persists forever while changing with phase. Such behavior is not transient, but it is also not invariant. The next standard lemma explains why the invariant gain cannot store it \cite{meyer1975role,puterman2014markov}.

\begin{lemma}\label{lem:invariant_constant_on_class}
Let $F$ be a finite closed irreducible class of $P$. If $x$ satisfies $Px=x$ on $F$, then $x$ is constant on $F$.
\end{lemma}

Lemma~\ref{lem:invariant_constant_on_class} explains why the classical gain collapses every periodic recurrent class to a single number. A nonconstant reward pattern on a deterministic cycle is not transient, but it cannot appear in $\rho$. It must therefore be stored in $h$ together with the cost of reaching the recurrent regime. This is the entanglement that the rest of the paper removes. In a unichain aperiodic model this issue is hidden, because the only persistent directions are invariant. When a recurrent class is periodic, however, the set of persistent directions is larger than the invariant subspace. The goal is therefore to keep the classical average reward as one output, but to refine it with the additional persistent coordinates that are visible along the actual trajectory.

\section{Minimal peripheral quotient and gauge-fixed decomposition}\label{Quotient}

This section identifies the part of the chain that should be removed before a transient fixed point is solved. The issue is not only how to stabilize an iteration. A direction that does not decay under repeated application of $P$ should not be interpreted as a finite transient correction, because its effect remains visible at arbitrarily long horizons. The persistent subspace below collects exactly these directions. Once it is quotiented out, the remaining dynamics are strictly contractive and can support a unique transient representative.

\subsection{The peripheral invariant subspace}

Let $P$ act on $\mathbb C^n$ and let $\mathcal K_{\mathbb C}(P)$ be the generalized invariant subspace spanned by eigenvalues with modulus one. The real peripheral invariant subspace is $\mathcal K(P)\coloneqq\mathcal K_{\mathbb C}(P)\cap\mathbb R^n$. For a finite Markov chain, this space consists exactly of the modes that do not decay under repeated application of $P$. It includes the invariant recurrent class directions and, when a class is periodic, the phase directions that oscillate forever. This is the subspace that should be treated as persistent.

\subsection{The minimal contractive quotient}

For any real $P$ invariant subspace $K$, write $\bar P_K$ for the induced operator on $\mathbb R^n/K$. We call $K$ an exact stabilizing quotient if $\rho(\bar P_K)<1$, where $\rho(\cdot)$ denotes spectral radius.

\begin{theorem}[Minimal peripheral quotient]\label{thm:minimal_quotient}
For a real $P$ invariant subspace $K$, the following are equivalent: $\rho(\bar P_K)<1$; $\mathcal K(P)\subseteq K$; and every non decaying mode of $P$ is zero in $\mathbb R^n/K$. Hence $\mathcal K(P)$ is the unique minimal real invariant subspace whose quotient dynamics are strictly stable.
\end{theorem}

The theorem has two consequences. First, it rules out smaller quotients. If a phase direction or recurrent class direction survives in the quotient, then the quotient still contains a mode that does not decay, so no strict contraction can hold. Second, it rules out unnecessary larger quotients. Removing more than $\mathcal K(P)$ would also remove directions that are transient in the dynamical sense. Thus the same subspace is forced by both sides of the paper: it is the smallest object needed for stable fixed point analysis, and it is exactly the object that should be treated as persistent in evaluation.

\begin{theorem}[Quotient contraction]\label{thm:quotient_contraction}
On $\mathbb R^n/\mathcal K(P)$ there exists a norm $\|\cdot\|_{\mathrm q}$ and a constant $\gamma<1$ such that $\|\bar P x\|_{\mathrm q}\le \gamma\|x\|_{\mathrm q}$ for every quotient class $x$. Equivalently, the pullback semi-norm on $\mathbb R^n$ satisfies $\|Pv\|_{\mathrm q}\le \gamma\|v\|_{\mathrm q}$ for every $v\in\mathbb R^n$.
\end{theorem}

This is the computational counterpart of Theorem~\ref{thm:minimal_quotient}. After quotienting by $\mathcal K(P)$, the Bellman type map no longer has unit modulus directions left to obstruct contraction. The norm in Theorem~\ref{thm:quotient_contraction} is not meant to be a new evaluation metric for users. It is a proof device showing that the transient part is a stable fixed point problem once all persistent directions have been removed. The proofs of Theorems~\ref{thm:minimal_quotient} and~\ref{thm:quotient_contraction} are given in Appendices~\ref{proof_minimal_quotient} and~\ref{proof_quotient_contraction}.

\subsection{Gauge fixed decomposition}

The quotient fixed point identifies only an equivalence class modulo $\mathcal K(P)$. This is the analogue of fixing an additive constant in classical average reward theory, except that the nullspace is now the whole persistent subspace rather than only constants. To obtain a concrete vector with evaluation meaning, choose a projection $\Pi$ with $\ker(\Pi)=\mathcal K(P)$ and let $\operatorname{range}(\Pi)$ be the gauge slice. Define the projected Bellman map
\begin{equation}\label{eq:projected_operator}
\mathcal T_\Pi(v)\coloneqq \Pi(r+Pv).
\end{equation}

\begin{theorem}[Gauge fixed persistent transient decomposition]\label{thm:wellposed}
For every reward $r$ and every projection $\Pi$ with kernel $\mathcal K(P)$, there is a unique $v_\Pi^\star\in\operatorname{range}(\Pi)$ satisfying $v_\Pi^\star=\mathcal T_\Pi(v_\Pi^\star)$. With $g_\Pi^\star\coloneqq r-(I-P)v_\Pi^\star$, we have
\begin{align}
r&=g_\Pi^\star+(I-P)v_\Pi^\star,\quad g_\Pi^\star\in\mathcal K(P),\quad v_\Pi^\star\in\operatorname{range}(\Pi),\label{eq:decomposition}\\
v_\Pi^\star&=\Pi(r+Pv_\Pi^\star),\label{eq:proj_fixed_point}\\
g_\Pi^\star&=r-(I-P)v_\Pi^\star\in\mathcal K(P),\label{eq:gstar}\\
[v_\Pi^\star]&=[r]+\bar P[v_\Pi^\star].\label{eq:quotient_fixed_point}
\end{align}
\end{theorem}

The theorem turns the quotient contraction into an evaluation decomposition. The vector $g_\Pi^\star$ is the part of the reward that remains in the persistent subspace, and $v_\Pi^\star$ is the transient representative selected by the gauge. The gauge is a normalization, not an additional modeling assumption. Different projections with kernel $\mathcal K(P)$ choose different representatives of the same quotient class, just as different normalizations choose different classical bias functions. The anchor gauge below is used because it gives the representative a direct hitting time interpretation. Here transience refers to the quotient: $\Pi P^t v_\Pi^\star\to0$. An anchor representative need not itself decay under $P^t$, so its persistent profile and boundary correction must be interpreted jointly. The commuting spectral gauge gives literal separation into an asymptotic reward profile and a decaying remainder. Appendix~\ref{app:cyclic_interpretation} explains the distinction.

\begin{corollary}[Recovery of average reward]\label{cor:recover_invariant_gain}
Let $g$ be the statewise average reward profile from \eqref{eq:gdef}. Then $g=P^\infty g_\Pi^\star=P^\infty r$.
\end{corollary}

Thus $g_\Pi^\star$ refines the classical average reward. Its invariant part is the usual gain, while its non invariant peripheral part records phase resolved persistent behavior.

\begin{proposition}[Reduction to the aperiodic case]\label{prop:aperiodic_reduction}
If every closed irreducible class of $P$ is aperiodic, then $\mathcal K(P)=\operatorname{range}(P^\infty)$. In this case the persistent profile carries no non invariant phase information, and the decomposition reduces to the classical multichain gain bias picture up to the chosen normalization.
\end{proposition}

\subsection{Comparison with classical gain and bias}

A natural concern is that the new pair may add information beyond the classical gain and bias. The following theorem shows that it does not. The decomposition contains the same Poisson information, but makes phase-dependent accounting explicit through the persistent profile and the selected gauge.

\begin{theorem}[Comparison with classical gain and bias]\label{thm:comparison_gain_bias}
Let $(g_\Pi^\star,v_\Pi^\star)$ be the decomposition from Theorem~\ref{thm:wellposed}, and let $(\rho,h)$ be the normalized classical pair from \eqref{eq:poisson}. Then $P^\infty g_\Pi^\star=\rho$. Moreover, there is a unique $\psi_\Pi\in\mathcal K(P)\cap\ker(P^\infty)$ such that $g_\Pi^\star=\rho+(I-P)\psi_\Pi$, and the classical normalized bias satisfies $h=v_\Pi^\star+\psi_\Pi-P^\infty v_\Pi^\star$.
\end{theorem}

\begin{corollary}[Projected classical bias]\label{cor:projected_bias}
The gauge projection recovers the transient component, $\Pi h=v_\Pi^\star$. Equivalently, $g_\Pi^\star=\rho+(I-P)(I-\Pi)h$.
\end{corollary}

Theorem~\ref{thm:comparison_gain_bias} is the formal link to the classical picture. The invariant part of $g_\Pi^\star$ is exactly the classical gain. The correction $\psi_\Pi$ is the persistent phase component that the classical gain cannot store because it is restricted to invariant vectors. The classical bias therefore contains this correction together with the transient component. The term $P^\infty v_\Pi^\star$ is only a normalization effect and disappears after applying the gauge projection. Thus periodic phase information is explicit in the new accounting coordinates, while $v_\Pi^\star$ represents the transient quotient in the selected gauge.

\subsection{Anchor gauge basis}\label{subsec:anchor_gauge_basis}

The paper uses an anchor gauge because it turns the abstract quotient into concrete coordinates. Decompose the chain into transient states $\mathsf T$ and closed irreducible classes $F_i$, and let $\mathsf F=\cup_iF_i$ be the recurrent set. If $F_i$ has period $d_i$, write its cyclic classes as $C_{i,0},\dots,C_{i,d_i-1}$, indexed so that one step from $C_{i,\ell}$ enters $C_{i,\ell+1 \bmod d_i}$. For a recurrent state $y\in F_i$, define $\class(y)=i$ and $\phase(y)=\ell$ when $y\in C_{i,\ell}$. Choose one anchor $a_{i,k}\in C_{i,k}$. Let $\mathcal I=\{(i,k):1\le i\le m,0\le k<d_i\}$ and $N=|\mathcal I|=\sum_i d_i\le |\mathsf F|\le n$. The basis below is a phase corrected absorption basis. Starting from a transient state, it records which closed class is eventually reached and which cyclic phase is reached after subtracting the elapsed time. This correction makes the span of the basis invariant under one step, with $P b_{i,k}=b_{i,k-1 \bmod d_i}$ on class $i$.

\begin{definition}[Phase offset absorption basis]\label{def:main_b}
For $(i,k)\in\mathcal I$, define $b_{i,k}$ by $b_{i,k}(s)=\mathbf 1\{s\in C_{i,k}\}$ on recurrent states. For $s\in\mathsf T$, let $\tau_{\mathsf F}$ be the first hitting time of the recurrent set and set
\begin{equation}\label{eq:b_def}
b_{i,k}(s)=\mathbb P_s\{\class(X_{\tau_{\mathsf F}})=i,\ \phase(X_{\tau_{\mathsf F}})-\tau_{\mathsf F}\equiv k\pmod {d_i}\}.
\end{equation}
\end{definition}

\begin{lemma}[Anchor gauge identifies the peripheral subspace]\label{lem:main_b_identify}
The functions $\{b_{i,k}\}_{(i,k)\in\mathcal I}$ form a basis of $\mathcal K(P)$. The projection
\begin{equation}\label{eq:anchor_projection}
(\Pi v)(s)=v(s)-\sum_{(i,k)\in\mathcal I}v(a_{i,k})b_{i,k}(s)
\end{equation}
has kernel $\mathcal K(P)$ and satisfies
\begin{equation}\label{eq:b_anchor_delta}
b_{i,k}(a_{j,\ell})=\mathbf 1\{(i,k)=(j,\ell)\}.
\end{equation}
\end{lemma}

Lemma~\ref{lem:main_b_identify} is where the abstract peripheral subspace becomes an implementable gauge. The projection $\Pi v$ subtracts the unique persistent profile that agrees with $v$ on all anchors, so the remaining representative vanishes at those anchors. This is the normalization used in the transient cost interpretation below and the quantity learned by the sample based estimator in Section~\ref{algorithm}.

\section{Evaluation consequences of the decomposition}\label{justification}

The previous section constructs the decomposition. We now show what it says about returns. Finite-horizon return can be written as a persistent orbit plus a boundary correction, and this is stronger than an asymptotic statement since at every horizon, it separates the contribution of behavior that continues from the contribution of the transient representative. For any potential $v$, define $g_v\coloneqq r-(I-P)v$.

\begin{lemma}[Return identity]\label{lem:app_telescoping}
For every $v$ and every $T\ge1$,
\begin{equation}\label{eq:app_telescoping}
J_T=\sum_{t=0}^{T-1}P^t g_v+v-P^T v .
\end{equation}
\end{lemma}

Applying Lemma~\ref{lem:app_telescoping} to $v^\star=v_\Pi^\star$ and $g^\star=g_\Pi^\star$ gives
\begin{equation}\label{eq:app_return_decomp_star}
J_T=\sum_{t=0}^{T-1}P^t g^\star+v^\star-P^T v^\star .
\end{equation}
This identity is the main reason for separating $g^\star$ and $v^\star$. The sum $\sum_{t<T}P^t g^\star$ keeps the recurrent class and phase resolved persistent profile visible at each time. The term $v^\star-P^T v^\star$ is a boundary term. It can affect finite horizon returns, but it does not create a new long run rate. Thus the decomposition matches the diagnostic distinction in the introduction: persistent behavior contributes throughout the horizon, while the transient component enters through the path to the selected representative.

Under the anchor gauge, Theorem~\ref{thm:wellposed} and Lemma~\ref{lem:main_b_identify} give $v^\star(a_{i,k})=0$ for all $(i,k)\in\mathcal I$. Let $\tau_{\mathcal A}$ be the first hitting time of the anchor set.

\begin{proposition}[Transient cost under the anchor gauge]\label{prop:app_vstar_transient_cost}
For every state $s$,
\begin{equation}\label{eq:app_vstar_transient_cost}
v^\star(s)=\mathbb E_s\!\left[\sum_{t=0}^{\tau_{\mathcal A}-1}\bigl(r(X_t)-g^\star(X_t)\bigr)\right].
\end{equation}
\end{proposition}

Thus $v^\star(s)$ is the cumulative excess reward, relative to the persistent profile, paid before the chain first reaches the anchor set. The anchors are only a normalization device, but they make the sign and magnitude of $v^\star$ interpretable. If the path to the eventual regime is costly relative to the persistent profile, the transient component records that cost. The next proposition is a sanity check for the interpretation. A deterministic cycle has no transient path once the initial state is fixed inside the cycle. Any nonzero classical bias in this example must therefore come from persistent phase structure rather than from transient cost.

\begin{proposition}[Deterministic cycles have no transient burden]\label{prop:deterministic_cycle}
Let $P$ be the deterministic cycle on $d$ states, $P(k,k+1\,\mathrm{mod}\,d)=1$. Then $\mathcal K(P)=\mathbb R^d$, so the anchor gauge has $\Pi=0$, $v_\Pi^\star=0$, and $g_\Pi^\star=r$. In contrast, the classical normalized pair is $\rho=\bar r\mathbf 1$, with $\bar r=d^{-1}\sum_k r(k)$, and $h$ solves $h(k)-h(k+1\,\mathrm{mod}\,d)=r(k)-\bar r$ with $\sum_k h(k)=0$. Thus $h$ is generally nonzero even though there is no transient path to diagnose. If $d$ is even and $r(k)=1$ for $k<d/2$ and $r(k)=0$ otherwise, then $\|h\|_\infty=d/8$ while $v_\Pi^\star=0$.
\end{proposition}

Proposition~\ref{prop:deterministic_cycle} gives the sharpest separation from the classical bias: a large bias can be caused entirely by persistent phase structure.

\begin{corollary}[Return error from decomposition error]\label{cor:return_error}
Let $(g^\star,v^\star)$ be the exact decomposition and define $\widehat J_T=\sum_{t=0}^{T-1}P^t\widehat g+\widehat v-P^T\widehat v$ for any candidate pair $(\widehat g,\widehat v)$. Then, for every $T\ge 1$,
\begin{equation}\label{eq:return_error_main}
\|\widehat J_T-J_T\|_\infty\le T\|\widehat g-g^\star\|_\infty+2\|\widehat v-v^\star\|_\infty.
\end{equation}
\end{corollary}

The corollary is the bridge to sample-based evaluation: estimating the persistent profile and the transient component gives a direct finite horizon return guarantee. The factor $T$ in front of the persistent profile error is unavoidable because the persistent profile contributes at every time step, while the transient component appears only through two boundary terms. This matches the intended meaning of the two quantities. Proofs of the results in this section are given in Appendix~\ref{app:evaluation_extra_proofs}.

\section{Sample-based estimation of the decomposition}\label{algorithm}

We now define an evaluation report with two quantities, $g_\Pi^\star$ and $v_\Pi^\star$, and estimate them under a tabular generative access model. Since our decomposition depends on closed classes, periods, cyclic phases, and absorption probabilities from transient states. From a single trajectory, or from logged data that does not cover some branch, these objects are not identifiable for general chains. A generative model is one clean way to impose the needed coverage and other data models with comparable coverage could be used. The estimator then follows the same order as the decomposition: learn the persistent coordinates, learn the anchor gauge, solve the stable quotient equation, and reconstruct the persistent profile. The detailed algorithms and proofs formally stated in Appendices~\ref{app:sample_theory} and~\ref{app:algorithms} while this section provides the statistical statement without technical constants. Numerical diagnostics and sample-based experiments are also reported in Appendix~\ref{sec:experiments}.

\subsection{Learning the persistent coordinates}

Assume access to a simulator that returns an independent sample from $P(\cdot| s)$ for any queried state $s$. The first task is structural. The estimator samples every state, forms the empirical support graph, computes closed strongly connected components, and then computes the period and cyclic partition of each closed class. This gives estimates of $F_i$, $d_i$, $C_{i,k}$, and the anchors $a_{i,k}$. If $p_{\min}$ is the smallest positive transition probability, then $K\ge p_{\min}^{-1}\log(n^2/\delta)$ samples per state recover the support graph with probability at least $1-\delta$. On this event, the recurrent classes, periods, cyclic partitions, and anchors are correct up to cyclic relabeling. This recovery step is important because the persistent coordinates are not numerical values that can be chosen after the fact. They are determined by the communication structure and periods of the chain. If a closed class or a cyclic phase is missed, the estimator would project out the wrong subspace and the later contraction statement would apply to the wrong quotient.

The second task is to estimate the basis $b_{i,k}$ from Definition~\ref{def:main_b}. For each transient state, the simulator runs independent episodes until the recurrent set is reached and records the terminal class together with the phase offset. The empirical frequencies define $\widehat b_{i,k}(s)$. With $N=|\mathcal I|$ and $M$ episodes per transient state, the concentration bound in Appendix~\ref{Learning the gauge map} gives $\max_{s\in\mathsf T}\|\widehat b(s)-b(s)\|_1\le \varepsilon_b$ with high probability once $M$ is of order $\varepsilon_b^{-2}(N+\log(|\mathsf T|/\delta))$. The learned projection is $\widehat\Pi v=v-\sum_{(i,k)\in\mathcal I}v(a_{i,k})\widehat b_{i,k}$. On the same event, $\|(\widehat\Pi-\Pi)v\|_\infty\le\varepsilon_b\|v\|_\infty$. This bound is the reason the estimator separates structural learning from numerical fixed point iteration. Once the projection is accurate, every subsequent error term can be measured as a perturbation of the stable quotient equation from Theorem~\ref{thm:quotient_contraction}.

The probabilistic meaning of $b_{i,k}$ is important for interpretation. Starting from a transient state, $b_{i,k}(s)$ is the probability that the chain eventually enters class $F_i$ with the phase offset corresponding to anchor $a_{i,k}$. Thus the learned basis does not approximate an arbitrary spectral vector. It estimates the weights with which the initial state is assigned to persistent regimes. This is the point at which the abstract peripheral subspace becomes a concrete tabular object. The projection $\widehat\Pi$ removes the estimated persistent coordinates by subtracting the anchor values extended through these absorption probabilities.

\begin{algorithm}[t]
\caption{Main estimator for persistent transient policy evaluation}
\label{alg:main_estimator_summary}
\begin{algorithmic}[1]
\REQUIRE Generative model for $P$, reward vector $r$, budgets $K,M,T,J$
\STATE Learn the support graph, closed classes, periods, cyclic phases, and anchors.
\STATE Estimate the phase offset absorption basis $\widehat b_{i,k}$ by absorption episodes.
\STATE Form the learned anchor projection $\widehat\Pi v=v-\sum_{(i,k)}v(a_{i,k})\widehat b_{i,k}$.
\STATE Run projected stochastic approximation on the anchor gauge for $T$ iterations.
\STATE Estimate anchor residuals with $J$ samples per anchor and reconstruct $\widehat g$.
\STATE \textbf{return} Persistent profile $\widehat g$ and transient component $\widehat v$.
\end{algorithmic}
\end{algorithm}

Algorithm~\ref{alg:main_estimator_summary} is deliberately aligned with the decomposition. The first two steps estimate the gauge and therefore determine which coordinates are persistent. The fourth step estimates only the quotient fixed point, where contraction is available. The last step reconstructs $g_\Pi^\star$ from the residual $r+(P-I)v_\Pi^\star$ on anchors and then extends those coordinates through the estimated basis.

This order is useful because the three tasks have different statistical roles. Support recovery is a discrete structural problem. Basis estimation is a probability estimation problem over absorption events. Stochastic approximation is a numerical fixed point problem on a stable quotient. If these tasks are mixed together, the meaning of the resulting value vector becomes unclear. The estimator keeps them separate, so each error term can be traced back to one part of the evaluation report. This is also why the algorithm is suited to tabular policy evaluation. Once a stationary policy is fixed, the induced Markov reward process is the object being evaluated, and the recurrent classes, periods, and phases are properties of this policy induced process.

\subsection{Projected stochastic approximation on the quotient}

Let $W=\{v\in\mathbb R^n:v(a_{i,k})=0\text{ for all }(i,k)\in\mathcal I\}$ be the anchored subspace. By Lemma~\ref{lem:main_b_identify}, this subspace intersects $\mathcal K(P)$ only at zero, so the quotient semi-norm becomes a genuine norm on $W$. The exact projected map $v\mapsto\Pi(r+Pv)$ is contractive on this space. The learned map uses $\widehat\Pi$ and one sample of the next state for each coordinate. If $\widetilde s\sim P(\cdot| s)$, one stochastic update has the form
\begin{equation}\label{eq:main_sa_update_summary}
v_{t+1}=\widehat\Pi\bigl((1-\alpha_t)v_t+\alpha_t\widehat T(v_t)\bigr),
\qquad \widehat T(v_t)(s)=r(s)+v_t(\widetilde s).
\end{equation}
This update shows where the quotient matters. Without removing $\mathcal K(P)$, the fixed point map may have eigenvalues on the unit circle. After projection, the deterministic part is a contraction. With a small basis error, the learned projected operator has contraction factor $\gamma+O(\varepsilon_b)$ and remains stable whenever $\varepsilon_b$ is small enough. This stability statement is the algorithmic counterpart of the minimality theorem. The estimator does not attempt to average away periodicity. It removes the full peripheral subspace, estimates the transient representative in the anchored complement, and later puts the persistent coordinates back into $\widehat g$.

The resulting transient component guarantee has the usual stochastic approximation form, plus the price of learning the gauge. Under the structural recovery event, the basis concentration event, a bounded second moment condition for the oracle noise, and a standard stepsize choice, Appendix~\ref{projSA} proves
\begin{equation}\label{eq:v_rate_summary}
\mathbb E\|[v_T]-[v_\Pi^\star]\|_{\mathrm q}\le O(T^{-1/2})+O(\varepsilon_b).
\end{equation}
Equivalently, setting $T$ of order $\varepsilon^{-2}$ and choosing the basis accuracy at the same scale gives quotient error of order $\varepsilon$. The constants depend on the quotient contraction, the norm equivalence on $W$, and the noise moment bound. These dependencies are unavoidable because the algorithm estimates a fixed point in a norm induced by the quotient. They are also useful diagnostically. A small quotient error means that the estimated transient component is accurate after ignoring persistent coordinates, exactly as the theory requires. It does not require the stochastic approximation iterates to learn phase oscillations as part of the transient value.

Periodicity alone does not imply failure of temporal-difference (TD) learning: differential TD admits convergence guarantees under suitable unichain, coverage, and stepsize assumptions \cite{wan2021learning}. A converged classical bias can correctly encode persistent phase effects; our projected update instead targets $v_\Pi^\star=\Pi h$. For a deterministic two-state cycle with rewards $(0,2)$, the normalized classical bias is $(-1/2,1/2)$, whereas our transient component is zero: both evaluations are correct, but they attribute phase effects differently. The minimality theorem concerns the original transition dynamics; relaxation can stabilize non-invariant phase modes without removing them. General multichain evaluation poses a separate issue: a scalar gain cannot represent closed classes with unequal averages, and one continuing trajectory does not generally cover all classes or repeatedly visit transient states. Generative access supplies that coverage at the structural-learning costs stated here. Thus classical estimation followed by the same gauge postprocessing is a valid alternative; our results do not establish universal TD failure or superior sample efficiency.

\subsection{Reconstructing the persistent profile and returns}

Once $v_T$ has been estimated, the persistent profile is recovered from anchor residuals. For each anchor, the estimator samples $Y\sim P(a_{i,k},\cdot)$ and estimates $Pv_T(a_{i,k})$ by an empirical average. The anchor coordinate is then $\widehat\theta_{i,k}=r(a_{i,k})+\widehat{Pv_T}(a_{i,k})-v_T(a_{i,k})$, and the full profile is reconstructed as $\widehat g(s)=\sum_{(i,k)\in\mathcal I}\widehat\theta_{i,k}\widehat b_{i,k}(s)$. This formula mirrors Theorem~\ref{thm:wellposed}: the residual lies in the peripheral subspace, and the anchor basis extends its coordinates from anchors to all states.

The persistent profile bound in Appendix~\ref{projSA} states that, for $J$ next state samples per anchor,
\begin{equation}\label{eq:g_rate_summary}
\mathbb E\|\widehat g-g_\Pi^\star\|_\infty
\le C_1\mathbb E\|v_T-v_\Pi^\star\|_{\mathrm q}+C_2\sqrt{\frac{\log N}{J}}+C_3\varepsilon_b.
\end{equation}
Combining this with \eqref{eq:v_rate_summary} yields the end to end form $\mathbb E\|\widehat g-g_\Pi^\star\|_\infty\le O(T^{-1/2})+O(\sqrt{\log N/J})+O(\varepsilon_b)$. Together with Corollary~\ref{cor:return_error}, the estimator therefore controls finite horizon prediction by estimating the two components that have evaluation meaning. The query complexity has three parts: $nK$ samples for the support graph, expected $|\mathsf T|MH_{\rm abs}$ samples for absorption episodes, where $H_{\rm abs}=\max_{s\in\mathsf T}\mathbb E_s[\tau_{\mathsf F}]$, and $nT+NJ$ samples for stochastic approximation and anchor residuals. Thus the theory is operational in tabular MDP evaluation: it estimates the average reward, the persistent phase resolved regime, and the transient to regime cost from the same generative access used by standard model based methods. The output can be used in two ways. For asymptotic evaluation, $P^\infty\widehat g$ estimates the statewise average reward. For diagnostic evaluation, the orbit of $\widehat g$, together with the boundary correction, provides phase-aware return accounting, and $\widehat v$ reports the transient cost of reaching the anchored regime. This is the sense in which the estimator implements the persistent transient decomposition rather than only solving an average reward equation.

This structural recovery step is part of the statistical target instead of merely a preprocessing convenience. The quotient and the anchor gauge need to be determined by the support graph through the closed classes, their periods, cyclic partitions, and anchors. If this map is incorrect, the learned projection may have a different kernel and the algorithm may solve a different quotient problem. Thus a finite-sample guarantee for the stated decomposition must work on a high-probability event where this structure is recovered correctly. The rare-edge parameter $p_{\min}$ is used only to give a sufficient sample size for this event. It is not an input to the estimator. Its appearance reflects an identifiability limit for exact support recovery. Specifically, if a rare transition is never sampled, then a chain with that transition and a chain without it are statistically indistinguishable, although they may induce different recurrent structure and hence a different decomposition. Appendix~\ref{Learning the gauge map} gives the formal support-recovery argument.

\begin{theorem}[sample-based decomposition guarantee]\label{thm:sample_summary_main}
Let $\mathcal E_{\rm str}$ be the event that the support graph, recurrent classes, periods, cyclic partitions, and anchors are recovered correctly. On $\mathcal E_{\rm str}$, suppose the basis estimation error is at most $\varepsilon_b$, the oracle noise has a bounded conditional second moment, and the learned quotient contraction factor is smaller than one. Then the estimator in Algorithm~\ref{alg:main_estimator_summary} satisfies
\begin{equation}\label{eq:sample_summary_v}
\mathbb E\|[\widehat v]-[v_\Pi^\star]\|_{\mathrm q}\le O(T^{-1/2})+O(\varepsilon_b)
\end{equation}
and, with $J$ samples per anchor for residual estimation,
\begin{equation}\label{eq:sample_summary_g}
\mathbb E\|\widehat g-g_\Pi^\star\|_\infty\le O(T^{-1/2})+O\!\left(\sqrt{\frac{\log N}{J}}\right)+O(\varepsilon_b).
\end{equation}
Consequently, the finite horizon prediction error is controlled by Corollary~\ref{cor:return_error}.
\end{theorem}

Together with the support recovery guarantee in Appendix~\ref{Learning the gauge map}, Theorem~\ref{thm:sample_summary_main} gives a high probability structural guarantee followed by estimation bounds on the recovered quotient. The term $T^{-1/2}$ is the stochastic approximation error, $\varepsilon_b$ is the price of learning the anchor projection, and $\sqrt{\log N/J}$ appears only when reconstructing the persistent profile from anchor residuals. These errors mirror the decomposition itself: errors in $\widehat v$ affect boundary terms, while errors in $\widehat g$ accumulate over the horizon.

The parameters in the guarantee have structural meanings. The number of persistent coordinates is $N=|\mathcal I|=\sum_i d_i\le |\mathsf F|\le n$, and $H_{\rm abs}$ only converts absorption episodes into expected simulator queries. A dense plug in estimator could estimate $P$ and compute the same decomposition from the learned model, so our claim is not uniform dominance over plug in methods. The point is that the estimator follows the persistent transient structure directly. For tabular MDP evaluation, the procedure is applied after fixing a stationary policy, so the learned recurrent classes and phases are properties of the policy induced chain. The output fits the usual evaluation tasks: $P^\infty\widehat g$ estimates the statewise average reward, the orbit of $\widehat g$ supplies the persistent accounting profile, and $\widehat v$ gives the anchor gauge transient diagnostic.

\section{Conclusion}\label{sec:conclusion}

This paper studies fixed policy evaluation in finite-state Markov chains beyond the aperiodic setting. Classical normalized gain and bias remain valid, but the bias can mix persistent phase behavior with transient cost. The real peripheral invariant subspace resolves this ambiguity. Quotienting by it is minimal for strict contraction, and an anchor gauge gives the decomposition $r=g_\Pi^\star+(I-P)v_\Pi^\star$. The pair contains the same Poisson information as classical gain and bias, but attributes it differently: $g_\Pi^\star$ records the persistent regime, while $v_\Pi^\star$ records the cost of reaching the anchored representative of that regime. The contribution is therefore not a new way to compute returns, but a new way to attribute them. The sample-based estimator follows the same structure by learning the recurrent phases, estimating the anchor projection, solving the stable quotient problem, and reconstructing the persistent profile. Its guarantees translate directly into finite horizon return error. The result is a tabular policy evaluation object that supports average reward evaluation, phase resolved diagnostics, and transient cost attribution from the same generative access.

\clearpage
\bibliographystyle{plain}
\bibliography{main}
\newpage
\appendix
\onecolumn
\renewcommand{\theHfigure}{appendix.figure.\arabic{figure}}
\renewcommand{\theHtable}{appendix.table.\arabic{table}}
\newcommand{\cF}{\mathcal{F}}

\section{Related work} \label{app:related_work}
We review prior work through the lens of the evaluation question posed in the Introduction: what object is used to summarize a fixed policy-induced Markov reward process, and which part of the analysis concerns control, prediction, fixed-point stabilization, or Markov-chain structure?  This distinction is important because much of the average-reward literature gives valid gain-bias formalisms or efficient control algorithms, while our paper studies a different object: a phase-aware persistent profile together with an anchor-gauge transient component for an arbitrary finite reducible or periodic policy-induced chain.

\subsection{Average-reward control and planning beyond ergodicity.}
Classical dynamic programming treats average-reward MDPs through gain and bias equations, including multichain policy and value iteration methods \cite{puterman2014markov,bertsekas2012dynamic,schweitzer1977asymptotic,schweitzer1979geometric,schweitzer1978foolproof}.  Modern learning work has refined this picture for communicating, weakly communicating, robust, and general average-reward MDPs.  Representative examples include regret and sample-complexity guarantees for communicating or weakly communicating models \cite{auer2008near,bartlett2009regal,zurek2024span,lee2025anchoring}, robust average-reward control via Halpern-type iteration \cite{roch2025robust}, and multichain fixed-point and relative-value methods \cite{gupta2015empirical,lee2025optimal,zurekchen2025faster}.  These works primarily address how to learn, plan, or optimize a policy, usually through an invariant gain and a bias or relative value function.  Our setting is fixed-policy evaluation.  We do not claim that the classical gain-bias object is invalid; rather, we show that when recurrent classes are periodic it can be diagnostically entangled because non-invariant persistent modes are stored in the bias.  The peripheral quotient is introduced to separate that persistent information from transient-to-regime cost.

\subsection{Average-reward policy evaluation and temporal-difference methods.}
Average-reward prediction is also well studied in ergodic or unichain settings.  Classical and modern analyses estimate a scalar reward rate and a differential value function, possibly with linear function approximation or off-policy sampling \cite{tsitsiklis1999average,yu2009convergence,abounadi2001learning,wan2021learning,zhang2021finite,zhang2021average,haque2025unbounded}.  Robust average-reward policy evaluation has recently been analyzed under contraction in an appropriate semi-norm \cite{xu2025finite}.  These methods are complementary to ours.  They address stochastic approximation, off-policy learning, function approximation, robustness, or unbounded Markovian noise, whereas our main difficulty is structural: for an arbitrary finite policy-induced chain, the non-decaying directions of the original transition dynamics are not merely constants or invariant recurrent-class shifts.  Periodic classes introduce additional unit-modulus phase directions, and the evaluation object must decide whether those directions are persistent behavior or transient cost.  Our decomposition answers this question before the stochastic approximation step is run.

\subsection{Nonexpansive and semi-norm-contractive fixed-point methods.}
Several recent average-reward algorithms use anchoring or Halpern iteration to stabilize nonexpansive Bellman operators \cite{halpern1967fixed,lee2025anchoring,roch2025robust}.  More generally, \cite{chen2025seminorm} gives a non-asymptotic theory of semi-norm Lyapunov stability and characterizes when a linear iteration contracts modulo a subspace containing all modes with modulus at least one.  This perspective is closely related to the contraction part of our analysis.  The difference is what is being constructed.  Those results give general fixed-point or control tools once the relevant semi-norm or quotient kernel is specified.  We identify the Markov-chain-specific kernel for fixed-policy evaluation, prove that it is exactly the real peripheral invariant subspace, construct a learnable anchor gauge from recurrent classes, periods, cyclic phases, and absorption weights, and interpret the resulting coordinates as persistent regime profile and transient-to-regime cost.

\subsection{Poisson equations and generalized inverses for Markov chains.}
Poisson equations, fundamental matrices, and generalized inverses are standard tools in finite Markov-chain theory \cite{schweitzer1968perturbation,meyer1975role,puterman2014markov,bertsekas2012dynamic}.  They characterize solutions of $r=\rho+(I-P)h$, sensitivity, and transient corrections once a normalization is chosen.  Recent work also uses Poisson equations as an analytic device for other statistical targets, such as Markov-chain asymptotic variance estimation \cite{agrawal2024variance}.  Our contribution is not the existence of a Poisson equation or a generalized inverse.  It is the decomposition of the same Poisson information along the dynamical boundary between persistence and transience.  The exact comparison theorem shows that the classical normalized gain and bias can be recovered from our pair, and conversely that the classical bias consists of our transient component plus a non-invariant peripheral correction.  This is why the construction is a refinement of classical Markov-chain Poisson theory rather than a replacement for it.

\section{Detailed sample-based estimation theory}\label{app:sample_theory}

This appendix gives the detailed guarantees behind the sample-based estimator summarized in Section~\ref{algorithm}. The estimator first learns the support graph, recurrent classes, cyclic phases, and anchors; then estimates the phase-offset basis from Definition~\ref{def:main_b}; finally it runs projected stochastic approximation for $v_\Pi^\star$ and reconstructs $g_\Pi^\star$ from anchor residuals. Algorithms are listed in Appendix~\ref{app:algorithms}, and proofs are in Appendices~\ref{proof_learning_gauge}--\ref{proof_peripheral}.

\subsection{Learning the anchor gauge}\label{Learning the gauge map}

Assume access to a generative model that returns an independent sample from $P(\cdot| s)$ for any queried state $s$.  Let
\[
p_{\min}:=\min\{P(s'| s):P(s'| s)>0\}.
\]
The first stage samples each state $K$ times, forms the empirical support graph, computes closed strongly connected components, their periods, cyclic partitions, and one anchor per phase.

\begin{theorem}[Exact structural recovery]\label{thm:main_support_recovery}
If $p_{\min}>0$ and $K\ge p_{\min}^{-1}\log(n^2/\delta)$, then with probability at least $1-\delta$ the learned support graph equals the true support graph. Consequently, the recurrent classes, their periods, and their cyclic partitions are correctly recovered up to cyclic relabeling.
\end{theorem}

\begin{remark}[Rare edges and support recovery]
The parameter $p_{\min}$ is not required by the algorithm. It appears only in the sufficient sample size for exact support recovery. This dependence is unavoidable in the worst case. Consider two chains that differ only at one state $s$. In the first chain, the edge $s\to v$ is absent. In the second chain, the edge $s\to v$ has probability $p_{\min}$ and the remaining mass is assigned to an otherwise identical transition. If $K$ samples are drawn from $s$, then under the second chain the edge is unseen with probability $(1-p_{\min})^K$. On this event, the sample transcript is compatible with the first chain. Thus any procedure that distinguishes the two supports with error probability at most $\delta$ must have $(1-p_{\min})^K\le \delta$, which gives $K=\Omega(p_{\min}^{-1}\log(1/\delta))$.
\end{remark}

\begin{remark}[Good event interpretation]
The end to end guarantee is stated on the event that the support graph and the induced recurrent and cyclic structure are recovered correctly. If the support graph is misspecified, the learned projection may have a different kernel and may target a different quotient problem. The theorem therefore gives a high probability exact structure guarantee followed by a conditional estimation guarantee on that event. An unconditional bound can be obtained by adding the failure probability times a uniform bad event bound whenever such a bound is imposed.
\end{remark}

The second stage estimates the basis $b_{i,k}$ by independent absorption episodes. For each transient state $s$, let $\widehat b(s)=(\widehat b_{i,k}(s))_{(i,k)\in\mathcal I}$ be the empirical distribution of the terminal class-phase offset in Definition~\ref{def:main_b}, and let $N:=|\mathcal I|$.

\begin{lemma}[Basis estimation]\label{lem:main_weight_conc}
If
\[
M\ge \frac{8}{\varepsilon_b^2}\left(N+\log \frac{|\Tset|}{\delta}\right),
\]
then with probability at least $1-\delta$,
\[
\max_{s\in\Tset}\|\widehat b(s)-b(s)\|_1\le \varepsilon_b .
\]
\end{lemma}

Define the learned anchor projection by
\begin{equation}\label{eq:pihat_def_compact}
(\widehat\Pi v)(s)=v(s)-\sum_{(i,k)\in\mathcal I}v(a_{i,k})\widehat b_{i,k}(s).
\end{equation}

\begin{lemma}[Learned projection deviation]\label{lem:main_Pihat_dev}
On the event of Lemma~\ref{lem:main_weight_conc}, assuming the recurrent-state basis values are exact,
\[
\| (\widehat\Pi-\Pi)v\|_\infty\le \varepsilon_b\|v\|_\infty,
\qquad \forall v\in\mathbb R^n.
\]
\end{lemma}

\subsection{Projected stochastic approximation and reconstruction}\label{projSA}
\label{peripheral}

Let
\[
W:=\{v\in\mathbb R^n:v(a_{i,k})=0\text{ for all }(i,k)\in\mathcal I\}
\]
be the anchored subspace.  The quotient semi-norm becomes a genuine norm on $W$.

\begin{lemma}[Norm equivalence on the anchored subspace]\label{lem:W_norm_equiv}
We have $W\cap\mathcal K(P)=\{0\}$. Hence $\|\cdot\|_{\mathrm q}$ is a norm on $W$, and there exists $C_W<\infty$ such that
\[
\|w\|_\infty\le C_W\|w\|_{\mathrm q},\qquad \forall w\in W.
\]
\end{lemma}

Given $v_t\in W$, the SA update draws one next-state sample $\widetilde s\sim P(\cdot| s)$ for each $s$, forms $\widehat T(v_t)(s)=r(s)+v_t(\widetilde s)$, and sets
\[
v_{t+1}=\widehat\Pi\big((1-\alpha_t)v_t+\alpha_t\widehat T(v_t)\big).
\]
The next lemma gives the stability reason for using the peripheral quotient.

\begin{lemma}[Perturbed quotient contraction]\label{lem:main_hat_contraction}
On the good structural and basis-estimation events, the learned projected operator is contractive on $W$ with contraction factor
\[
\widehat\gamma=\gamma+O(\varepsilon_b).
\]
In particular, for sufficiently small $\varepsilon_b$, $\widehat\gamma<1$.
\end{lemma}

\begin{theorem}[Transient-component estimation]\label{thm:main_sa_rate}
Assume the structural learning and weight estimation steps succeed, the oracle noise has a uniform conditional second moment bound, and $\widehat\gamma<1$. Then the projected SA iterates satisfy
\[
\mathbb E\|[v_T]-[v^\star]\|_{\mathrm q}
\le O(T^{-1/2})+O(\varepsilon_b),
\]
where the constants depend on the contraction and second-moment parameters.
\end{theorem}

\begin{corollary}[End-to-end query complexity]\label{cor:main_end_to_end_v}
Let $H_{\rm abs}:=\max_{s\in\Tset}\mathbb E_s[\tauF]$. To achieve $\mathbb E\|[v_T]-[v^\star]\|_{\mathrm q}\le\varepsilon$, it suffices to choose
\[
K=\widetilde O(1/p_{\min}),\qquad T=\widetilde O(\varepsilon^{-2}),\qquad M=\widetilde O(N\varepsilon^{-2}),
\]
with expected simulator queries at most
\[
nK+|\Tset|MH_{\rm abs}+nT.
\]
\end{corollary}

\begin{remark}[Structural parameters]
The quantity $N=|\mathcal I|=\sum_i d_i$ is the number of recurrent class and phase coordinates. Since $d_i\le |F_i|$ for each recurrent class, we have $N\le |\mathsf F|\le n$. The quantity $H_{\rm abs}=\max_{s\in\Tset}\mathbb E_s[\tauF]$ only converts absorption episodes into expected simulator queries. The norm equivalence constant $C_W$ and the constants in Theorem~\ref{thm:main_g_rate} depend on the conditioning of the quotient and the anchor gauge geometry, but not on the sample budgets $K,M,T,J$.
\end{remark}

Finally, estimate the anchor residual coordinates
\[
\widehat\theta_{i,k}=r(a_{i,k})+\widehat{P v_T}(a_{i,k})-v_T(a_{i,k})
\]
using $J$ samples per anchor, and reconstruct $\widehat g(s)=\sum_{(i,k)}\widehat\theta_{i,k}\widehat b_{i,k}(s)$.

\begin{theorem}[Persistent-profile reconstruction]\label{thm:main_g_rate}
Under the conditions of Theorem~\ref{thm:main_sa_rate}, if the iterates lie in $W$, then
\[
\mathbb E\|\widehat g-g_\Pi^\star\|_\infty
\le C_1\mathbb E\|v_T-v_\Pi^\star\|_{\mathrm q}
+C_2\sqrt{\frac{\log N}{J}}+C_3\varepsilon_b.
\]
Consequently,
\[
\mathbb E\|\widehat g-g_\Pi^\star\|_\infty
\le O(T^{-1/2})+O\!\left(\sqrt{\frac{\log N}{J}}\right)+O(\varepsilon_b).
\]
\end{theorem}

Together with Corollary~\ref{cor:return_error}, these guarantees convert decomposition estimation into finite-horizon evaluation error.  The learned output estimates the objects introduced in Section~\ref{Quotient}: the persistent regime profile, the anchor-gauge transient component, and the classical average reward obtained from the invariant part of the profile.

\section{Algorithms for the sample-based estimator}\label{app:algorithms}
The main text summarizes the estimator; the detailed procedures are listed here.

\begin{algorithm}[t]
\caption{Learn support graph, recurrent classes, cyclic structure, and anchors}
\label{alg:main_learn_structure}
\begin{algorithmic}[1]
\REQUIRE Generative model for $P$, sample budget $K$
\ENSURE Closed classes $\{F_i\}_{i=1}^m$, recurrent set $\mathcal F$, transient set $\mathcal T$
\ENSURE Periods $\{d_i\}_{i=1}^m$, cyclic sets $\{C_{i,k}\}$, maps $\mathrm{cls}$ and $\mathrm{phase}$
\ENSURE Anchors $\{a_{i,k}\}$ and index set $\mathcal I$
\STATE Initialize directed edge set $\widehat E\leftarrow\emptyset$
\FOR{$s\in\mathcal S$}
    \FOR{$k=1,2,\dots,K$}
        \STATE Draw $\tilde s\sim P(s,\cdot)$
        \STATE Add edge $(s,\tilde s)$ to $\widehat E$
    \ENDFOR
\ENDFOR
\STATE Form directed graph $\widehat G=(\mathcal S,\widehat E)$
\STATE Compute strongly connected components of $\widehat G$ and identify closed components
\STATE Let $\{F_i\}_{i=1}^m$ be the closed components
\STATE Set $\mathcal F\leftarrow \cup_{i=1}^m F_i$ and $\mathcal T\leftarrow \mathcal S\setminus \mathcal F$
\FOR{each $F_i$}
    \STATE Compute period $d_i$
    \STATE Compute cyclic partition $F_i=\cup_{k=0}^{d_i-1} C_{i,k}$ with transitions from $C_{i,k}$ to $C_{i,k+1 \bmod d_i}$
    \FOR{$k=0,1,\dots,d_i-1$}
        \STATE Choose one anchor $a_{i,k}\in C_{i,k}$
    \ENDFOR
\ENDFOR
\STATE Define $\mathrm{cls}(s)=i$ and $\mathrm{phase}(s)=k$ for $s\in C_{i,k}\subseteq\mathcal F$
\STATE Define $\mathcal I=\{(i,k): i\in\{1,\dots,m\},\ k\in\{0,\dots,d_i-1\}\}$
\STATE \textbf{return} $\{F_i\}$, $\mathcal F$, $\mathcal T$, $\{d_i\}$, $\{C_{i,k}\}$, $\mathrm{cls}$, $\mathrm{phase}$, $\{a_{i,k}\}$, $\mathcal I$
\end{algorithmic}
\end{algorithm}

\begin{algorithm}[t]
\caption{Estimate phase offset absorption weights}
\label{alg:main_estimate_weights}
\begin{algorithmic}[1]
\REQUIRE Generative model for $P$
\REQUIRE Recurrent set $\mathcal F$, transient set $\mathcal T$
\REQUIRE Periods $\{d_i\}$, cyclic sets $\{C_{i,k}\}$, maps $\mathrm{cls}$ and $\mathrm{phase}$, index set $\mathcal I$
\REQUIRE Episode budget $M$
\ENSURE Estimated weights $\{\widehat b_{i,k}(s)\}_{s\in\mathcal S,(i,k)\in\mathcal I}$
\STATE Initialize $\widehat b_{i,k}(s)\leftarrow 0$ for all $s\in\mathcal S$ and $(i,k)\in\mathcal I$
\FOR{$s\in \mathcal F$}
    \FOR{$(i,k)\in \mathcal I$}
        \STATE Set $\widehat b_{i,k}(s)\leftarrow \mathbf 1\{s\in C_{i,k}\}$
    \ENDFOR
\ENDFOR
\FOR{$s\in \mathcal T$}
    \STATE Initialize counts $c_{i,k}\leftarrow 0$ for all $(i,k)\in\mathcal I$
    \FOR{$j=1,2,\dots,M$}
        \STATE Set $X\leftarrow s$ and $\tau\leftarrow 0$
        \WHILE{$X\notin \mathcal F$}
            \STATE Draw $X'\sim P(X,\cdot)$, set $X\leftarrow X'$, set $\tau\leftarrow \tau+1$
        \ENDWHILE
        \STATE Set $i\leftarrow \mathrm{cls}(X)$ and $\ell\leftarrow \mathrm{phase}(X)$
        \STATE Set $k\leftarrow (\ell-\tau)\bmod d_i$
        \STATE Increment $c_{i,k}\leftarrow c_{i,k}+1$
    \ENDFOR
    \FOR{$(i,k)\in\mathcal I$}
        \STATE Set $\widehat b_{i,k}(s)\leftarrow c_{i,k}/M$
    \ENDFOR
\ENDFOR
\STATE \textbf{return} $\widehat b$
\end{algorithmic}
\end{algorithm}

\begin{algorithm}[t]
\caption{Build learned gauge map}
\label{alg:main_build_gauge}
\begin{algorithmic}[1]
\REQUIRE Anchors $\{a_{i,k}\}$, index set $\mathcal I$, estimated weights $\widehat b$
\ENSURE A callable linear map $\widehat\Pi:\mathbb R^n\to\mathbb R^n$
\STATE Define $\widehat\Pi$ by the rule below for any $v\in\mathbb R^n$ and any $s\in\mathcal S$
\[
(\widehat\Pi v)(s)\leftarrow v(s)-\sum_{(i,k)\in\mathcal I} v(a_{i,k})\widehat b_{i,k}(s)
\]
\STATE \textbf{return} $\widehat\Pi$
\end{algorithmic}
\end{algorithm}

\begin{algorithm}[t]
\caption{Projected quotient stochastic approximation}
\label{alg:main_psa}
\begin{algorithmic}[1]
\REQUIRE Generative model for $P$, reward vector $r\in\mathbb R^n$
\REQUIRE Learned map $\widehat\Pi$, stepsizes $\{\alpha_t\}_{t=0}^{T-1}$, iteration budget $T$
\ENSURE Value iterate $v_T\in\mathbb R^n$ in the range of $\widehat\Pi$
\STATE Initialize $v_0\in\mathbb R^n$ and set $v_0\leftarrow \widehat\Pi v_0$
\FOR{$t=0,1,\dots,T-1$}
    \STATE Initialize vector $\widehat T(v_t)\in\mathbb R^n$
    \FOR{$s\in\mathcal S$}
        \STATE Draw $\tilde s\sim P(s,\cdot)$
        \STATE Set $\widehat T(v_t)(s)\leftarrow r(s)+v_t(\tilde s)$
    \ENDFOR
    \STATE Set $v_{t+\frac12}\leftarrow (1-\alpha_t)v_t+\alpha_t\,\widehat T(v_t)$
    \STATE Set $v_{t+1}\leftarrow \widehat\Pi v_{t+\frac12}$
\ENDFOR
\STATE \textbf{return} $v_T$
\end{algorithmic}
\end{algorithm}

\begin{algorithm}[t]
\caption{Estimate and reconstruct persistent regime profile}
\label{alg:main_g}
\begin{algorithmic}[1]
\REQUIRE Generative model for $P$, reward vector $r$, iterate $v_T$
\REQUIRE Anchors $\{a_{i,k}\}$, index set $\mathcal I$, weights $\widehat b$, sample budget $J$
\ENSURE Coordinate estimates $\widehat\theta$ and reconstructed profile $\widehat g$
\FOR{each $(i,k)\in\mathcal I$}
    \STATE Draw $J$ independent samples $Y^{(1)},\dots,Y^{(J)}$ from $P(a_{i,k},\cdot)$
    \STATE Set $\widehat{(Pv_T)}(a_{i,k})\leftarrow \frac{1}{J}\sum_{j=1}^J v_T(Y^{(j)})$
    \STATE Set $\widehat\theta_{i,k}\leftarrow r(a_{i,k})+\widehat{(Pv_T)}(a_{i,k})-v_T(a_{i,k})$
\ENDFOR
\FOR{$s\in\mathcal S$}
    \STATE Set $\widehat g(s)\leftarrow \sum_{(i,k)\in\mathcal I}\widehat\theta_{i,k}\,\widehat b_{i,k}(s)$
\ENDFOR
\STATE \textbf{return} $\widehat\theta$ and $\widehat g$
\end{algorithmic}
\end{algorithm}

\section{Additional proofs for Section~\ref{justification}}\label{app:evaluation_extra_proofs}

\begin{proof}[Proof of Lemma~\ref{lem:app_telescoping}]
By the definition of $g_v$, we have $r = g_v + v - Pv$. Multiplying by $P^t$ gives
$P^t r = P^t g_v + P^t v - P^{t+1}v$.
Summing over $t=0,\dots,T-1$ telescopes:
$\sum_{t=0}^{T-1} P^t r = \sum_{t=0}^{T-1} P^t g_v + v - P^T v$.
Evaluating at coordinate $s$ gives \eqref{eq:app_telescoping}.
\end{proof}

\begin{proof}[Proof of Proposition~\ref{prop:app_vstar_transient_cost}]
From \eqref{eq:gstar}, rearrange to $r-g^\star = v^\star - Pv^\star$.
Define
$
M_t := v^\star(X_t) + \sum_{u=0}^{t-1} (r(X_u)-g^\star(X_u)).
$
A one-step conditional expectation gives
\[
\mathbb E[M_{t+1}-M_t| X_t]
=
(Pv^\star)(X_t)-v^\star(X_t) + (r(X_t)-g^\star(X_t)) = 0,
\]
so $(M_t)_{t\ge 0}$ is a martingale.
By optional stopping at $\tau_{\mathcal A}$ (integrability holds since $\mathbb E_s[\tau_{\mathcal A}]<\infty$ and everything is bounded),
$\mathbb E_s[M_{\tau_{\mathcal A}}]=\mathbb E_s[M_0]=v^\star(s)$.
Using the anchor normalization gives $v^\star(X_{\tau_{\mathcal A}})=0$ almost surely, yielding \eqref{eq:app_vstar_transient_cost}.
\end{proof}

\begin{proof}[Proof of Proposition~\ref{prop:deterministic_cycle}]
For a deterministic $d$-cycle, $P^d=I$, so all eigenvalues are $d$th roots of unity and the real peripheral subspace is all of $\mathbb R^d$. Hence any projection with kernel $\mathcal K(P)$ is the zero projection, and Theorem~\ref{thm:wellposed} gives $v_\Pi^\star=0$ and $g_\Pi^\star=r$.
The classical gain is the Ces\`aro average $\rho=\bar r\mathbf 1$. The normalized bias equation is exactly $h-Ph=r-\rho$ with $P^\infty h=0$, which gives the displayed recurrence and zero-mean condition.
For the even-cycle example $d=2m$, the recurrence decreases by $1/2$ for the first $m$ states and increases by $1/2$ for the remaining $m$ states. The zero-mean solution has maximum $m/4$ and minimum $-m/4$, so $\|h\|_\infty=m/4=d/8$.
\end{proof}

\begin{proof}[Proof of Corollary~\ref{cor:return_error}]
Subtract the exact identity \eqref{eq:app_return_decomp_star} from the definition of $\widehat J_T$:
\[
\widehat J_T-J_T=\sum_{t=0}^{T-1}P^t(\widehat g-g^\star)+(\widehat v-v^\star)-P^T(\widehat v-v^\star).
\]
Since $P$ is row-stochastic, $\|P^t x\|_\infty\le\|x\|_\infty$ for all $t\ge0$. Taking sup norms gives the result.
\end{proof}

\section{Relation to plug in estimation}\label{app:plugin}

A natural alternative is to estimate the full transition matrix $\widehat P$ and then compute recurrent classes, periods, the peripheral subspace, the gauge projection, and the Poisson solution from the learned model. This approach is valid when $\widehat P$ is accurate enough to recover the relevant structure, and on small dense tabular problems it may have smaller final numerical error. Our claim is not that the proposed estimator uniformly dominates plug in estimation. The distinction is that the proposed procedure estimates only the structural and quotient objects needed for the decomposition, avoids storing a dense transition matrix during stochastic approximation, and separates the sources of error according to the decomposition. This is why the guarantees are stated in terms of support recovery, anchor basis estimation, quotient stochastic approximation, and residual reconstruction.

\section{Proofs for Section~\ref{Quotient}}

\subsection{Proof of Theorem \ref{thm:minimal_quotient}}
\label{proof_minimal_quotient}

\begin{proof}
Write the three statements in the order used in Theorem~\ref{thm:minimal_quotient}. We first compare statements two and three. By definition, $\mathcal K(P)$ is the real span of all generalized eigenmodes whose eigenvalues have modulus one. Hence quotienting by $K$ kills every non decaying mode exactly when $\mathcal K(P)\subseteq K$.

Assume next that $\mathcal K(P)\subseteq K$. Since $P$ is a finite state stochastic matrix, all eigenvalues lie in the closed unit disk. The finite dimensional spectral decomposition gives a real $P$ invariant subspace $L$ such that $\mathbb R^n=\mathcal K(P)\oplus L$ and $\rho(P|_L)<1$. Since $K$ is invariant and contains $\mathcal K(P)$, the quotient $\mathbb R^n/K$ is naturally a quotient of $L$. The induced operator on this quotient has spectral radius at most $\rho(P|_L)$, and is therefore strictly smaller than one.

It remains to prove the converse. Suppose $\mathcal K(P)$ is not contained in $K$. Then some nonzero peripheral mode survives in the quotient. After complexification, the quotient operator has an eigenvalue $\lambda$ with $|\lambda|=1$ on the image of $\mathcal K_{\mathbb C}(P)$. Hence $\rho(\bar P_K)\ge1$. This proves that strict stability of the quotient forces $\mathcal K(P)\subseteq K$. Combining the two directions proves the equivalence and the minimality of $\mathcal K(P)$.
\end{proof}

\subsection{Proof of Theorem \ref{thm:quotient_contraction}} \label{proof_quotient_contraction}

\begin{proof}
We start by showing that $\bar P([v])=[Pv]$ is well-defined on $\mathbb R^n/\mathcal K(P)$.
By construction, $\mathcal K_{\mathbb C}(P)$ is a $P$-invariant subspace of $\mathbb C^n$ since it is a direct sum of generalized eigenspaces.
Hence its realification
\begin{equation}
\mathcal K(P)=\{\Re z:z\in\mathcal K_{\mathbb C}(P)\}+\{\Im z:z\in\mathcal K_{\mathbb C}(P)\}\subseteq\mathbb R^n
\end{equation}
is also $P$-invariant. Therefore, if $v-w\in \mathcal K(P)$ then
$Pv-Pw=P(v-w)\in \mathcal K(P)$, which implies $[Pv]=[Pw]$.
Thus the induced operator $\bar P([v])=[Pv]$ on $\mathbb R^n/\mathcal K(P)$ is well-defined and linear.

Since $P$ is row-stochastic, $\|P\|_\infty=1$, hence every eigenvalue $\lambda$ of $P$ satisfies $|\lambda|\le \rho(P)\le \|P\|_\infty=1$.

Work over $\mathbb C^n$ and let $J$ be the Jordan normal form of $P$:
there exists an invertible matrix $S$ such that
\[
S^{-1} P S = J = \mathrm{diag}(J_1,\dots,J_L),
\]
where each $J_\ell$ is a Jordan block associated with an eigenvalue $\lambda_\ell$.
Reorder the blocks so that the first $L_0$ blocks correspond exactly to eigenvalues with $|\lambda_\ell|=1$.
Let $U\subseteq \mathbb C^n$ be the span of the columns of $S$ associated with these first $L_0$ blocks.
By construction, $U$ is precisely the generalized invariant subspace spanned by generalized eigenvectors with $|\lambda|=1$,
so $U=\mathcal K_{\mathbb C}(P)$.

Consider the induced operator on the quotient $\mathbb C^n/U$.
In the basis induced by $S$, the action of $P$ is block diagonal, and quotienting by $U$ removes the first $L_0$ blocks.
Hence the induced map $\overline{P}_{\mathbb C}$ on $\mathbb C^n/U$ is similar to the block diagonal matrix
$\mathrm{diag}(J_{L_0+1},\dots,J_L)$, so its spectrum is exactly
$\{\lambda_\ell: \ell>L_0\}$ (counting algebraic multiplicities), and therefore we have
\begin{equation}
\rho(\overline{P}_{\mathbb C})=\max_{\ell>L_0} |\lambda_\ell|.
\end{equation}
All eigenvalues remaining after removing the $|\lambda|=1$ generalized eigenspace satisfy $|\lambda_\ell|<1$,
hence $\rho(\overline{P}_{\mathbb C})<1$.

Finally, the real quotient $\mathbb R^n/\mathcal K(P)$ is the realification of $\mathbb C^n/\mathcal K_{\mathbb C}(P)$,
and $\bar P$ is the corresponding real linear operator. Therefore
$\rho(\bar P)=\rho(\overline{P}_{\mathbb C})<1$, which yields $\rho(\bar P) < 1$.

We now construct the contractive quotient norm. Fix any $\gamma\in(\rho(\bar P),1)$ and define the scaled operator
\[
\bar B := \gamma^{-1}\bar P \quad \text{on } \mathbb R^n/\mathcal K(P).
\]
Then $\rho(\bar B)=\rho(\bar P)/\gamma<1$.

Let $d:=\dim(\mathbb R^n/\mathcal K(P))$ and fix any linear isomorphism
$\Phi:\mathbb R^n/\mathcal K(P)\to \mathbb R^d$.
Let $B\in\mathbb R^{d\times d}$ denote the matrix representation of $\bar B$ under $\Phi$:
$B := \Phi\circ \bar B\circ \Phi^{-1}$.
Then $\rho(B)=\rho(\bar B)<1$.

Define the following discrete-time Lyapunov series \cite{boyd1994linear}
\begin{equation}\label{eq:Lyap_series}
H := \sum_{t=0}^\infty (B^\top)^t B^t \in \mathbb R^{d\times d}.
\end{equation}
Since $\rho(B)<1$, the series converges in any matrix norm and $H$ is symmetric positive definite.
This Lyapunov series construction is standard in discrete-time stability theory.
Moreover, $H$ satisfies the discrete Lyapunov equation
\begin{equation}\label{eq:Lyap_eq}
H - B^\top H B = I_d,
\end{equation}
which follows by multiplying \eqref{eq:Lyap_series} by $B^\top(\cdot)B$ and telescoping.

Now define a norm on the quotient by
\begin{equation}\label{eq:q_norm_def}
\|z\|_{\mathrm q} := \sqrt{ (\Phi z)^\top H (\Phi z)} \qquad \text{for } z\in \mathbb R^n/\mathcal K(P).
\end{equation}
This is a valid norm because $H\succ 0$ and $\Phi$ is an isomorphism.

We claim that $\bar P$ is $\gamma$-contractive under $\|\cdot\|_{\mathrm q}$.
Indeed, for any $z\in \mathbb R^n/\mathcal K(P)$, let $x:=\Phi z\in\mathbb R^d$. Then
\[
\begin{aligned}
\|\bar P z\|_{\mathrm q}^2
&= \| \gamma \bar B z\|_{\mathrm q}^2
 = \gamma^2 \| \bar B z\|_{\mathrm q}^2 \\
&= \gamma^2 (\Phi(\bar B z))^\top H (\Phi(\bar B z))
 = \gamma^2 (Bx)^\top H (Bx) \\
&= \gamma^2 x^\top B^\top H B x .
\end{aligned}
\]
Using \eqref{eq:Lyap_eq}, $B^\top H B = H - I_d \preceq H$, hence
\[
\|\bar P z\|_{\mathrm q}^2 \le \gamma^2 \, x^\top H x = \gamma^2 \|z\|_{\mathrm q}^2,
\]
which gives the quotient contraction.

We now lift the quotient norm back to $\mathbb{R}^n$. Define $\|v\|_{\mathrm q} := \|[v]\|_{\mathrm q}$ for $v\in\mathbb R^n$.
This is a semi-norm because it is the pullback of a norm by the quotient map.
Its kernel is
\[
\ker(\|\cdot\|_{\mathrm q})=\{v:\|[v]\|_{\mathrm q}=0\}=\{v:[v]=0\}=\mathcal K(P),
\]
and for any $v\in\mathbb R^n$ we have
\[
\|Pv\|_{\mathrm q}
= \|[Pv]\|_{\mathrm q}
= \|\bar P([v])\|_{\mathrm q}
\le \gamma \|[v]\|_{\mathrm q}
= \gamma \|v\|_{\mathrm q}.
\]
This proves the pullback contraction and completes the proof of Theorem~\ref{thm:quotient_contraction}.
\end{proof}

\subsection{Proof of Theorem \ref{thm:wellposed}} \label{wellposed}

We start by the following proposition on the semi-norm,
\begin{proposition}\label{lem:Pi_identity}
Let $\Pi$ be a linear projection with $\ker(\Pi)=\mathcal K(P)$.
Then for all $v\in\R^n$, we have $[\,\Pi v\,]=[v]$. In particular, for the quotient semi-norm
$\|v\|_{\mathrm q}:=\|[v]\|_{\mathrm q}$, $\|\Pi v\|_{\mathrm q}=\|v\|_{\mathrm q}$.
\end{proposition}
\begin{proof}
Fix $v\in\R^n$. Since $\Pi$ is a projection, we can write
\[
v-\Pi v = (I-\Pi)v.
\]
Moreover, $(I-\Pi)v\in\ker(\Pi)$ because $\Pi(I-\Pi)=\Pi-\Pi^2=0$. By the assumption $\ker(\Pi)=\mathcal K(P)$,
we obtain $v-\Pi v\in \mathcal K(P)$, i.e., $v$ and $\Pi v$ belong to the same equivalence class in the quotient.
Hence $[\Pi v]=[v]$.

For the semi-norm claim, by definition $\|\Pi v\|_{\mathrm q}=\|[\Pi v]\|_{\mathrm q}=\|[v]\|_{\mathrm q}=\|v\|_{\mathrm q}$.
\end{proof}

Let $W\coloneqq\mathrm{range}(\Pi)$ and let $Q\coloneqq\R^n/\mathcal K(P)$. We first Identify $W$ with the $Q$. Define the linear map $S:W\to Q$ by $S(w)=[w]$.
We claim $S$ is a linear isomorphism.

Injectivity: if $S(w)=0$, then $w\in \mathcal K(P)=\ker(\Pi)$. Since also $w\in W=\mathrm{range}(\Pi)$, we have
$w=\Pi u$ for some $u$ and $0=\Pi w=\Pi^2 u=\Pi u=w$, hence $w=0$.

Surjectivity: given any class $[v]\in Q$, set $w:=\Pi v\in W$. By Proposition~\ref{lem:Pi_identity}, $[\Pi v]=[v]$, hence $S(w)=[v]$. Thus $S$ is bijective.

On $Q$ we have the affine map $\bar T(z)=[r]+\bar P z$.
Because $\rho(\bar P)<1$, the linear operator $I-\bar P$ is invertible on $Q$. Hence the equation $z=\bar T(z)$ has the unique solution
$
z^\star = (I-\bar P)^{-1}[r] \in Q
$. Since $S$ is bijective, we then define $v^\star\coloneqq S^{-1}(z^\star)\in W$. We show $v^\star$ satisfies
\eqref{eq:proj_fixed_point} and is the unique such element in $W$.

First, we verify that $[v^\star]=z^\star$ solves the quotient fixed point equation:
\[
[v^\star] = z^\star = [r] + \bar P z^\star = [r] + \bar P([v^\star]) = [r + Pv^\star].
\]
Using Proposition~\ref{lem:Pi_identity} with $x:=r+Pv^\star$ gives
\[
[\Pi(r+Pv^\star)] = [r+Pv^\star] = [v^\star].
\]
Both $\Pi(r+Pv^\star)$ and $v^\star$ lie in $W$. Since $S$ is injective on $W$, equality of their quotient classes implies
\[
\Pi(r+Pv^\star)=v^\star,
\]
which is exactly \eqref{eq:proj_fixed_point}. This proves existence.

For uniqueness, let $\tilde v\in W$ satisfy $\tilde v=\Pi(r+P\tilde v)$. Applying Proposition~\ref{lem:Pi_identity} yields
\[
[\tilde v] = [\Pi(r+P\tilde v)] = [r+P\tilde v] = [r] + \bar P([\tilde v]),
\]
so $[\tilde v]$ is a fixed point of $\bar T$ on $Q$. By uniqueness of the fixed point of $\bar T$ on $Q$, we have $[\tilde v]=z^\star=[v^\star]$.
Since both $\tilde v$ and $v^\star$ lie in $W$ and $S$ is injective, we conclude $\tilde v=v^\star$.

Let $g^\star:=r+Pv^\star-v^\star$. Applying $\Pi$ and using $\Pi v^\star=v^\star$ gives
\[
\Pi g^\star = \Pi(r+Pv^\star)-\Pi v^\star = v^\star - v^\star = 0,
\]
so $g^\star\in \ker(\Pi)=\mathcal K(P)$.

Finally, \eqref{eq:quotient_fixed_point} holds since $[v^\star]=[r+Pv^\star]=\bar T([v^\star])$.

\subsection{Proof of Corollary \ref{cor:recover_invariant_gain}}
\begin{proof}
Fix $s\in\mathcal S$.
By Lemma~\ref{lem:app_telescoping} applied to $v^\star$ (so that $g_{v^\star}=g^\star$ by \eqref{eq:gstar}; see Section~\ref{justification}), for any horizon $T\ge 1$,
\[
\mathbb E_s\!\left[\sum_{t=0}^{T-1} r(X_t)\right]
=
\sum_{t=0}^{T-1} (P^t g^\star)(s)
+
v^\star(s)
-
(P^T v^\star)(s).
\]
Divide both sides by $T$ and let $T\to\infty$.
Since the state space is finite, $v^\star$ is bounded, hence
$\bigl(v^\star(s)-(P^T v^\star)(s)\bigr)/T\to 0$.
By the definition of the long-run average reward in \eqref{eq:gdef}, the left-hand side converges to $g(s)$.
Therefore,
\[
g(s)
=
\lim_{T\to\infty}\frac1T\sum_{t=0}^{T-1} (P^t g^\star)(s),
\]
which shows that $g^{\mathrm{avg}}$ exists and coincides with $g(\cdot)$ componentwise.

It remains to show $g^{\mathrm{avg}}\in\ker(I-P)$.
Let $A_T:=\frac1T\sum_{t=0}^{T-1}P^t$.
Then $P A_T = A_T + \frac{P^T-I}{T}$, hence
\[
\|P A_T g^\star - A_T g^\star\|_\infty
=
\left\|\frac{(P^T-I)g^\star}{T}\right\|_\infty
\le \frac{2\|g^\star\|_\infty}{T}\to 0.
\]
Passing to the limit yields $P g^{\mathrm{avg}}=g^{\mathrm{avg}}$, i.e., $g^{\mathrm{avg}}\in\ker(I-P)$.
\end{proof}

\subsection{Proof of Proposition \ref{prop:aperiodic_reduction}}

\begin{proof}
Write the finite state chain in its standard transient and recurrent block form. The recurrent blocks are the closed communicating classes. For each recurrent class, the peripheral spectrum of the corresponding stochastic block consists of the roots of unity determined by its period. If the class is aperiodic, then its period is one, so the only peripheral eigenvalue is $1$.

Therefore, if every recurrent class is aperiodic, the only eigenvalue of $P$ on the unit circle is $1$. It follows that the real peripheral invariant subspace coincides with the invariant subspace:
\[
\mathcal K(P)=\ker(I-P).
\]
Hence the decomposition in Theorem~\ref{thm:wellposed} reduces to the usual gain and bias decomposition up to the choice of gauge.
\end{proof}

\subsection{Proof of Theorem \ref{thm:comparison_gain_bias}}

\begin{proof}
Let $(g_\Pi^\star,v_\Pi^\star)$ be the decomposition from Theorem~\ref{thm:wellposed}. By definition, we have $r = g_\Pi^\star + (I-P)v_\Pi^\star$. Let $(\rho,h)$ be the normalized classical gain and bias pair, so
\[
r=\rho+(I-P)h,
\qquad
\rho=P^\infty r,
\qquad
P^\infty h=0.
\]

Apply $P^\infty$ to the decomposition $r = g_\Pi^\star + (I-P)v_\Pi^\star$. Since $P^\infty(I-P)=0$, we obtain $P^\infty r = P^\infty g_\Pi^\star$. Because $\rho=P^\infty r$, this gives $P^\infty g_\Pi^\star = \rho$.

We already know from Theorem~\ref{thm:wellposed} that $g_\Pi^\star\in\mathcal K(P)$. Thus,
\[
P^\infty(g_\Pi^\star-\rho)=0.
\]
Also, $\rho\in\ker(I-P)\subseteq \mathcal K(P)$, so
\[
g_\Pi^\star-\rho \in \mathcal K(P)\cap \ker(P^\infty):=\mathcal K_\circ(P).
\]

Since $\mathcal K_\circ(P)$ is $P$ invariant and contains no invariant directions, the operator $I-P$ is injective on $\mathcal K_\circ(P)$. Because $\mathcal K_\circ(P)$ is finite dimensional, injectivity implies bijectivity. Therefore there exists a unique vector $\psi_\Pi\in\mathcal K_\circ(P)$ such that
\[
(I-P)\psi_\Pi = g_\Pi^\star-\rho,
\]
which also means $g_\Pi^\star = \rho + (I-P)\psi_\Pi$. Substitute the identity above into the decomposition:
\[
r
=
g_\Pi^\star + (I-P)v_\Pi^\star
=
\rho + (I-P)\psi_\Pi + (I-P)v_\Pi^\star
=
\rho + (I-P)(v_\Pi^\star+\psi_\Pi).
\]
Compare this with the classical equation $r=\rho+(I-P)h$ and subtracting the two equations gives
\[
(I-P)\bigl(h-v_\Pi^\star-\psi_\Pi\bigr)=0.
\]
Hence, we have $h-v_\Pi^\star-\psi_\Pi \in \ker(I-P)$. Apply $P^\infty$ to this identity. Since $P^\infty h=0$ and $\psi_\Pi\in\mathcal K_\circ(P)\subseteq \ker(P^\infty)$, we obtain
\[
P^\infty(h-v_\Pi^\star-\psi_\Pi)
=
-\,P^\infty v_\Pi^\star.
\]
Because the left-hand side equals $h-v_\Pi^\star-\psi_\Pi$, we conclude
\[
h-v_\Pi^\star-\psi_\Pi = -P^\infty v_\Pi^\star.
\]
Thus we have $h = v_\Pi^\star + \psi_\Pi - P^\infty v_\Pi^\star$.
\end{proof}

\subsection{Proof of Corollary \ref{cor:projected_bias}}

\begin{proof}
Start from the identity in Theorem~\ref{thm:comparison_gain_bias}:
\[
h = v_\Pi^\star + \psi_\Pi - P^\infty v_\Pi^\star.
\]
Apply $\Pi$ to both sides. Since $v_\Pi^\star\in\mathrm{range}(\Pi)$, we have $\Pi v_\Pi^\star=v_\Pi^\star$. Since $\psi_\Pi\in\mathcal K_\circ(P)\subseteq \mathcal K(P)=\ker(\Pi)$, we have $\Pi\psi_\Pi=0$. Also, $P^\infty v_\Pi^\star\in \ker(I-P)\subseteq \mathcal K(P)=\ker(\Pi)$, so $\Pi P^\infty v_\Pi^\star=0$. Therefore
\[
\Pi h = v_\Pi^\star.
\]
This proves the first identity.

For the second identity, use the first identity together with
\[
g_\Pi^\star = r + Pv_\Pi^\star - v_\Pi^\star
\]
and the classical equation
\[
r = \rho + (I-P)h.
\]
Substituting $v_\Pi^\star=\Pi h$ gives
\[
g_\Pi^\star
=
\rho + (I-P)h + P\Pi h - \Pi h
=
\rho + (I-P)(I-\Pi)h.
\]
This proves the second identity.
\end{proof}

\section{Proofs for Appendix~\ref{app:sample_theory}}
\subsection{Proofs in Section \ref{Learning the gauge map}} \label{proof_learning_gauge}

We first formally define several notations. Let $E=\{(i,j):P_{ij}>0\}$ be the directed support. Then we can redefine the intrinsic edge rarity parameter as
\begin{equation}\label{eq:pmin}
p_{\min} := \min_{(i,j)\in E} P_{ij}.
\end{equation}
If $p_{\min}=0$, exact support recovery is impossible from finite samples; all guarantees below hold conditional
on recovering the support graph, or stated in terms of $p_{\min}$ when $p_{\min}>0$.

\begin{lemma}\label{lem:recurrent_decomp}
Let $G=(S,E)$ be the directed support graph of $P$ (i.e., $(i,j)\in E \iff P_{ij}>0$).
Then the closed strongly connected components of $G$ are exactly the recurrent communicating
classes $\{F_i\}_{i=1}^m$. Their union $\Fset=\cup_i F_i$ is the set of recurrent states and
$\Tset=S\setminus\Fset$ is the set of transient states.
\end{lemma}
\begin{proof}
Let $G=(S,E)$ be the directed support graph, so $(i,j)\in E$ if and only if $P_{ij}>0$.

First, let $C\subseteq S$ be a closed strongly connected component of $G$.
Closed means there is no edge from $C$ to $S\setminus C$, hence $P_{ij}=0$ for all $i\in C$ and $j\notin C$.
Therefore the Markov chain started in $C$ stays in $C$ almost surely.
Since $C$ is strongly connected, the restriction of $P$ to $C$ is irreducible on a finite state space, hence every state in $C$ is recurrent.
Thus each closed SCC is a recurrent communicating class.

Conversely, let $F$ be a recurrent communicating class.
If there were an edge $(i,j)\in E$ with $i\in F$ and $j\notin F$, then $P_{ij}>0$ and, by definition of communicating classes, there is no directed path from $j$ back to $i$ in $G$.
Hence with probability at least $P_{ij}$ the chain leaves $F$ and never returns to $i$, contradicting recurrence of $i$.
Thus no such edge exists and $F$ is closed in $G$.
Since $F$ is a communicating class, it is strongly connected, hence it is a closed SCC.

Therefore the closed SCCs of $G$ are exactly the recurrent communicating classes.
Their union is the recurrent set $\Fset$ and the complement is the transient set $\Tset$.
\end{proof}

\begin{theorem}[Formal version of Theorem \ref{thm:main_support_recovery}]\label{thm:identify}
Suppose we have a generative model that can sample $K$ i.i.d.\ next-states
from $P(i,\cdot)$ for every $i\in S$. Let $\widehat E$ include edge $(i,j)$ iff $j$ appears at least once among the $K$ samples from state $i$.
If $p_{\min}>0$ and
\begin{equation}\label{eq:Kidentify}
K \;\ge\; \frac{1}{p_{\min}}\log\frac{n^2}{\delta},
\end{equation}
then with probability at least $1-\delta$, $\widehat E = E$ and hence $\widehat G$ equals the true support graph.
Consequently, the closed communicating classes $\{F_i\}$ are correctly recovered, along with their periods $d_i$
and cyclic partitions $\{C_{i,k}\}$ (up to cyclic relabeling within each class).
\end{theorem}
\begin{proof}
If $P_{ij}=0$, then $j$ is sampled from $P(i,\cdot)$ with probability $0$, hence
$(i,j)\notin \widehat E$ almost surely. Therefore $\widehat E \subseteq E$ always, meaning there are no false positives. We now consider a fixed edge $(i,j)\in E$, so $P_{ij}>0$. The probability that $j$ never appears
among the $K$ i.i.d.\ samples from $P(i,\cdot)$ equals
\begin{equation}
\text{Pr}\bigl( (i,j)\notin \widehat E \bigr)
= (1-P_{ij})^K
\le e^{-K P_{ij}}
\le e^{-K p_{\min}},
\end{equation}
with $1-x\le e^{-x}$ and $P_{ij}\ge p_{\min}$.

By a union bound over all possible pairs $(i,j)\in S\times S$ (at most $n^2$ pairs),
\begin{equation}
\text{Pr}(\widehat E \neq E)
=
\text{Pr}(\exists (i,j)\in E:\ (i,j)\notin \widehat E)
\le
\sum_{(i,j)\in E} \text{Pr}\bigl( (i,j)\notin \widehat E \bigr)
\le
n^2 e^{-K p_{\min}}.
\end{equation}
If $K \ge \frac{1}{p_{\min}}\log\frac{n^2}{\delta}$, then $n^2 e^{-K p_{\min}} \le \delta$,
so $\text{Pr}(\widehat E = E)\ge 1-\delta$.

On the event $\widehat E=E$, the estimated support graph $\widehat G=(S,\widehat E)$ equals the true support graph $G=(S,E)$.
By Lemma~\ref{lem:recurrent_decomp}, the closed SCCs of $G$ are exactly the recurrent communicating classes $\{F_i\}_{i=1}^m$.
Therefore $\Fset$ and $\Tset$ are recovered exactly.

Fix a recovered recurrent class $F_i$.
Its period $d_i$ depends only on the directed edges inside the induced subgraph $G|_{F_i}$.
Equivalently, $d_i$ is the greatest common divisor of lengths of directed cycles in $G|_{F_i}$, so it is determined by the support graph.
Since $\widehat G|_{F_i}=G|_{F_i}$ on the event $\widehat E=E$, the period $d_i$ is recovered correctly.

Given $d_i$, the standard cyclic decomposition yields a partition
$F_i=C_{i,0}\cup\cdots\cup C_{i,d_i-1}$ such that all transitions from $C_{i,k}$ go into $C_{i,k+1}$, where indices are modulo $d_i$.
This cyclic partition is unique up to a cyclic relabeling of the indices.
Hence the cyclic partition is also recovered correctly.
\end{proof}

We use the peripheral basis functions $\{b_{i,k}\}_{(i,k)\in\mathcal I}$ introduced in
Definition~\ref{def:main_b} of the main paper. 
\begin{lemma}\label{lem:shift_b}
Fix $(i,k)\in\cI$ and define $b_{i,k}$ as in Definition~\ref{def:main_b}.
Then for all $s\in S$,
\[
(P b_{i,k})(s) \;=\; b_{i,k-1 \bmod d_i}(s),
\]
where the index addition is taken modulo $d_i$ within class $i$ (and $b_{i,k}\equiv 0$ outside class $i$ on $\Fset$).
In particular, $\mathrm{span}\{b_{i,k}\}_{(i,k)\in\cI}$ is $P$-invariant and carries a block-cyclic action.
\end{lemma}
\begin{proof}
Fix a recurrent class index $i\in[m]$ and $k\in\{0,\dots,d_i-1\}$.
Let $b_{i,k}$ be defined as in Definition~\ref{def:main_b}.
For any state $s\in S$,
\begin{equation}
(P b_{i,k})(s) =\sum_{s'\in S} P(s,s')\, b_{i,k}(s')
= \mathbb E_s\!\left[b_{i,k}(X_1)\right].
\end{equation}
We prove that $(P b_{i,k})(s)=b_{i,k-1\bmod d_i}(s)$ by considering cases.

\paragraph{Case 1: $s\in \mathsf F$ and $s\in F_i$.}
Let $\ell:=\phase(s)$, so $s\in C_{i,\ell}$.
By the cyclic decomposition of the periodic class $F_i$, we have
$P(s,\cdot)$ supported on $C_{i,\ell+1\bmod d_i}$.
Hence,
\[
(P b_{i,k})(s)
= \sum_{s'\in C_{i,\ell+1}} P(s,s') \bbm{1}\{s'\in C_{i,k}\}
= \bbm{1}\{\ell+1 \equiv k \!\!\!\pmod{d_i}\}
= \bbm{1}\{\ell \equiv k-1 \!\!\!\pmod{d_i}\}.
\]
Since $b_{i,k-1}(s)=\bbm{1}\{s\in C_{i,k-1}\}=\bbm{1}\{\ell\equiv k-1\!\!\pmod{d_i}\}$,
this shows $(P b_{i,k})(s)=b_{i,k-1}(s)$.

\paragraph{Case 2: $s\in \mathsf F$ and $s\notin F_i$.}
Then $s\in C_{j,\ell}$ for some $j\neq i$ and some $\ell$.
By Definition~\ref{def:main_b}, $b_{i,k}(s)=\bbm{1}\{s\in C_{i,k}\}=0$.
Also, since $F_j$ is closed, $P(s,\cdot)$ is supported on $F_j$ and therefore on $\mathsf F\setminus F_i$.
But $b_{i,k}(\cdot)$ equals $0$ on $\mathsf F\setminus F_i$, so
\begin{equation}
(P b_{i,k})(s)=\sum_{s'\in S}P(s,s')b_{i,k}(s')=0=b_{i,k-1}(s).
\end{equation}

\paragraph{Case 3: $s\in \mathsf T$ .}
Let $\tau_{\mathsf F}:=\min\{t\ge 0: X_t\in \mathsf F\}$ denote the hitting time of $\mathsf F$
when starting from $X_0=s$. Since $s$ is transient in a finite Markov chain, $\tau_{\mathsf F}\ge 1$ almost surely.

Define the shifted hitting time from $X_1$ as $\tau_{\mathsf F}^{+} := \min\{t\ge 0: X_{t+1}\in\mathsf F\}$. On the event $\{\tau_{\mathsf F}\ge 1\}$, we have the identities $\tau_{\mathsf F}^{+} = \tau_{\mathsf F}-1$,
and $X_{1+\tau_{\mathsf F}^{+}} = X_{\tau_{\mathsf F}}$. We further compute $(P b_{i,k})(s)=\mathbb E_s[b_{i,k}(X_1)]$.

By Definition~\ref{def:main_b}, for the possibly transient state $X_1$,
\begin{equation}
b_{i,k}(X_1)
=
\text{Pr}\!\left(
\class(X_{1+\tau_{\mathsf F}^{+}})=i,\ 
(\phase(X_{1+\tau_{\mathsf F}^{+}})-\tau_{\mathsf F}^{+})\bmod d_i=k
| X_1
\right).
\end{equation}
Using the Markov property and the identities above,
\begin{equation}
(\phase(X_{1+\tau_{\mathsf F}^{+}})-\tau_{\mathsf F}^{+})\bmod d_i
=
(\phase(X_{\tau_{\mathsf F}})-(\tau_{\mathsf F}-1))\bmod d_i
=
(\phase(X_{\tau_{\mathsf F}})-\tau_{\mathsf F}+1)\bmod d_i.
\end{equation}
Therefore the condition
$(\phase(X_{1+\tau_{\mathsf F}^{+}})-\tau_{\mathsf F}^{+})\bmod d_i=k$
is equivalent to
$(\phase(X_{\tau_{\mathsf F}})-\tau_{\mathsf F})\bmod d_i = k-1$.
Hence,
\begin{equation}
b_{i,k}(X_1)
=
\text{Pr}\!\left(
\class(X_{\tau_{\mathsf F}})=i,\ 
(\phase(X_{\tau_{\mathsf F}})-\tau_{\mathsf F})\bmod d_i=k-1
| X_1
\right).
\end{equation}
Taking expectation over $X_1$ yields
\begin{equation}
(P b_{i,k})(s)
=
\text{Pr}_s\!\left(
\class(X_{\tau_{\mathsf F}})=i,\ 
(\phase(X_{\tau_{\mathsf F}})-\tau_{\mathsf F})\bmod d_i=k-1
\right)
=
b_{i,k-1}(s),
\end{equation}
where the last equality is exactly Definition~\ref{def:main_b} for transient $s$.

Based on the above, all cases give $(P b_{i,k})(s)=b_{i,k-1\bmod d_i}(s)$ for every $s\in S$.
Therefore, $\mathrm{span}\{b_{i,k}\}_{(i,k)\in\mathcal I}$ is $P$-invariant and the action of $P$
on this span is block-cyclic within each recurrent class $i$.
\end{proof}

\begin{lemma}\label{lem:Pi_kernel}
Let $\{b_{i,k}\}_{(i,k)\in\cI}$ be as in Definition~\ref{def:main_b}, and fix anchors $\{a_{i,k}\}$ with $a_{i,k}\in C_{i,k}$.
Define $\Pi:\R^n\to\R^n$ by
\begin{equation}
(\Pi v)(s)=v(s)-\sum_{(i,k)\in\cI} v(a_{i,k})\,b_{i,k}(s).
\end{equation}
Then:
\begin{enumerate}
\item $\Pi$ is a linear projection ($\Pi^2=\Pi$), and $(\Pi v)(a_{i,k})=0$ for all $(i,k)$.
\item $\ker(\Pi)=\mathrm{span}\{b_{i,k}\}_{(i,k)\in\cI}$ and $\dim(\ker(\Pi))=|\cI|=\sum_{i=1}^m d_i$.
\end{enumerate}
\end{lemma}
\begin{proof}[Proof of Lemma~\ref{lem:Pi_kernel}]

Since $a_{i,k}\in C_{i,k}\subset \Fset$ and $b_{j,\ell}(s)=\bbm{1}\{s\in C_{j,\ell}\}$ on $\Fset$,
we have for all $(i,k),(j,\ell)\in\cI$:
\begin{equation}\label{eq:delta_anchor}
b_{j,\ell}(a_{i,k})=\bbm{1}\{a_{i,k}\in C_{j,\ell}\}=\bbm{1}\{(j,\ell)=(i,k)\}.
\end{equation}

We first note that the linearity of $\Pi$ is immediate from its definition.
For any $(i,k)\in\cI$,
\begin{equation}
(\Pi v)(a_{i,k})
= v(a_{i,k}) - \sum_{(j,\ell)\in\cI} v(a_{j,\ell})\, b_{j,\ell}(a_{i,k})
= v(a_{i,k}) - \sum_{(j,\ell)\in\cI} v(a_{j,\ell}) \bbm{1}\{(j,\ell)=(i,k)\}=0,
\end{equation}
using~\eqref{eq:delta_anchor}. Hence $\Pi v$ satisfies $(\Pi v)(a_{i,k})=0$ for all anchors.

To show $\Pi^2=\Pi$, fix any $v$ and write
\begin{equation}
(\Pi^2 v)(s) = (\Pi(\Pi v))(s)
= (\Pi v)(s) - \sum_{(i,k)\in\cI} (\Pi v)(a_{i,k})\, b_{i,k}(s)
= (\Pi v)(s),
\end{equation}
because we already proved $(\Pi v)(a_{i,k})=0$ for every $(i,k)$. Thus $\Pi$ is a projection.

For any coefficients $\{\theta_{i,k}\}$, let $v=\sum_{(i,k)\in\cI}\theta_{i,k} b_{i,k}$.
Then by~\eqref{eq:delta_anchor}, $v(a_{i,k})=\theta_{i,k}$, and therefore
\begin{equation}
(\Pi v)(s)
= \sum_{(i,k)}\theta_{i,k} b_{i,k}(s) - \sum_{(i,k)} v(a_{i,k}) b_{i,k}(s)
= \sum_{(i,k)}\theta_{i,k} b_{i,k}(s) - \sum_{(i,k)} \theta_{i,k} b_{i,k}(s) = 0.
\end{equation}
So $\mathrm{span}\{b_{i,k}\}\subseteq \ker(\Pi)$.

Conversely, if $\Pi v=0$, then for every $s\in S$,
\begin{equation}
v(s) = \sum_{(i,k)\in\cI} v(a_{i,k})\, b_{i,k}(s),
\end{equation}
which shows $v\in \mathrm{span}\{b_{i,k}\}$. Hence $\ker(\Pi)=\mathrm{span}\{b_{i,k}\}_{(i,k)\in\cI}$.

Finally, $\{b_{i,k}\}_{(i,k)\in\cI}$ is linearly independent: if $\sum_{(i,k)}\theta_{i,k} b_{i,k}=0$,
evaluate at $s=a_{j,\ell}$ and use~\eqref{eq:delta_anchor} to get $\theta_{j,\ell}=0$ for all $(j,\ell)$.
Thus $\dim(\ker(\Pi))=|\cI|=\sum_{i=1}^m d_i$.
\end{proof}

\begin{lemma}[Formal version of Lemma \ref{lem:main_b_identify}]\label{lem:kernel_equals_peripheral}
The $P$-invariant subspace
$\mathrm{span}\{b_{i,k}\}_{(i,k)\in\cI}$ coincides with the real peripheral invariant subspace $\mathcal{K}(P)$.
Consequently, $\ker(\Pi)=\mathcal{K}(P)$.
\end{lemma}
\begin{proof}
Let
\[
\mathcal B \;\coloneqq\; \mathrm{span}\{b_{i,k}\}_{(i,k)\in\cI}\subseteq \mathbb R^n.
\]

We first show $\mathcal B\subseteq \mathcal K(P)$.
By Lemma~\ref{lem:shift_b},
\begin{equation}\label{eq:Pb_shift}
P b_{i,k} \;=\; b_{i,k-1\bmod d_i}
\qquad\text{for each }(i,k)\in\cI,
\end{equation}
and $b_{i,k}\equiv 0$ on $\Fset\setminus F_i$.

Fix a class $i$. Let $\omega_i\coloneqq e^{2\pi \mathrm i/d_i}$ and define, for $\ell=0,1,\dots,d_i-1$,
the complex ``Fourier mode''
\[
\psi_{i,\ell}\;\coloneqq\;\sum_{k=0}^{d_i-1} \omega_i^{\ell k}\, b_{i,k}\;\in\;\mathbb C^n.
\]
Using~\eqref{eq:Pb_shift} and a change of index,
\begin{equation}
P\psi_{i,\ell}
= \sum_{k=0}^{d_i-1}\omega_i^{\ell k}\, P b_{i,k}
= \sum_{k=0}^{d_i-1}\omega_i^{\ell k}\, b_{i,k-1}
= \sum_{j=0}^{d_i-1}\omega_i^{\ell (j+1)}\, b_{i,j}
= \omega_i^{\ell}\, \psi_{i,\ell}.
\end{equation}
Thus each $\psi_{i,\ell}$ is an eigenvector of $P$ with eigenvalue $\omega_i^\ell$, and $|\omega_i^\ell|=1$.
Therefore $\psi_{i,\ell}\in \mathcal K_{\mathbb C}(P)$ for all $i,\ell$.
Since $b_{i,k}$ is a real linear combination of $\{\psi_{i,\ell}\}_{\ell=0}^{d_i-1}$
(invert the discrete Fourier transform), we conclude that each $b_{i,k}$ lies in the realification
$\mathcal K(P)$, hence $\mathcal B\subseteq \mathcal K(P)$.

We now show that $\dim(\mathcal K(P))=|\cI|$.
Permute the states so that transients come first and recurrent classes are grouped:
\[
P \;=\;
\begin{pmatrix}
Q & R\\
0 & J
\end{pmatrix},
\qquad
J=\mathrm{diag}(P_1,\dots,P_m),
\]
where $Q$ is the transient-to-transient submatrix and each $P_i$ is the transition matrix
restricted to recurrent class $F_i$.

Because $Q$ corresponds to transient states in a finite chain, we have $\rho(Q)<1$; in particular,
$Q$ has no eigenvalues on the unit circle. Since $P$ is block upper-triangular, the multiset of eigenvalues
of $P$ equals the union of the eigenvalues of $Q$ and those of the blocks $P_i$.
Hence all eigenvalues of $P$ with $|\lambda|=1$ come from the recurrent blocks $\{P_i\}$.

Now fix $i$. The block $P_i$ is row-stochastic, irreducible, and has period $d_i$.
A standard Perron--Frobenius periodicity result implies:
the eigenvalues of $P_i$ on the unit circle are exactly $\{e^{2\pi \mathrm i \ell/d_i}:\ell=0,\dots,d_i-1\}$,
each with algebraic (and geometric) multiplicity $1$. Consequently, we have $\dim\big(\mathcal K_{\mathbb C}(P)\big)=\sum_{i=1}^m d_i=|\cI|$.

Finally, since the peripheral eigenvalues occur in complex-conjugate pairs and the peripheral spectrum
is semisimple for stochastic $P$ (no Jordan blocks on $|\lambda|=1$), the realification does not change
the total dimension of the corresponding real invariant subspace; hence
\[
\dim\big(\mathcal K(P)\big)\;=\;|\cI|.
\]

Since we have $\mathcal B\subseteq \mathcal K(P)$, and by Lemma~\ref{lem:Pi_kernel},
$\dim(\mathcal B)=|\cI|$. Step~2 gives $\dim(\mathcal K(P))=|\cI|$.
Therefore $\mathcal B=\mathcal K(P)$, and together with Lemma~\ref{lem:Pi_kernel},
\[
\ker(\Pi)=\mathcal B=\mathcal K(P).\qedhere
\]
\end{proof}

\begin{remark}[Semisimplicity of the peripheral spectrum]
Since $P$ is row-stochastic, $\|P^t\|_\infty\le 1$ for all $t\ge 0$. Hence all eigenvalues with $|\lambda|=1$
are semisimple (no Jordan blocks), and $\mathcal K_{\mathbb C}(P)$ equals the direct sum of eigenspaces for $|\lambda|=1$.
\end{remark}

\begin{lemma}[Formal version of Lemma \ref{lem:main_weight_conc}]\label{lem:bhats}
Assume Theorem~\ref{thm:identify} holds (so $\{F_i\},\{C_{i,k}\}$ are known).
For each transient state $s\in\Tset$, estimate $\bhat(s)=(\bhat_{i,k}(s))_{(i,k)\in\cI}$ by $M$ i.i.d.\ absorption episodes as in Algorithm~\ref{alg:main_estimate_weights}.
Let $N:=|\cI|=\sum_{i=1}^m d_i$.
Then for any $\epsb\in(0,1)$, if
\[
M \;\ge\; \frac{8}{\epsb^2}\Bigl( N + \log\frac{|\Tset|}{\delta}\Bigr),
\]
we have with probability at least $1-\delta$:
\begin{equation}
\max_{s\in\Tset}\|\bhat(s)-b(s)\|_1 \le \epsb.
\end{equation}
\end{lemma}

\begin{proof}[Proof of Lemma~\ref{lem:bhats}]
Fix a transient state $s\in\Tset$.
One absorption episode from $s$ produces a trajectory $(X_t)_{t\ge 0}$ until
$\tau_{\Fset}=\min\{t\ge 0:X_t\in\Fset\}$ and yields a terminal index
$Y\in\cI$ (the recovered recurrent class and phase).
By definition, $Y$ is a categorical random variable with distribution $b(s)=(b_{i,k}(s))_{(i,k)\in\cI}$, and $\bhat(s)$ is the empirical distribution of $M$ i.i.d.\ copies of $Y$.

Define $f_s:=\|\bhat(s)-b(s)\|_1$ as a function of the $M$ samples.
Changing one sample changes $\bhat(s)$ by at most $2/M$ in $\ell_1$, hence $f_s$ has bounded differences with constants $2/M$.
By McDiarmid's inequality, for any $t>0$,
\[
\Pr\!\left(f_s-\E f_s\ge t\right)\le \exp\!\left(-\frac{M t^2}{2}\right).
\]
Also, by Cauchy--Schwarz and Jensen,
\[
\E f_s
\le \sqrt{N}\,\E\|\bhat(s)-b(s)\|_2
\le \sqrt{N}\,\sqrt{\E\|\bhat(s)-b(s)\|_2^2}.
\]
A direct computation for the multinomial estimator gives
\[
\E\|\bhat(s)-b(s)\|_2^2
=
\sum_{(i,k)\in\cI} \text{Var} \big(\bhat_{i,k}(s)\big)
=
\frac{1-\|b(s)\|_2^2}{M}
\le \frac{1}{M},
\]
so $\E f_s\le \sqrt{N/M}$.

Now set $t=\sqrt{\frac{2\log(|\Tset|/\delta)}{M}}$ and take a union bound over all $s\in\Tset$ to obtain
\[
\max_{s\in\Tset} f_s
\le
\sqrt{\frac{N}{M}}+\sqrt{\frac{2\log(|\Tset|/\delta)}{M}}
\]
with probability at least $1-\delta$.
The stated sufficient condition on $M$ makes the right-hand side at most $\epsb$.
\end{proof}

\begin{corollary}[Choosing $M$ to ensure $\widehat\gamma<1$]\label{cor:M_for_contraction}
Fix $\gamma\in(0,1)$ and constants $C_{{\mathrm q},\infty}$ and $C_{\infty,{\mathrm q}}$.
Set
\[
\varepsilon_b^{\mathrm{tar}}
:=
\frac{1-\gamma}{2\,C_{{\mathrm q},\infty}\,C_{\infty,{\mathrm q}}}.
\]
If Lemma~\ref{lem:bhats} holds with $\varepsilon_b\le \varepsilon_b^{\mathrm{tar}}$, then we have $\widehat\gamma \le \frac{1+\gamma}{2}<1$.
In particular, by Lemma~\ref{lem:bhats}, it suffices to choose
\[
M \;\ge\; \frac{8}{(\varepsilon_b^{\mathrm{tar}})^2}\Bigl( N + \log\frac{|\Tset|}{\delta}\Bigr)
\]
to guarantee $\widehat\gamma<1$ on the event of Lemma~\ref{lem:bhats}.
\end{corollary}

\subsection{Proofs in Section \ref{projSA}} \label{proof_proj_SA}

\begin{lemma}\label{lem:optionB_stability}
Work on the event of Theorem~\ref{thm:identify} and Lemma~\ref{lem:bhats}.
 Under Algorithm \ref{alg:main_psa} with the recursion $v_{t+1} = \Pihat\Bigl((1-\alpha_t)v_t + \alpha_t\,\widehat T(v_t)\Bigr)$ and $\alpha_t=\alpha/(t+1)$. Assume the oracle noise along the iterates satisfies the uniform conditional second-moment bound
\begin{equation}\label{eq:oracleA_along_iterates}
\E\!\left[\|\widehat T(v_t)-(r+Pv_t)\|_\infty^2 | \cF_t\right]\le \sigma^2
\qquad\text{for all }t,
\end{equation}
where $\cF_t=\sigma(v_0,\dots,v_t)$.
Then there exists a finite constant $C_{\mathrm{stab}}$ (depending on $\alpha,\gamma,N,\|r\|_\infty$, and norm-equivalence constants)
such that
\begin{equation}
    \sup_{t\ge 0}\E\|v_t-v^\star\|_{\mathrm q}^2 \leq C_{\mathrm{stab}}.
\end{equation}

Consequently, by norm equivalence on $\mathrm{range}(\Pi)$, there is a finite constant $C_{\infty,q}$ such that
\begin{equation}
\sup_{t\ge 0}\E\|v_t\|_\infty^2
\leq
2C_{\infty,q}^2\,C_{\mathrm{stab}} + 2\|v^\star\|_\infty^2
\eqqcolon B_{\sup}^2 <\infty.
\end{equation}
\end{lemma}
\begin{proof}[Proof of Lemma~\ref{lem:optionB_stability}]
Work on the event
\begin{equation}
\mathsf E \;:=\;\Bigl(\text{Theorem~\ref{thm:identify} holds}\Bigr)\cap\Bigl(\text{Lemma~\ref{lem:bhats} holds}\Bigr),
\end{equation}
so that in particular the anchor exactness and transient accuracy hold:
\begin{equation}
\bhat_{j,\ell}(a_{i,k}) = b_{j,\ell}(a_{i,k})=\bbm{1}\{(j,\ell)=(i,k)\},
\qquad
\max_{s\in\Tset}\|\bhat(s)-b(s)\|_1\le \epsb.
\end{equation}
Recall the projections
\begin{equation}
(\Pi v)(s)=v(s)-\sum_{(i,k)\in\cI} v(a_{i,k})\,b_{i,k}(s),
\qquad
(\Pihat v)(s)=v(s)-\sum_{(i,k)\in\cI} v(a_{i,k})\,\bhat_{i,k}(s),
\end{equation}
and define the anchored subspace
\begin{equation}
W \;:=\;\Bigl\{v\in\mathbb R^n:\ v(a_{i,k})=0\;\;\forall (i,k)\in\cI\Bigr\}.
\end{equation}

Using anchor exactness, for any \(v\) and any anchor \(a_{i,k}\),
\begin{equation}
(\Pihat v)(a_{i,k})
=
v(a_{i,k})-\sum_{(j,\ell)\in\cI} v(a_{j,\ell})\,\bhat_{j,\ell}(a_{i,k})
=
v(a_{i,k})-v(a_{i,k})
=
0.
\end{equation}
Hence \(\mathrm{range}(\Pihat)\subseteq W\). Conversely, if \(w\in W\), then all anchor values vanish and thus \(\Pihat w=w\), implying \(W\subseteq \mathrm{range}(\Pihat)\). Therefore,
\begin{equation}
\mathrm{range}(\Pihat)=W
\qquad\text{and}\qquad
\Pihat|_{W}=I_{W}.
\end{equation}
The same holds for \(\Pi\). In particular, for all \(t\ge 1\),
\begin{equation}
v_t\in W,
\qquad
v^\star\in W,
\qquad
e_t:=v_t-v^\star\in W.
\end{equation}

We now formulate the error recursion. Let
\begin{equation}
T(v):=r+Pv,
\qquad
\xi_{t+1}:=\widehat T(v_t)-T(v_t).
\end{equation}
By the oracle condition, \(\E[\xi_{t+1}|\cF_t]=0\) and \(\E[\|\xi_{t+1}\|_\infty^2|\cF_t]\le \sigma^2\).
Define \(\eta_{t+1}:=\Pihat\xi_{t+1}\in W\). Since \(\Pihat\) is \(\cF_t\)-measurable and linear,
\begin{equation}
\E[\eta_{t+1}|\cF_t]=0.
\end{equation}
The update is
\begin{equation}
v_{t+1}=\Pihat\bigl((1-\alpha_t)v_t+\alpha_t\widehat T(v_t)\bigr)
=
(1-\alpha_t)v_t+\alpha_t\Pihat T(v_t)+\alpha_t\eta_{t+1}.
\end{equation}
Subtract \(v^\star=\Pi T(v^\star)\) and add/subtract \(\Pi T(v_t)\) to obtain
\begin{equation}\label{eq:stab_recursion}
e_{t+1}
=
(1-\alpha_t)e_t+\alpha_t\Pi(Pe_t)
+\alpha_t\delta_t+\alpha_t\eta_{t+1},
\end{equation}
where the projection-mismatch term is
\begin{equation}
\delta_t:=(\Pihat-\Pi)T(v_t)=(\Pihat-\Pi)(r+Pv_t)\in W.
\end{equation}

Let \(D:=1-\gamma\in(0,1)\). By Theorem~\ref{thm:quotient_contraction} and the definition of \(\|\cdot\|_{\mathrm q}\),
\begin{equation}
\|\Pi(Pe_t)\|_{\mathrm q}\le \gamma \|e_t\|_{\mathrm q},
\end{equation}
and hence by the triangle inequality,
\begin{equation}\label{eq:stab_contraction}
\bigl\|(1-\alpha_t)e_t+\alpha_t\Pi(Pe_t)\bigr\|_{\mathrm q}
\le (1-\alpha_t)\|e_t\|_{\mathrm q}+\alpha_t\gamma\|e_t\|_{\mathrm q}
=
\bigl(1-\alpha_tD\bigr)\|e_t\|_{\mathrm q}.
\end{equation}

Let
\begin{equation}
u_t:=(1-\alpha_t)e_t+\alpha_t\Pi(Pe_t)\in W.
\end{equation}
Then \eqref{eq:stab_recursion} becomes \(e_{t+1}=u_t+\alpha_t\delta_t+\alpha_t\eta_{t+1}\).
Since \(\|\cdot\|_{\mathrm q}\) is induced by an inner product (by construction in Theorem~\ref{thm:quotient_contraction}),
the martingale cross term cancels:
\begin{equation}
\E\bigl[\langle u_t+\alpha_t\delta_t,\eta_{t+1}\rangle_{\mathrm q}| \cF_t\bigr]
=
\Bigl\langle u_t+\alpha_t\delta_t,\E[\eta_{t+1}|\cF_t]\Bigr\rangle_{\mathrm q}
=0.
\end{equation}
Therefore,
\begin{equation}\label{eq:stab_second_moment_expand}
\E[\|e_{t+1}\|_{\mathrm q}^2| \cF_t]
=
\|u_t+\alpha_t\delta_t\|_{\mathrm q}^2
+\alpha_t^2\E[\|\eta_{t+1}\|_{\mathrm q}^2| \cF_t].
\end{equation}
Use Young's inequality \(2\langle u,\delta\rangle\le D\|u\|^2 + D^{-1}\|\delta\|^2\) to bound
\begin{equation}
\|u_t+\alpha_t\delta_t\|_{\mathrm q}^2
\le
\|u_t\|_{\mathrm q}^2 + \alpha_t D\|u_t\|_{\mathrm q}^2
+\frac{\alpha_t}{D}\|\delta_t\|_{\mathrm q}^2 + \alpha_t^2\|\delta_t\|_{\mathrm q}^2.
\end{equation}
Together with \eqref{eq:stab_contraction}, this yields
\begin{equation}\label{eq:stab_m_recursion_cond}
\E[\|e_{t+1}\|_{\mathrm q}^2| \cF_t]
\le
(1-\alpha_tD)\|e_t\|_{\mathrm q}^2
+\frac{\alpha_t}{D}\|\delta_t\|_{\mathrm q}^2
+\alpha_t^2\|\delta_t\|_{\mathrm q}^2
+\alpha_t^2\E[\|\eta_{t+1}\|_{\mathrm q}^2| \cF_t].
\end{equation}

To bound the terms\(\delta_t\) and \(\eta_{t+1}\), let \(N:=|\cI|\). By the definition of \(\delta_t\) and the accuracy event,
for all \(s\in S\),
\begin{equation}
|(\delta_t)(s)|
\le
\|T(v_t)\|_\infty \|b(s)-\bhat(s)\|_1
\le
\epsb\,\|T(v_t)\|_\infty.
\end{equation}
Hence, for a norm-equivalence constant \(C_{q,\infty}>0\) on \(W\),
\begin{equation}\label{eq:stab_delta_bound}
\|\delta_t\|_{\mathrm q}\le C_{q,\infty}\|\delta_t\|_\infty
\le
C_{q,\infty}\epsb\,\|r+Pv_t\|_\infty
\le
C_{q,\infty}\epsb\,(R+\|v_t\|_\infty).
\end{equation}
Similarly, \(\|\Pihat x\|_\infty\le 2\|x\|_\infty\), so for \(\eta_{t+1}=\Pihat\xi_{t+1}\),
\begin{equation}\label{eq:stab_eta_bound}
\E[\|\eta_{t+1}\|_{\mathrm q}^2| \cF_t]
\le
C_{q,\infty}^2\E[\|\eta_{t+1}\|_\infty^2| \cF_t]
\le
4C_{q,\infty}^2\E[\|\xi_{t+1}\|_\infty^2| \cF_t]
\le
4C_{q,\infty}^2\sigma^2.
\end{equation}

We now provide uniform boundedness of \(\E\|e_t\|_{\mathrm q}^2\). Taking expectations in \eqref{eq:stab_m_recursion_cond} and using \eqref{eq:stab_delta_bound}-\eqref{eq:stab_eta_bound},
we get a scalar recursion of the form
\begin{equation}\label{eq:stab_m_recursion}
m_{t+1}
\le
(1-\alpha_tD)m_t
+\frac{\alpha_t}{D}\E\|\delta_t\|_{\mathrm q}^2
+\alpha_t^2\Bigl(\E\|\delta_t\|_{\mathrm q}^2 + \sigma_\eta^2\Bigr),
\end{equation}
where \(m_t:=\E\|e_t\|_{\mathrm q}^2\) and \(\sigma_\eta:=2C_{q,\infty}\sigma\).
Using \eqref{eq:stab_delta_bound} and \(\|Pv_t\|_\infty\le \|v_t\|_\infty\),
\begin{equation}
\E\|\delta_t\|_{\mathrm q}^2
\le
2(C_{q,\infty}\epsb)^2\Bigl(R^2+\E\|v_t\|_\infty^2\Bigr).
\end{equation}
On \(W\), norms are equivalent, so there is \(C_{\infty,q}>0\) such that
\begin{equation}
\|v_t\|_\infty
\le
\|v^\star\|_\infty + \|e_t\|_\infty
\le
\|v^\star\|_\infty + C_{\infty,q}\|e_t\|_{\mathrm q}.
\end{equation}
Thus \(\E\|v_t\|_\infty^2 \le C_0 + C_1 m_t\) for finite constants \(C_0,C_1\), and hence
\begin{equation}
\E\|\delta_t\|_{\mathrm q}^2 \le C_\delta^2\epsb^2\bigl(1+m_t\bigr)
\end{equation}
for a finite constant \(C_\delta\).
Plugging into \eqref{eq:stab_m_recursion} yields, for all large enough \(t\),
\begin{equation}
m_{t+1}
\le
\Bigl(1-\alpha_tD + \tfrac{\alpha_t}{D}C_\delta^2\epsb^2\Bigr)m_t
+
\tfrac{\alpha_t}{D}C_\delta^2\epsb^2
+
\alpha_t^2\cdot C_{\mathrm{noise}},
\end{equation}
with \(C_{\mathrm{noise}}<\infty\).
If \(\epsb\) is chosen so that \(D-\tfrac{1}{D}C_\delta^2\epsb^2\ge D/2\), then the above is a standard stable Robbins--Monro recursion with stepsize \(\alpha_t=\alpha/(t+1)\), which implies
\begin{equation}
\sup_{t\ge 0} m_t \le C_{\mathrm{stab}}<\infty.
\end{equation}
Finally, using norm equivalence on \(W\),
\begin{equation}
\sup_{t\ge 0}\E\|v_t\|_\infty^2
\le
2C_{\infty,q}^2\sup_{t\ge 0}\E\|e_t\|_{\mathrm q}^2 + 2\|v^\star\|_\infty^2
<\infty.
\end{equation}
This completes the proof.
\end{proof}

\begin{lemma}[Formal version of Lemma \ref{lem:main_Pihat_dev}]\label{lem:pihat_dev}
Assume $\widehat\Pi$ is built from anchors $\{a_{i,k}\}$ and weights $\{\widehat b_{i,k}\}$ via
\begin{equation}\label{eq:pihat_def_again}
(\widehat\Pi v)(s)
=
v(s)-\sum_{(i,k)\in\mathcal I} v(a_{i,k})\,\widehat b_{i,k}(s),
\qquad
(\Pi v)(s)
=
v(s)-\sum_{(i,k)\in\mathcal I} v(a_{i,k})\, b_{i,k}(s),
\end{equation}
and suppose
\begin{equation}\label{eq:bhateps}
\max_{s\in\mathcal T}\|\widehat b(s)-b(s)\|_1
\le
\varepsilon_b,
\qquad
\widehat b_{i,k}(s)=b_{i,k}(s)\ \text{for all }s\in\mathcal F,
\end{equation}
where $\widehat b(s)=(\widehat b_{i,k}(s))_{(i,k)\in\mathcal I}$ and $b(s)=(b_{i,k}(s))_{(i,k)\in\mathcal I}$.

Then for every $v\in\mathbb R^n$,
\begin{equation}\label{eq:pihat_infty_op}
\|(\widehat\Pi-\Pi)v\|_\infty
\le
\varepsilon_b\,\|v\|_\infty.
\end{equation}
Moreover, for any semi-norm $\|\cdot\|_{\mathrm q}$ satisfying a norm-equivalence
\begin{equation}\label{eq:q_infty_equiv}
\|x\|_{\mathrm q}\le C_{{\mathrm q},\infty}\|x\|_\infty,
\end{equation}
we also have
\begin{equation}\label{eq:pihat_q_op}
\|(\widehat\Pi-\Pi)v\|_{\mathrm q}
\le
C_{{\mathrm q},\infty}\,\varepsilon_b\,\|v\|_\infty.
\end{equation}
\end{lemma}

\begin{proof}
Fix $s\in S$. From \eqref{eq:pihat_def_again},
\begin{equation}
\big((\widehat\Pi-\Pi)v\big)(s)
=
-\sum_{(i,k)\in\mathcal I} v(a_{i,k})\big(\widehat b_{i,k}(s)-b_{i,k}(s)\big).
\end{equation}
On $s\in\mathcal F$ the difference is zero by assumption.
On $s\in\mathcal T$, taking absolute values and using $|v(a_{i,k})|\le \|v\|_\infty$ gives
\begin{equation}
\big|\big((\widehat\Pi-\Pi)v\big)(s)\big|
\le
\|v\|_\infty \sum_{(i,k)\in\mathcal I}\big|\widehat b_{i,k}(s)-b_{i,k}(s)\big|
=
\|v\|_\infty\,\|\widehat b(s)-b(s)\|_1
\le
\varepsilon_b\,\|v\|_\infty.
\end{equation}
Taking $\max_s$ yields \eqref{eq:pihat_infty_op}. Then \eqref{eq:pihat_q_op} follows from
\eqref{eq:q_infty_equiv} applied to $x=(\widehat\Pi-\Pi)v$.
\end{proof}

\begin{lemma}[Well-posedness of the $\widehat\Pi$-projected fixed point]\label{lem:vhat_star_unique}
Work on the ``good events''. Recall the projections
\begin{equation}
(\Pi v)(s)=v(s)-\sum_{(i,k)\in\cI} v(a_{i,k})\,b_{i,k}(s),
\qquad
(\Pihat v)(s)=v(s)-\sum_{(i,k)\in\cI} v(a_{i,k})\,\bhat_{i,k}(s),
\end{equation}
and define the anchored subspace
\begin{equation}
W \;:=\;\Bigl\{v\in\mathbb R^n:\ v(a_{i,k})=0\;\;\forall (i,k)\in\cI\Bigr\}.
\end{equation}

Using anchor exactness, for any \(v\) and any anchor \(a_{i,k}\),
\begin{equation}
(\Pihat v)(a_{i,k})
=
v(a_{i,k})-\sum_{(j,\ell)\in\cI} v(a_{j,\ell})\,\bhat_{j,\ell}(a_{i,k})
=
v(a_{i,k})-v(a_{i,k})
=
0.
\end{equation}
Hence \(\mathrm{range}(\Pihat)\subseteq W\). Conversely, if \(w\in W\), then all anchor values vanish and thus \(\Pihat w=w\), implying \(W\subseteq \mathrm{range}(\Pihat)\). Therefore,
\begin{equation}
\mathrm{range}(\Pihat)=W
\qquad\text{and}\qquad
\Pihat|_{W}=I_{W}.
\end{equation}
Define the affine map
\begin{equation}\label{eq:Fhat_def}
\widehat F(v)\;:=\;\widehat\Pi(r+Pv).
\end{equation}
Let $\|\cdot\|_{\mathrm q}$ be the quotient norm used in Theorem~\ref{thm:quotient_contraction},
and assume norm equivalence on $W$:
\begin{equation}\label{eq:norm_equiv_W}
\|x\|_{\mathrm q}\le C_{{\mathrm q},\infty}\|x\|_\infty,
\qquad
\|x\|_\infty\le C_{\infty,{\mathrm q}}\|x\|_{\mathrm q},
\qquad \forall x\in W.
\end{equation}
Assume also the uniform projection estimation error bound (Lemma~\ref{lem:pihat_dev}):
\begin{equation}\label{eq:pihat_dev_assump}
\|(\widehat\Pi-\Pi)u\|_\infty \le \varepsilon_b \|u\|_\infty,
\qquad \forall u\in\mathbb R^n.
\end{equation}

Let $B:W\to W$ be the linear operator
\begin{equation}\label{eq:B_def}
B\;:=\; I_W-\Pi P|_W,
\end{equation}
and denote
\begin{equation}\label{eq:kappaB_def}
\kappa_B\coloneqq\|B^{-1}\|_{{\mathrm q}\to{\mathrm q}} < \infty .
\end{equation}
If $\varepsilon_b$ is small enough so that
\begin{equation}\label{eq:small_gain_unique}
\kappa_B\,C_{{\mathrm q},\infty}\,C_{\infty,{\mathrm q}}\,\varepsilon_b < 1,
\end{equation}
then the fixed point equation
\begin{equation}\label{eq:vhat_star_eq}
\widehat v^\star = \widehat\Pi(r+P\widehat v^\star)
\end{equation}
admits a unique solution $\widehat v^\star\in W$.
\end{lemma}

\begin{proof}
If $v$ satisfies \eqref{eq:vhat_star_eq}, then
\begin{equation}
v = \widehat\Pi(r+Pv)\in \mathrm{range}(\widehat\Pi)=W.
\end{equation}
Conversely, for $v\in W$ we have $\widehat\Pi v=v$, hence \eqref{eq:vhat_star_eq} is equivalent to
\begin{equation}\label{eq:fixed_point_linear_W}
v = \widehat\Pi r + \widehat\Pi P v
\qquad\Longleftrightarrow\qquad
(I_W-\widehat\Pi P|_W)\,v = \widehat\Pi r.
\end{equation}
Therefore existence/uniqueness of $\widehat v^\star$ is equivalent to invertibility of the operator
\begin{equation}\label{eq:Bhat_def}
\widehat B\;:=\;I_W-\widehat\Pi P|_W:W\to W.
\end{equation}

We show $\ker(B)=\{0\}$.
Let $x\in W$ satisfy $Bx=0$. Then
\begin{equation}\label{eq:Bx0}
x = \Pi P x.
\end{equation}
Since $\ker(\Pi)=\mathcal K(P)$ and $\Pi$ is the canonical representative map of the quotient
(Theorem~\ref{thm:wellposed}), we have
\begin{equation}
[x] = [\Pi x] \quad\text{and}\quad [\Pi y]=[y]\quad\text{for all }y,
\end{equation}
where $[\cdot]$ denotes the equivalence class in $Q:=\mathbb R^n/\mathcal K(P)$.
Applying the quotient map $[\cdot]$ to \eqref{eq:Bx0} gives
\begin{equation}\label{eq:quotient_fixed}
[x] = [P x] = \overline P[x].
\end{equation}
Hence
\begin{equation}\label{eq:kernel_I_Pbar}
(I-\overline P)[x]=0.
\end{equation}
By Theorem~\ref{thm:quotient_contraction}, $\rho(\overline P)<1$, so $1$ is not an eigenvalue of $\overline P$ and
therefore $I-\overline P$ is invertible on $Q$. Thus \eqref{eq:kernel_I_Pbar} implies
\begin{equation}
[x]=0 \quad\Longrightarrow\quad x\in\mathcal K(P).
\end{equation}
But $x\in W=\mathrm{range}(\Pi)$ and $\ker(\Pi)=\mathcal K(P)$ implies
\begin{equation}
W\cap \mathcal K(P)=\{0\}.
\end{equation}
Therefore $x=0$, proving $\ker(B)=\{0\}$. Since $W$ is finite-dimensional, $B$ is invertible and $\kappa_B<\infty$ in
\eqref{eq:kappaB_def} is well-defined.

We now bound the perturbation $\widehat B-B$ in operator norm.
From \eqref{eq:B_def}--\eqref{eq:Bhat_def},
\begin{equation}
\widehat B - B
=
(\Pi-\widehat\Pi)P|_W.
\end{equation}
Fix $x\in W$. Using \eqref{eq:norm_equiv_W}, \eqref{eq:pihat_dev_assump}, and $\|Px\|_\infty\le \|x\|_\infty$,
\begin{align}
\|(\widehat B-B)x\|_{\mathrm q}
&=
\|(\Pi-\widehat\Pi)P x\|_{\mathrm q}\\
&\le
C_{{\mathrm q},\infty}\|(\Pi-\widehat\Pi)P x\|_\infty\\
&\le
C_{{\mathrm q},\infty}\,\varepsilon_b\,\|P x\|_\infty\\
&\le
C_{{\mathrm q},\infty}\,\varepsilon_b\,\|x\|_\infty\\
&\le
C_{{\mathrm q},\infty}\,C_{\infty,{\mathrm q}}\,\varepsilon_b\,\|x\|_{\mathrm q}.
\end{align}
Therefore
\begin{equation}\label{eq:EB_bound}
\|\widehat B-B\|_{{\mathrm q}\to{\mathrm q}}
\le
C_{{\mathrm q},\infty}\,C_{\infty,{\mathrm q}}\,\varepsilon_b.
\end{equation}

We now write
\begin{equation}\label{eq:factorization}
\widehat B
=
B - (\widehat B-B)
=
B\Big(I_W - B^{-1}(\widehat B-B)\Big).
\end{equation}
Let
\begin{equation}\label{eq:E_def}
E := B^{-1}(\widehat B-B).
\end{equation}
Then by \eqref{eq:kappaB_def} and \eqref{eq:EB_bound},
\begin{equation}\label{eq:E_norm}
\|E\|_{{\mathrm q}\to{\mathrm q}}
\le
\|B^{-1}\|_{{\mathrm q}\to{\mathrm q}}\cdot \|\widehat B-B\|_{{\mathrm q}\to{\mathrm q}}
\le
\kappa_B C_{{\mathrm q},\infty}\,C_{\infty,{\mathrm q}}\,\varepsilon_b.
\end{equation}
Under the small-gain condition \eqref{eq:small_gain_unique}, we have $\|E\|_{{\mathrm q}\to{\mathrm q}}<1$,
so $I_W-E$ is invertible with Neumann series
\begin{equation}
(I_W-E)^{-1}=\sum_{m=0}^\infty E^m,
\end{equation}
and hence $\widehat B$ is invertible by \eqref{eq:factorization}.

To show the existence and uniqueness of $\widehat v^\star$), since $\widehat B$ is invertible on $W$, the linear system \eqref{eq:fixed_point_linear_W} has the unique solution
\begin{equation}
\widehat v^\star = \widehat B^{-1}\widehat\Pi r \in W.
\end{equation}
By Step 1 this $\widehat v^\star$ is exactly the unique fixed point of \eqref{eq:vhat_star_eq}.
\end{proof}

\begin{lemma}[Simplex property of the extended basis]\label{lem:b_simplex}
For every state $s\in S$ and every index $(i,k)\in\mathcal I$, we have $b_{i,k}(s)\ge 0$ and
\[
\sum_{(i,k)\in\mathcal I} b_{i,k}(s)=1.
\]
Moreover, Algorithm~\ref{alg:main_estimate_weights} outputs $\widehat b_{i,k}$ satisfying $\widehat b_{i,k}(s)\ge 0$ and
\[
\sum_{(i,k)\in\mathcal I} \widehat b_{i,k}(s)=1
\qquad\text{for all }s\in S.
\]
\end{lemma}

\begin{proof}
If $s\in\mathcal F$, then exactly one pair $(i,k)$ satisfies $s\in C_{i,k}$, so the sum equals $1$ by definition.

If $s\in\mathcal T$, then the chain hits $\mathcal F$ almost surely and the random pair
\[
\Bigl(\class(X_{\tau_{\mathcal F}}),\ (\phase(X_{\tau_{\mathcal F}})-\tau_{\mathcal F})\bmod d_{\class(X_{\tau_{\mathcal F}})}\Bigr)
\]
takes values in $\mathcal I$ and is uniquely defined on each trajectory. The events indexed by $(i,k)\in\mathcal I$
form a partition of the sample space, hence their probabilities sum to $1$.

For $\widehat b$, the recurrent case is set deterministically to indicators. For $s\in\mathcal T$,
each simulated absorption episode increments exactly one coordinate $(i,k)$ and the subsequent normalization by $M$
forces the sum over $(i,k)\in\mathcal I$ to equal $1$.
\end{proof}

\begin{lemma}[Infinity-norm operator bound for the gauge map]\label{lem:pihat_infty_norm}
Assume $\widehat b_{i,k}(s)\ge 0$ for all $s\in S$ and $(i,k)\in\mathcal I$, and
\[
\sum_{(i,k)\in\mathcal I}\widehat b_{i,k}(s)=1
\qquad\text{for all }s\in S.
\]
Let $\widehat\Pi$ be defined by \eqref{eq:pihat_def_again}. Then for all $v\in\mathbb R^n$,
\[
\|\widehat\Pi v\|_\infty \le 2\|v\|_\infty.
\]
The same bound holds for $\Pi$.
\end{lemma}

\begin{proof}
Fix $s\in S$. By definition,
\[
(\widehat\Pi v)(s)=v(s)-\sum_{(i,k)\in\mathcal I} v(a_{i,k})\,\widehat b_{i,k}(s).
\]
Taking absolute values and using $\widehat b_{i,k}(s)\ge 0$ gives
\[
|(\widehat\Pi v)(s)|
\le
|v(s)|+\sum_{(i,k)\in\mathcal I}|v(a_{i,k})|\,\widehat b_{i,k}(s)
\le
\|v\|_\infty+\|v\|_\infty\sum_{(i,k)\in\mathcal I}\widehat b_{i,k}(s)
=
2\|v\|_\infty.
\]
Taking the maximum over $s$ proves the claim. The proof for $\Pi$ is identical.
\end{proof}

\begin{lemma}[Contraction and bounded noise imply bounded iterates in expectation]
\label{lem:contract_noise_bounded_iterates}
Assume the conditions of Lemma~\ref{lem:projSA_stability_template}.
Assume the stepsizes satisfy $\alpha_t=\frac{\alpha}{t+t_0}$
and $t_0>1$. Let $v_t$ be generated by \eqref{eq:projSA_rec} and let $\widehat v^\star$ be the fixed point in \eqref{eq:vhat_star_def}.
Then
\begin{equation}
\sup_{t\ge 0}\E\|v_t-\widehat v^\star\|_{\mathrm q}^2
\le
\|v_0-\widehat v^\star\|_{\mathrm q}^2
+
4\alpha^2 C_{{\mathrm q},\infty}^2\sigma^2 \sum_{t=0}^\infty \frac{1}{(t+t_0)^2}
\le
\|v_0-\widehat v^\star\|_{\mathrm q}^2
+
\frac{4\alpha^2 C_{{\mathrm q},\infty}^2\sigma^2}{t_0-1}.
\end{equation}
Moreover, using norm equivalence on $\mathrm{range}(\widehat\Pi)$ as in \eqref{eq:infty_q_equiv_on_range},
\begin{equation}
\sup_{t\ge 0}\E\|v_t\|_\infty^2
\le
2C_{\infty,{\mathrm q}}^2
\left(
\|v_0-\widehat v^\star\|_{\mathrm q}^2
+
\frac{4\alpha^2 C_{{\mathrm q},\infty}^2\sigma^2}{t_0-1}
\right)
+
2\|\widehat v^\star\|_\infty^2.
\end{equation}
\end{lemma}

\begin{proof}
Let $e_t:=v_t-\widehat v^\star$.
From \eqref{eq:error_rec} we have
\begin{equation}
e_{t+1}
=
(1-\alpha_t)e_t+\alpha_t\,\widehat\Pi P e_t+\alpha_t\,\widehat\Pi\xi_{t+1},
\end{equation}
where $\xi_{t+1}$ is the oracle noise with $\E[\xi_{t+1}|\cF_t]=0$ and $\E\big[\|\xi_{t+1}\|_\infty^2|\cF_t\big]\le \sigma^2$. Furthermore, by the contraction bound in the proof of Lemma~\ref{lem:projSA_stability_template},
\begin{equation}
\|(1-\alpha_t)e_t+\alpha_t\,\widehat\Pi P e_t\|_{\mathrm q}\le \|e_t\|_{\mathrm q}.
\end{equation}
Since $\|\widehat\Pi z\|_\infty\le 2\|z\|_\infty$ by Lemma~\ref{lem:pihat_infty_norm}, we obtain
\begin{equation}
\E\big[\|\widehat\Pi\xi_{t+1}\|_{\mathrm q}^2|\cF_t\big]
\le
C_{{\mathrm q},\infty}^2 \E\big[\|\widehat\Pi\xi_{t+1}\|_\infty^2|\cF_t\big]
\le
4C_{{\mathrm q},\infty}^2\sigma^2.
\end{equation}
Expanding the square in the $\|\cdot\|_{\mathrm q}$ inner product and using
$\E[\xi_{t+1}|\cF_t]=0$ to drop the cross term gives
\begin{equation}
\E\big[\|e_{t+1}\|_{\mathrm q}^2|\cF_t\big]
\le
\|e_t\|_{\mathrm q}^2
+
4\alpha_t^2 C_{{\mathrm q},\infty}^2\sigma^2.
\end{equation}
Taking expectation and summing over $t$ yields
\begin{equation}
\E\|e_T\|_{\mathrm q}^2
\le
\|e_0\|_{\mathrm q}^2
+
4\alpha^2 C_{{\mathrm q},\infty}^2\sigma^2
\sum_{t=0}^{T-1}\frac{1}{(t+t_0)^2}
\le
\|e_0\|_{\mathrm q}^2
+
\frac{4\alpha^2 C_{{\mathrm q},\infty}^2\sigma^2}{t_0-1}.
\end{equation}
The $\|\cdot\|_\infty$ bound follows from $\|v_t\|_\infty\le \|e_t\|_\infty+\|\widehat v^\star\|_\infty$ and
$\|e_t\|_\infty\le C_{\infty,{\mathrm q}}\|e_t\|_{\mathrm q}$ on $\mathrm{range}(\widehat\Pi)$.
\end{proof}

\begin{lemma}[Formal version of Lemma \ref{lem:main_hat_contraction}]\label{lem:projSA_stability_template}
Let $T(v)=r+Pv$ with $\|r\|_\infty\le R$ and $P$ row-stochastic.
Assume the quotient semi-norm $\|\cdot\|_{\mathrm q}$ satisfies
\begin{equation}\label{eq:contract_P_q}
\|P x\|_{\mathrm q}\le \gamma\|x\|_{\mathrm q}
\qquad\text{for all }x\in\mathbb R^n,
\end{equation}
for some $\gamma\in(0,1)$.
Assume also there is a constant $C_{\infty,{\mathrm q}}$ such that on the subspace
$\mathrm{range}(\widehat\Pi)$,
\begin{equation}\label{eq:infty_q_equiv_on_range}
\|x\|_\infty\le C_{\infty,{\mathrm q}}\|x\|_{\mathrm q}.
\end{equation}

Consider the recursion (initialized with $v_0\in\mathrm{range}(\widehat\Pi)$)
\begin{equation}\label{eq:projSA_rec}
v_{t+1}
=
\widehat\Pi\Big((1-\alpha_t)v_t+\alpha_t\,\widehat T(v_t)\Big),
\qquad
\alpha_t = \frac{\alpha}{t+t_0},
\qquad
t_0\ge \alpha,
\end{equation}
where the oracle satisfies
\begin{equation}\label{eq:oracleB}
\E[\widehat T(v_t)| \cF_t]=r+Pv_t,
\qquad
\E\!\left[\|\widehat T(v_t)-(r+Pv_t)\|_\infty^2| \cF_t\right]\le \sigma^2.
\end{equation}
with $\mathcal F_t=\sigma(v_0,\dots,v_t)$.

Let $\widehat v^\star$ be the unique fixed point of $\widehat F(v):=\widehat\Pi T(v)$ on
$\mathrm{range}(\widehat\Pi)$, i.e.
\begin{equation}\label{eq:vhat_star_def}
\widehat v^\star=\widehat\Pi(r+P\widehat v^\star).
\end{equation}
Define
\begin{equation}\label{eq:gamma_hat_condition}
\widehat\gamma := \gamma + C_{{\mathrm q},\infty}\varepsilon_b\,C_{\infty,{\mathrm q}}.
\end{equation}
Assume $\varepsilon_b$ is chosen so that $\widehat\gamma<1$.
Here $C_{{\mathrm q},\infty}$ is from \eqref{eq:q_infty_equiv} and $\varepsilon_b$ is from Lemma~\ref{lem:pihat_dev}.

Then there exists a finite constant $C_{\mathrm{stab}}$ (depending on
$\alpha,\widehat\gamma,\sigma,R$, and the norm-equivalence constants) such that
\begin{equation}\label{eq:stability_bound}
\sup_{t\ge 0}\mathbb E\|v_t-\widehat v^\star\|_{\mathrm q}^2
\le
C_{\mathrm{stab}}.
\end{equation}
Consequently, there is a finite constant $B_{\sup}$ such that
\begin{equation}\label{eq:stability_infty}
\sup_{t\ge 0}\mathbb E\|v_t\|_\infty^2\le B_{\sup}^2<\infty.
\end{equation}
\end{lemma}

\begin{proof}

Since $v_0\in\mathrm{range}(\widehat\Pi)$ and $\widehat\Pi$ is a projection, we have $v_t=\widehat\Pi v_t$ for all $t$.
Thus \eqref{eq:projSA_rec} becomes
\begin{equation}\label{eq:projSA_simplify}
v_{t+1}
=
(1-\alpha_t)v_t+\alpha_t\,\widehat\Pi\widehat T(v_t)
=
v_t+\alpha_t\Big(\widehat\Pi T(v_t)-v_t\Big)
+\alpha_t\,\widehat\Pi\xi_{t+1},
\end{equation}
where $\xi_{t+1}:=\widehat T(v_t)-T(v_t)$ satisfies
\begin{equation}
\mathbb E[\xi_{t+1}| \mathcal F_t]=0,
\qquad
\mathbb E[\|\xi_{t+1}\|_\infty^2| \mathcal F_t]\le \sigma^2.
\end{equation}

Define the error $e_t:=v_t-\widehat v^\star$.
Using \eqref{eq:vhat_star_def}, subtract $\widehat v^\star$ from \eqref{eq:projSA_simplify}:
\begin{equation}\label{eq:error_rec}
e_{t+1}
=
(1-\alpha_t)e_t+\alpha_t\,\widehat\Pi P e_t+\alpha_t\,\widehat\Pi\xi_{t+1}.
\end{equation}

For any $x\in\mathrm{range}(\widehat\Pi)$, write
\begin{equation}
\widehat\Pi P x
=
\Pi P x +(\widehat\Pi-\Pi)P x.
\end{equation}
Hence, using \eqref{eq:contract_P_q}, Lemma~\ref{lem:pihat_dev} (in $\|\cdot\|_{\mathrm q}$),
and $\|P x\|_\infty\le \|x\|_\infty$,
\begin{align}
\|\widehat\Pi P x\|_{\mathrm q}
&\le
\|\Pi P x\|_{\mathrm q}+\|(\widehat\Pi-\Pi)P x\|_{\mathrm q}\\
&\le
\gamma\|x\|_{\mathrm q}
+
C_{{\mathrm q},\infty}\varepsilon_b\,\|P x\|_\infty\\
&\le
\gamma\|x\|_{\mathrm q}
+
C_{{\mathrm q},\infty}\varepsilon_b\,\|x\|_\infty\\
&\le
\Big(\gamma + C_{{\mathrm q},\infty}\varepsilon_b\,C_{\infty,{\mathrm q}}\Big)\|x\|_{\mathrm q}
=
\widehat\gamma\|x\|_{\mathrm q}.
\end{align}

Apply this to $x=e_t\in\mathrm{range}(\widehat\Pi)$ and combine with \eqref{eq:error_rec}:
\begin{equation}\label{eq:drift_contract}
\|(1-\alpha_t)e_t+\alpha_t\widehat\Pi P e_t\|_{\mathrm q}
\le
\big(1-\alpha_t(1-\widehat\gamma)\big)\|e_t\|_{\mathrm q}.
\end{equation}

Let $d_t := (1-\alpha_t)e_t + \alpha_t \widehat \Pi(Pe_t)$, then \eqref{eq:error_rec} becomes
\begin{equation}
e_{t+1}=d_t+\alpha_t\widehat\Pi\xi_{t+1}.
\end{equation}
Since \(\|\cdot\|_{\mathrm q}\) is induced by an inner product,
we can expand the square and use \(\cF_t\)-measurability of \(d_t\) and \(\widehat\Pi\) to get
\begin{align}
\E\!\left[\|e_{t+1}\|_{\mathrm q}^2| \cF_t\right]
&=
\|d_t\|_{\mathrm q}^2
+2\alpha_t\,\E\!\left[\langle d_t,\widehat\Pi\xi_{t+1}\rangle_{\mathrm q}| \cF_t\right]
+\alpha_t^2\,\E\!\left[\|\widehat\Pi\xi_{t+1}\|_{\mathrm q}^2| \cF_t\right] \notag\\
&=
\|d_t\|_{\mathrm q}^2
+\alpha_t^2\,\E\!\left[\|\widehat\Pi\xi_{t+1}\|_{\mathrm q}^2| \cF_t\right],
\label{eq:step3_cond_moment}
\end{align}
where the cross term vanishes because \(\E[\xi_{t+1}| \cF_t]=0\).

Next, using \eqref{eq:drift_contract} and \(\alpha_t\in(0,1]\),
\begin{equation}\label{eq:step3_drift}
\|d_t\|_{\mathrm q}
\le
(1-\alpha_t)\|e_t\|_{\mathrm q}+\alpha_t\widehat\gamma\,\|e_t\|_{\mathrm q}
=
\bigl(1-\alpha_t(1-\widehat\gamma)\bigr)\|e_t\|_{\mathrm q}.
\end{equation}
Therefore
\begin{equation}\label{eq:step3_drift_square}
\|d_t\|_{\mathrm q}^2
\le
\bigl(1-\alpha_t(1-\widehat\gamma)\bigr)^2\|e_t\|_{\mathrm q}^2.
\end{equation}

Finally, Lemma~\ref{lem:pihat_infty_norm} gives
\(\|\widehat\Pi z\|_\infty\le 2\|z\|_\infty\),
so with \(\|x\|_{\mathrm q}\le C_{{\mathrm q},\infty}\|x\|_\infty\) we obtain
\begin{equation}\label{eq:step3_noise_bound}
\E\!\left[\|\widehat\Pi\xi_{t+1}\|_{\mathrm q}^2| \cF_t\right]
\le
4C_{{\mathrm q},\infty}^2\,\E\!\left[\|\xi_{t+1}\|_\infty^2| \cF_t\right]
\le
4C_{{\mathrm q},\infty}^2\sigma^2.
\end{equation}

Combining \eqref{eq:step3_cond_moment}, \eqref{eq:step3_drift_square}, and \eqref{eq:step3_noise_bound} and taking total expectation yields
\begin{equation}\label{eq:moment_recursion_corrected}
\E\|e_{t+1}\|_{\mathrm q}^2
\le
\bigl(1-\alpha_t(1-\widehat\gamma)\bigr)^2\,\E\|e_t\|_{\mathrm q}^2
+
4\alpha_t^2C_{{\mathrm q},\infty}^2\sigma^2.
\end{equation}

Since $\sum_t \alpha_t^2<\infty$ and $\sum_t \alpha_t=\infty$, the recursion \eqref{eq:moment_recursion_corrected}
implies $\sup_t \mathbb E\|e_t\|_{\mathrm q}^2<\infty$ by a standard Robbins--Siegmund argument.
This proves \eqref{eq:stability_bound}.

Finally, \eqref{eq:stability_infty} follows from \eqref{eq:infty_q_equiv_on_range} and
$\|v_t\|_{\mathrm q}\le \|v_t-\widehat v^\star\|_{\mathrm q}+\|\widehat v^\star\|_{\mathrm q}$, plus norm equivalence
between $\|\cdot\|_{\mathrm q}$ and $\|\cdot\|_\infty$ on the finite-dimensional space.
\end{proof}

\begin{theorem}[Formal version of Theorem \ref{thm:main_sa_rate}]\label{thm:projSA_rate_template}
Assume the conditions of Lemma~\ref{lem:projSA_stability_template} hold and that
\(\alpha_t=\alpha/(t+t_0)\) with \(t_0\ge \max\{\alpha,1\}\).
Let \(v_t\) be generated by \eqref{eq:projSA_rec}, and let \(\widehat v^\star\) be the unique fixed point of
\(\widehat F(v)=\widehat\Pi(r+Pv)\) in \(W=\mathrm{range}(\widehat\Pi)\).

Define $a := \alpha(1-\widehat\gamma)$. Assume \(a>1/2\).
Then for every \(T\ge 0\),
\begin{equation}\label{eq:rate_to_vhat}
\E\|v_T-\widehat v^\star\|_{\mathrm q}
\le
\Bigl(\frac{t_0}{T+t_0}\Bigr)^{a}\,\|v_0-\widehat v^\star\|_{\mathrm q}
+
\frac{2^{a+1}\alpha\,C_{{\mathrm q},\infty}\sigma}{\sqrt{(2a-1)(T+t_0)}}.
\end{equation}

Let \(v^\star\) be the fixed point of \(F(v)=\Pi(r+Pv)\) in \(\mathrm{range}(\Pi)\).
On the event \eqref{eq:bhateps}, the fixed points satisfy
\begin{equation}\label{eq:fixed_point_gap}
\|\widehat v^\star-v^\star\|_{\mathrm q}
\le
\frac{C_{{\mathrm q},\infty}\varepsilon_b}{1-\gamma}\Bigl(R+\|\widehat v^\star\|_\infty\Bigr).
\end{equation}
Moreover,
\begin{equation}\label{eq:vhat_infty_bound}
\|\widehat v^\star\|_\infty
\le
\frac{2C_{\infty,{\mathrm q}}C_{{\mathrm q},\infty}R}{1-\widehat\gamma}.
\end{equation}
Combining \eqref{eq:rate_to_vhat}, \eqref{eq:fixed_point_gap}, and \eqref{eq:vhat_infty_bound} yields an explicit bound on
\(\E\|v_T-v^\star\|_{\mathrm q}\).
\end{theorem}

\begin{proof}
Let \(e_t:=v_t-\widehat v^\star\).
By Lemma~\ref{lem:projSA_stability_template}, we have the recursion
\begin{equation}
\E\|e_{t+1}\|_{\mathrm q}^2
\le
\bigl(1-\alpha_t(1-\widehat\gamma)\bigr)^2\,\E\|e_t\|_{\mathrm q}^2
+
4\alpha_t^2\,C_{{\mathrm q},\infty}^2\sigma^2.
\end{equation}
With \(\alpha_t=\alpha/(t+t_0)\) and \(a=\alpha(1-\widehat\gamma)\), this becomes
\begin{equation}\label{eq:rate_recursion_m}
m_{t+1}
\le
\Bigl(1-\frac{a}{t+t_0}\Bigr)^2 m_t
+
\frac{4\alpha^2 C_{{\mathrm q},\infty}^2\sigma^2}{(t+t_0)^2},
\end{equation}
where \(m_t:=\E\|e_t\|_{\mathrm q}^2\).

Define \(w_t:=(t+t_0)^{2a}m_t\).
Using \((1+1/(t+t_0))^{2a}\le \exp(2a/(t+t_0))\) and
\((1-a/(t+t_0))^2\le \exp(-2a/(t+t_0))\), we obtain
\begin{equation}
(t+t_0+1)^{2a}\Bigl(1-\frac{a}{t+t_0}\Bigr)^2 \le (t+t_0)^{2a}.
\end{equation}
Multiplying \eqref{eq:rate_recursion_m} by \((t+t_0+1)^{2a}\) yields
\begin{equation}
w_{t+1}
\le
w_t
+
4\alpha^2 C_{{\mathrm q},\infty}^2\sigma^2\cdot
\frac{(t+t_0+1)^{2a}}{(t+t_0)^2}.
\end{equation}
Since \((t+t_0+1)^{2a}\le 2^{2a}(t+t_0)^{2a}\), we obtain
\begin{equation}
w_{t+1}
\le
w_t
+
2^{2a+2}\alpha^2 C_{{\mathrm q},\infty}^2\sigma^2\,(t+t_0)^{2a-2}.
\end{equation}
Summing from \(t=0\) to \(T-1\) gives
\begin{equation}
w_T
\le
w_0
+
2^{2a+2}\alpha^2 C_{{\mathrm q},\infty}^2\sigma^2
\sum_{t=0}^{T-1}(t+t_0)^{2a-2}.
\end{equation}
When \(a>1/2\), we have \(2a-2>-1\), so
\begin{equation}
\sum_{t=0}^{T-1}(t+t_0)^{2a-2}
\le
\int_{t_0-1}^{T+t_0} x^{2a-2}\,dx
=
\frac{(T+t_0)^{2a-1}-(t_0-1)^{2a-1}}{2a-1}
\le
\frac{(T+t_0)^{2a-1}}{2a-1}.
\end{equation}
Therefore,
\begin{equation}
m_T
=
\frac{w_T}{(T+t_0)^{2a}}
\le
\Bigl(\frac{t_0}{T+t_0}\Bigr)^{2a}m_0
+
\frac{2^{2a+2}\alpha^2 C_{{\mathrm q},\infty}^2\sigma^2}{(2a-1)(T+t_0)}.
\end{equation}
Taking square roots and using \(\sqrt{x+y}\le \sqrt x+\sqrt y\) yields \eqref{eq:rate_to_vhat}.

For the fixed point gap, start from
\begin{equation}
\widehat v^\star - v^\star
=
(\widehat\Pi-\Pi)(r+P\widehat v^\star) + \Pi\bigl(P(\widehat v^\star-v^\star)\bigr).
\end{equation}
Taking \(\|\cdot\|_{\mathrm q}\) and using \(\|\Pi(Px)\|_{\mathrm q}\le \gamma\|x\|_{\mathrm q}\) gives
\begin{equation}
\|\widehat v^\star-v^\star\|_{\mathrm q}
\le
\|(\widehat\Pi-\Pi)(r+P\widehat v^\star)\|_{\mathrm q}
+
\gamma\|\widehat v^\star-v^\star\|_{\mathrm q}.
\end{equation}
Rearranging yields
\begin{equation}
\|\widehat v^\star-v^\star\|_{\mathrm q}
\le
\frac{1}{1-\gamma}\|(\widehat\Pi-\Pi)(r+P\widehat v^\star)\|_{\mathrm q}.
\end{equation}
On \eqref{eq:bhateps}, Lemma~\ref{lem:pihat_dev} implies
\(\|(\widehat\Pi-\Pi)u\|_\infty\le \varepsilon_b\|u\|_\infty\),
so
\begin{equation}
\|(\widehat\Pi-\Pi)(r+P\widehat v^\star)\|_{\mathrm q}
\le
C_{{\mathrm q},\infty}\varepsilon_b\|r+P\widehat v^\star\|_\infty
\le
C_{{\mathrm q},\infty}\varepsilon_b\bigl(R+\|\widehat v^\star\|_\infty\bigr),
\end{equation}
which gives \eqref{eq:fixed_point_gap}.

Finally, from the fixed point equation
\(\widehat v^\star=\widehat\Pi r + \widehat\Pi(P\widehat v^\star)\),
we get
\begin{equation}
\|\widehat v^\star\|_{\mathrm q}
\le
\|\widehat\Pi r\|_{\mathrm q}+\widehat\gamma\|\widehat v^\star\|_{\mathrm q},
\qquad
\|\widehat v^\star\|_{\mathrm q}\le \frac{\|\widehat\Pi r\|_{\mathrm q}}{1-\widehat\gamma}.
\end{equation}
Using \(\|\widehat\Pi r\|_{\mathrm q}\le C_{{\mathrm q},\infty}\|\widehat\Pi r\|_\infty\le 2C_{{\mathrm q},\infty}R\)
and then \(\|\cdot\|_\infty\le C_{\infty,{\mathrm q}}\|\cdot\|_{\mathrm q}\) yields \eqref{eq:vhat_infty_bound}.
\end{proof}

\begin{corollary}\label{cor:optionB_rate}
Work on the event of Theorem~\ref{thm:identify} and Lemma~\ref{lem:bhats}.
Assume \eqref{eq:oracleB} holds and assume $\widehat\gamma<1$ in \eqref{eq:gamma_hat_condition}.

Run Algorithm~\ref{alg:main_psa} with stepsizes
\begin{equation}
\alpha_t=\frac{\alpha}{t+t_0},
\qquad t_0\ge \alpha,
\end{equation}
and initialization $v_0\in \mathrm{range}(\widehat\Pi)$.

Let $\widehat v^\star$ be the unique fixed point of $\widehat F(v)=\widehat\Pi(r+Pv)$ on
$W=\mathrm{range}(\widehat\Pi)$ and let $v^\star$ be the fixed point of $F(v)=\Pi(r+Pv)$ on $\mathrm{range}(\Pi)$.

Define
\[
a:=\alpha(1-\widehat\gamma),
\qquad \text{and assume } a>\frac12 .
\]
Then for every $T\ge 0$,
\begin{align}
\E\|v_T-v^\star\|_{\mathrm q}
&\le
\Bigl(\frac{t_0}{T+t_0}\Bigr)^{a}\,\|v_0-\widehat v^\star\|_{\mathrm q}
+
\frac{2^{a+1}\alpha\,C_{{\mathrm q},\infty}\sigma}{\sqrt{(2a-1)(T+t_0)}}
+
\frac{C_{{\mathrm q},\infty}\varepsilon_b}{1-\gamma}\Bigl(R+\|\widehat v^\star\|_\infty\Bigr).
\label{eq:optionB_rate_clean}
\end{align}
Moreover, $\|\widehat v^\star\|_\infty$ is bounded by \eqref{eq:vhat_infty_bound}.
\end{corollary}

\begin{proof}
Apply Theorem~\ref{thm:projSA_rate_template} with $K=t_0$.
\end{proof}

\begin{corollary}[Formal version of Corollary \ref{cor:main_end_to_end_v}]\label{cor:end_to_end_v}
Fix a target accuracy $\varepsilon\in(0,1)$.
Assume $p_{\min}>0$ in \eqref{eq:pmin}.
Assume the SA stepsizes are $\alpha_t=\alpha/(t+t_0)$ with $t_0\ge \alpha$ and $a=\alpha(1-\widehat\gamma)>1/2$.

Choose
\[
K_{\mathrm{graph}}
\geq
\frac{1}{p_{\min}}\log\frac{n^2}{\delta_1}.
\]
Choose $\varepsilon_b$ so that $\widehat\gamma<1$ and
\[
\frac{C_{{\mathrm q},\infty}\varepsilon_b}{1-\gamma}\Bigl(R+\|\widehat v^\star\|_\infty\Bigr)
\le \frac{\varepsilon}{3}.
\]
Choose
\[
M_{\mathrm{abs}}
\;\ge\;
\frac{8}{\varepsilon_b^2}\Bigl( N + \log\frac{|\mathcal T|}{\delta_2}\Bigr).
\]
Choose the SA horizon $T$ so that
\[
\Bigl(\frac{t_0}{T+t_0}\Bigr)^a\|v_0-\widehat v^\star\|_{\mathrm q}
\le \frac{\varepsilon}{3},
\qquad
\frac{2^{a+1}\alpha\,C_{{\mathrm q},\infty}\sigma}{\sqrt{(2a-1)(T+t_0)}}
\le \frac{\varepsilon}{3}.
\]
Then on the event of Theorem~\ref{thm:identify} and Lemma~\ref{lem:bhats},
\[
\E\|[v_T]-[v^\star]\|_{\mathrm q}=\E\|v_T-v^\star\|_{\mathrm q}\le \varepsilon.
\]

Moreover, the expected number of simulator queries is at most
\[
nK_{\mathrm{graph}} + |\mathcal T|\,M_{\mathrm{abs}}\,H_{\mathrm{abs}} + nT.
\]
\end{corollary}
\begin{remark}\label{rem:no_log}
If one insists on analyzing the recursion \eqref{eq:projSA_rec} directly relative to the true fixed point $v^\star$
(and treats $(\widehat\Pi-\Pi)T(v_t)$ as a bounded drift bias),
then because $\sum_{t\le T}\alpha_t=\Theta(\log T)$ one can indeed derive a term of order
$\varepsilon_b\log T/(1-\gamma)$, exactly like the biased-noise term in standard contractive SA.

The sharper analysis above avoids this artifact by (i) introducing the perturbed contraction $\widehat F(v)=\widehat\Pi T(v)$,
(ii) proving SA converges to its fixed point $\widehat v^\star$ at the usual $O(1/\sqrt T)$ rate,
and (iii) bounding the time-independent fixed-point gap $\|\widehat v^\star-v^\star\|$.
This converts a would-be $\varepsilon\log T$ accumulation into an $\varepsilon/(1-\gamma)$ error floor.
\end{remark}


\subsection{Proofs in Section \ref{peripheral}} \label{proof_peripheral}

Recall from Theorem~\ref{thm:wellposed} that the unique gauge-fixed solution $v^\star\in\mathrm{range}(\Pi)$
induces a peripheral residual
\begin{equation}\label{eq:def_gstar_again}
g^\star := r + P v^\star - v^\star \in \mathcal K(P).
\end{equation}
When the peripheral space is represented via the basis $\{b_{i,k}\}_{(i,k)\in\mathcal I}$ from
Definition~\ref{def:main_b} with anchors $\{a_{i,k}\}$, we can recover $g^\star$ coordinate-wise.

\begin{lemma}\label{lem:anchor_interpolation}
Let $\{b_{i,k}\}_{(i,k)\in\mathcal I}$ be as in Definition~\ref{def:main_b}, and choose anchors $a_{i,k}\in C_{i,k}$.
Then for all $(i,k),(j,\ell)\in\mathcal I$,
\begin{equation}\label{eq:b_anchor_delta_app}
b_{j,\ell}(a_{i,k}) = \bbm{1}\{(j,\ell)=(i,k)\}.
\end{equation}
Consequently, every $g\in \mathrm{span}\{b_{i,k}\}_{(i,k)\in\mathcal I}$ admits the unique expansion
\begin{equation}\label{eq:peripheral_expansion}
g(\cdot) = \sum_{(i,k)\in\mathcal I} g(a_{i,k})\, b_{i,k}(\cdot).
\end{equation}
\end{lemma}

\begin{proof}
For $a_{i,k}\in C_{i,k}\subseteq \mathcal F$, Definition~\ref{def:main_b} gives
$b_{j,\ell}(a_{i,k})=\bbm{1}\{a_{i,k}\in C_{j,\ell}\}$, which equals $\bbm{1}\{(j,\ell)=(i,k)\}$,
proving \eqref{eq:b_anchor_delta}.
For \eqref{eq:peripheral_expansion}, write $g=\sum_{(j,\ell)}\theta_{j,\ell} b_{j,\ell}$. Evaluating at $a_{i,k}$ and using
\eqref{eq:b_anchor_delta} yields $g(a_{i,k})=\theta_{i,k}$. Uniqueness follows immediately.
\end{proof}

\begin{definition}[Peripheral coordinates of the residual]\label{def:theta_star}
Let $v^\star$ and $g^\star$ be as in Theorem~\ref{thm:wellposed}. Define the peripheral coordinates
\begin{equation}\label{eq:theta_star}
\theta^\star_{i,k} := g^\star(a_{i,k}),\qquad (i,k)\in\mathcal I.
\end{equation}
Then $g^\star(\cdot)=\sum_{(i,k)\in\mathcal I}\theta^\star_{i,k} b_{i,k}(\cdot)$ by Lemma~\ref{lem:anchor_interpolation}.
\end{definition}

\begin{theorem}[Formal version of Theorem \ref{thm:main_g_rate}]\label{thm:theta_rate}
Fix anchors $\{a_{i,k}\}$ and basis functions $\{b_{i,k}\}$ from Definition~\ref{def:main_b}.
Let $v^\star$ and $\theta^\star$ be as in Definition~\ref{def:theta_star}.
Let $v\in\mathbb R^n$ be any possibly random gauge-fixed vector with $v(a_{i,k})=0$ for all $(i,k)\in\mathcal I$.

Run Algorithm~\ref{alg:main_g} with $J$ samples per anchor.
Define $B_v^2:=\mathbb E\|v\|_\infty^2$ and $N:=|\mathcal I|$. Then
\begin{equation}
\mathbb E\Big[\max_{(i,k)\in\mathcal I}|\hat\theta_{i,k}-\theta^\star_{i,k}|\Big]
\le
\mathbb E\|v-v^\star\|_\infty
+
\sqrt{\frac{2B_v^2\log(2N)}{J}}.
\end{equation}
Moreover, if \eqref{eq:bhateps} holds with parameter $\varepsilon_b$, then
\begin{equation}\label{eq:g_recon_bound_tight}
\mathbb E\|\hat g-g^\star\|_\infty
\le
\mathbb E\Big[\max_{(i,k)\in\mathcal I}|\hat\theta_{i,k}-\theta^\star_{i,k}|\Big]
+
\|g^\star\|_\infty\,\varepsilon_b.
\end{equation}

\end{theorem}

\begin{proof}
Fix $(i,k)$ and condition on $v$.
The estimator $\widehat{(Pv)}(a_{i,k})$ is the sample mean of $J$ i.i.d.\ draws of $v(Y)$ with $Y\sim P(a_{i,k},\cdot)$.
The centered mean $\widehat{(Pv)}(a_{i,k})-(Pv)(a_{i,k})$ is subgaussian with proxy variance at most $\|v\|_\infty^2/J$.
A maximal inequality for $N$ subgaussian variables gives
\begin{equation}
\mathbb E\Big[\max_{(i,k)}|\widehat{(Pv)}(a_{i,k})-(Pv)(a_{i,k})|\ \Big|\ v\Big]
\le
\|v\|_\infty\sqrt{\frac{2\log(2N)}{J}}.
\end{equation}
Unconditioning and applying Cauchy--Schwarz yields
\begin{equation}
\mathbb E\Big[\max_{(i,k)}|\widehat{(Pv)}(a_{i,k})-(Pv)(a_{i,k})|\Big]
\le
\sqrt{\frac{2B_v^2\log(2N)}{J}}.
\end{equation}
For the bias term, $(Pv)(a)-(Pv^\star)(a)\le \|v-v^\star\|_\infty$ since $P$ is row-stochastic.
Combining gives the first inequality.

For the residual, write $\hat g-g^\star
=
\sum_{(i,k)}(\hat\theta_{i,k}-\theta^\star_{i,k})\bhat_{i,k}
+
\sum_{(i,k)}\theta^\star_{i,k}(\bhat_{i,k}-b_{i,k})$. For the residual, fix $s\in S$ and write
\[
(\hat g-g^\star)(s)
=
\sum_{(i,k)}(\hat\theta_{i,k}-\theta^\star_{i,k})\,\bhat_{i,k}(s)
+
\sum_{(i,k)}\theta^\star_{i,k}\bigl(\bhat_{i,k}(s)-b_{i,k}(s)\bigr).
\]
By Lemma~\ref{lem:b_simplex}, $\bhat_{i,k}(s)\ge 0$ and $\sum_{(i,k)}\bhat_{i,k}(s)=1$, hence
\[
\left|\sum_{(i,k)}(\hat\theta_{i,k}-\theta^\star_{i,k})\,\bhat_{i,k}(s)\right|
\le
\max_{(i,k)}|\hat\theta_{i,k}-\theta^\star_{i,k}|.
\]
For the second term, use $|\theta^\star_{i,k}|\le\|g^\star\|_\infty$ and \eqref{eq:bhateps} to obtain
\[
\sum_{(i,k)}|\theta^\star_{i,k}|\cdot |\bhat_{i,k}(s)-b_{i,k}(s)|
\le
\|g^\star\|_\infty \sum_{(i,k)}|\bhat_{i,k}(s)-b_{i,k}(s)|
=
\|g^\star\|_\infty \,\|\bhat(s)-b(s)\|_1
\le
\|g^\star\|_\infty\,\varepsilon_b.
\]
Taking $\max_s$ and expectations yields \eqref{eq:g_recon_bound_tight}.
\end{proof}

\section{Numerical Experiments}\label{sec:experiments}

The experiments are designed to validate the central attribution claim of the paper: the decomposition separates persistent regime behavior from transient-to-regime cost.  We therefore do not use average-reward error as the only metric.  Instead, we report errors for the three quantities that appear in the theory: the persistent profile $g_\Pi^\star$, the anchor-gauge transient component $v_\Pi^\star$, and the finite-horizon return $J_H$.  The experiments are synthetic by design.  They isolate exactly the reducible and periodic pathologies studied in the paper, allow exact ground truth for all decomposition components, and permit component-level ablations of the estimator.  Throughout, the estimator has access only to a tabular generative model for transitions.  The exact transition matrix is used only to compute ground truth and plug-in baselines.

Table~\ref{tab:numexp_design_map} summarizes which question each experiment is intended to answer.  The table is included to make clear that the numerical section is not a generic benchmark suite: each experiment targets one claim made in the theory.

\begin{table}[t]
\centering
\small
\caption{Experiment design map.  Each row targets a specific claim or possible objection.}
\label{tab:numexp_design_map}
\begin{tabular}{@{}p{0.30\linewidth}p{0.35\linewidth}p{0.28\linewidth}@{}}
\toprule
Question tested & Experimental object & Evidence reported \\
\midrule
Can classical bias be large without transient cost? & Deterministic periodic cycle & $\|h\|_\infty=3$ while $\|v_\Pi^\star\|_\infty=0$ \\
Can two starts have the same eventual regime but different finite-horizon behavior? & Two deterministic paths into the same recurrent anchor & Same $g_\Pi^\star$, different $v_\Pi^\star$ and $J_H$ curves \\
Can average reward hide persistent phase structure? & Two cycles with equal Ces\`aro average & Same $\rho$, different $g_\Pi^\star$ \\
Does the refinement disappear when no periodic persistence exists? & Aperiodic reduction check & $\|g_\Pi^\star-\rho\|_\infty$ at numerical roundoff \\
Can the estimator learn the decomposition from samples? & Periodic two-class MRP with transient branching & Errors in $g_\Pi^\star$, $v_\Pi^\star$, and $J_{40}$ \\
Why not only estimate the classical gain? & Oracle avg-only comparator & Avg-only has no statistical error in $\rho$ but large diagnostic error \\
Why not estimate the full transition matrix? & Dense plug-in baseline & Plug-in is competitive; ours is compared honestly and not claimed to dominate uniformly \\
Which pipeline component causes residual error? & Structure, weight, projection ablation & Exact absorption weights make profile error numerical \\
How fragile is structural recovery? & Single-edge deletion check & $0/582$ deletions change structural invariants in this benchmark \\
\bottomrule
\end{tabular}
\end{table}

\subsection{Experiment design and metrics}

We evaluate finite Markov reward processes $(P,r)$, which are the policy-induced objects obtained after fixing a stationary policy in an MDP.  Rewards are deterministic and bounded, and a simulator query at state $s$ returns an independent sample from $P(s,\cdot)$.  The main metrics are
\[
E_g = \|\widehat g-g_\Pi^\star\|_\infty,\qquad
E_v = \|\widehat v-v_\Pi^\star\|_\infty,\qquad
E_J(H) = \|\widehat J_H-J_H\|_\infty,
\]
where
\[
\widehat J_H = \sum_{t=0}^{H-1}P^t\widehat g+\widehat v-P^H\widehat v.
\]
These metrics match Corollary~\ref{cor:return_error}: persistent-profile error accumulates over the horizon, while transient-component error enters through boundary terms.  We use $H=40$ in the sample-based benchmark.

We compare three estimators.  \textbf{Ours} learns the recurrent classes, cyclic phases, anchors, and phase-offset absorption weights, then runs projected stochastic approximation in the learned anchor gauge and reconstructs $\widehat g$ from anchor residuals.  \textbf{Avg-only} is an oracle classical-gain comparator: it is given the exact phase-averaged profile $\rho=P^\infty r$ and predicts returns using only $H\rho$.  This comparator is intentionally strong statistically, but it cannot represent non-invariant persistent phase behavior or transient cost. It is not a TD baseline: a full gain--bias estimator retains the bias and can be postprocessed using the same gauge. \textbf{Plug-in} estimates the full transition matrix $\widehat P$ from samples, recovers the structure of $\widehat P$, and computes the same decomposition from the empirical model.  Plug-in is a natural and strong tabular baseline.  Our claim is not uniform numerical dominance over plug-in solvers on small tabular models; the claim is that the proposed estimator targets the persistent-transient decomposition directly and makes the sources of evaluation error interpretable.

\subsection{Exact diagnostic checks}

We first run four exact diagnostic checks, shown in Figure~\ref{fig:numexp_diagnostics} and summarized in Table~\ref{tab:numexp_diagnostics}.  These experiments use exact transition matrices and are meant to test whether the decomposition has the intended semantics.

For reproducibility, the four exact checks use the following constructions.  In panel~(a), $S=\{0,\ldots,23\}$, $P(k,k+1\bmod 24)=1$, and $r(k)=1\{k<12\}$.  Since the whole state space is recurrent and periodic, $\mathcal K(P)=\mathbb R^{24}$ and the anchor-gauge projection is the zero map.  In panel~(b), a single absorbing recurrent anchor $a$ has reward $r(a)=1$, while two deterministic transient chains of lengths $3$ and $14$ have reward zero and then enter $a$.  A state at distance $\ell$ from $a$ has the same persistent profile $g_\Pi^\star=1$ but transient component $v_\Pi^\star=-\ell$.  In panel~(c), two deterministic $4$-cycles have reward profiles $r_A=(1,1,0,0)$ and $r_B=(1,0,1,0)$.  Both have classical gain $1/2$, but their persistent profiles differ in phase.  In panel~(d), we compare an aperiodic chain, where Proposition~\ref{prop:aperiodic_reduction} implies $\mathcal K(P)=\ker(I-P)$ and therefore no non-invariant persistent residual remains, with a periodic chain whose phase profile differs from its invariant average.

\paragraph{Pure periodic chain.}
The first example is a deterministic $24$-cycle with reward one on the first half of the cycle and zero on the second half.  Since the process starts already inside its recurrent regime, there is no transient-to-regime path.  The proposed decomposition gives $v_\Pi^\star=0$ and $g_\Pi^\star=r$.  In contrast, the classical gain is the constant average $\rho=1/2$, and the classical normalized bias is nonzero because it must absorb the phase pattern discarded by $\rho$.  This confirms the motivation behind Proposition~\ref{prop:deterministic_cycle}: a large classical bias need not mean a large transient burden.

\paragraph{Same eventual regime, different transient burden.}
The second example compares two starts that eventually enter the same unit-reward regime but after different deterministic path lengths.  The persistent profile is identical, $g_\Pi^\star=1$, but the anchor-gauge transient costs differ: $v_\Pi^\star=-3$ versus $v_\Pi^\star=-14$.  The finite-horizon return curves differ accordingly.  This is the attribution behavior the paper is designed to recover: the eventual regime is good in both cases, but one path pays a much larger transient cost before reaching it.

\paragraph{Same average reward, different persistent profile.}
The third example compares two periodic profiles with the same Ces\`aro average reward.  The classical gain cannot distinguish them, so $\|\rho_A-\rho_B\|_\infty=0$.  The persistent profiles differ by $\|g_A^\star-g_B^\star\|_\infty=1$, however, because the phase-resolved regimes are different.  This verifies that average reward alone is too coarse when phase behavior persists forever.

\paragraph{Aperiodic reduction.}
The fourth example checks that the proposed refinement disappears when it should.  In an aperiodic chain, the non-invariant persistent part is numerically zero: $\|g_\Pi^\star-\rho\|_\infty=5.6\times 10^{-16}$.  In a periodic comparison chain the corresponding gap is $0.5$.  Thus the decomposition reduces to the classical picture in aperiodic models and refines it only when non-invariant peripheral behavior exists.

\begin{figure}[t]
    \centering
    \includegraphics[width=\textwidth]{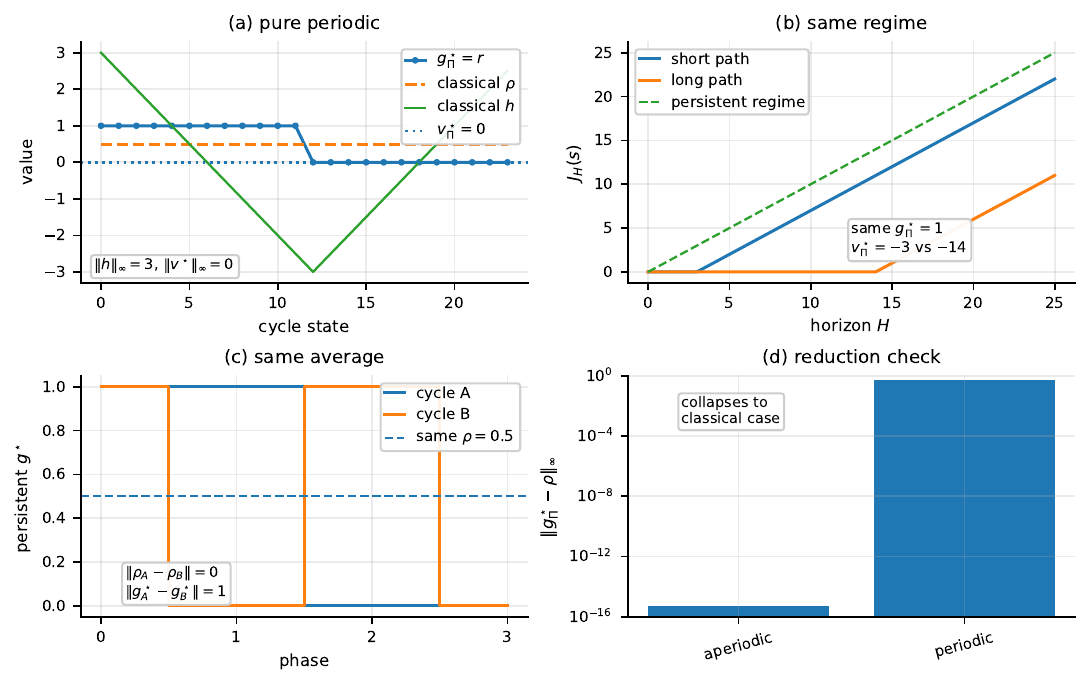}
    \caption{Exact diagnostic checks for the persistent-transient decomposition.  Panel~(a) shows that a deterministic periodic chain has nonzero classical bias even though the anchor-gauge transient component is zero.  Panel~(b) separates identical eventual regimes from different transient path costs.  Panel~(c) shows that equal average reward can hide different persistent phase profiles.  Panel~(d) verifies collapse to the classical picture in the aperiodic case.}
    \label{fig:numexp_diagnostics}
\end{figure}

\begin{table}[t]
\centering
\small
\caption{Exact diagnostic checks.  The numbers isolate the attribution issue rather than benchmarking statistical efficiency.}
\label{tab:numexp_diagnostics}
\begin{tabular}{@{}p{0.43\linewidth}p{0.50\linewidth}@{}}
\toprule
Diagnostic claim & Observed values \\
\midrule
Pure periodic: classical bias need not be transient & $\|h\|_\infty=3$, $\|v_\Pi^\star\|_\infty=0$ \\
Same regime, different path cost & $g_\Pi^\star=1$ for both starts, $v_\Pi^\star=-3$ versus $-14$ \\
Same average, different persistent profile & $\|\rho_A-\rho_B\|_\infty=0$, $\|g_A^\star-g_B^\star\|_\infty=1$ \\
Aperiodic reduction & $\|g_\Pi^\star-\rho\|_\infty=5.6\times 10^{-16}$ aperiodic versus $0.5$ periodic \\
\bottomrule
\end{tabular}
\end{table}

\subsection{Sample-based decomposition estimation}

We next test the learned estimator on a periodic two-class MRP with transient branching.  The benchmark has two recurrent classes with periods $d_1=2$ and $d_2=3$, phase sizes $m_1=10$ and $m_2=9$, and a transient line of length $L=35$, giving $n=82$ states and $N=d_1+d_2=5$ persistent coordinates.  The recurrent phase rewards are
\[
(0.05,0.95) \quad \text{and} \quad (0.10,0.55,0.95),
\]
while transient rewards are zero.  From transient state $t_j$, the chain either self-loops, moves forward along the transient line, or exits to one of the two recurrent classes.  We use self-loop parameter $\epsilon=0.20$, exit mass $\eta=0.08$, and a linearly varying class-one exit probability $q_j\in[0.15,0.85]$.  This construction creates both periodic phase structure and state-dependent absorption into recurrent classes.

For the learned estimator, support recovery uses $K=180$ transition samples per state.  Phase-offset absorption weights use $M=900$ absorption episodes per transient state.  Projected stochastic approximation runs for $2600$ synchronous iterations, drawing one transition sample per state at each iteration, with stepsize $\alpha_t=1.5(t+80)^{-0.72}$.  Anchor residuals are estimated with $J=100$ samples per anchor.  Results are averaged over five random seeds.  Figure~\ref{fig:numexp_sample_curves} plots the three errors as a function of stochastic-approximation transition samples; the support and absorption-weight budgets are fixed for the curve.  The shaded bands are approximate $95\%$ confidence intervals over seeds.

\begin{figure}[t]
    \centering
    \includegraphics[width=\textwidth]{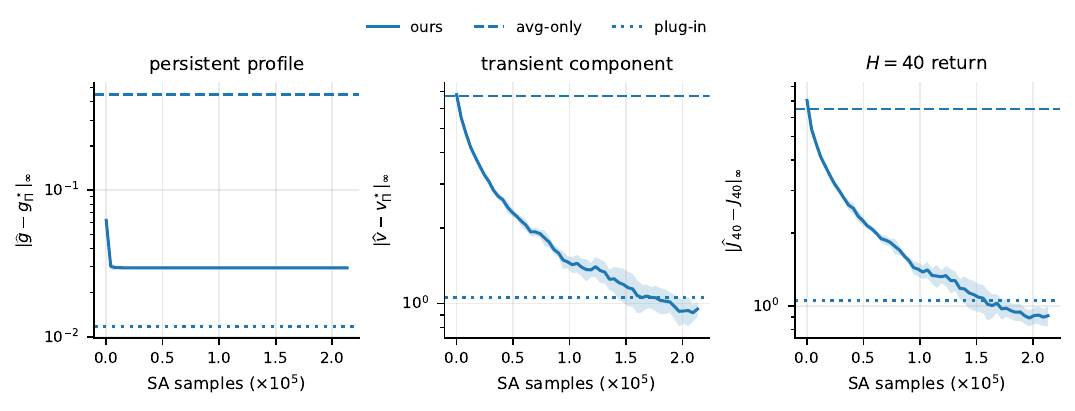}
    \caption{Sample-based estimation of the persistent-transient decomposition on a periodic multichain MRP.  The average-only comparator is an oracle phase-averaged gain profile and therefore has no statistical error in estimating $\rho$, but it discards phase-resolved persistence and transient cost.  The plug-in comparator estimates the full transition matrix and computes the same decomposition from the empirical model.  Ours directly estimates the decomposition and substantially improves over the average-only comparator on persistent profile, transient component, and $H=40$ return prediction.}
    \label{fig:numexp_sample_curves}
\end{figure}

\begin{table}[t]
\centering
\small
\caption{Final sample-based errors on the periodic multichain benchmark.  Entries are mean $\pm$ one standard deviation over five seeds.  The average-only comparator is an oracle classical-gain profile; the plug-in comparator estimates the full transition matrix before computing the same decomposition from the empirical model.}
\label{tab:numexp_sample_metrics}
\begin{tabular}{@{}p{0.47\linewidth}ccc@{}}
\toprule
Metric & Ours & Avg-only & Plug-in \\
\midrule
Persistent profile $\|\widehat g-g_\Pi^\star\|_\infty$ & $0.0295\pm0.0009$ & $0.45$ & $0.0118\pm0.0005$ \\
Transient component $\|\widehat v-v_\Pi^\star\|_\infty$ & $0.951\pm0.08$ & $6.73$ & $1.05\pm0.32$ \\
Return prediction $\|\widehat J_{40}-J_{40}\|_\infty$ & $0.911\pm0.10$ & $6.48$ & $1.05\pm0.32$ \\
\bottomrule
\end{tabular}
\end{table}

Table~\ref{tab:numexp_sample_metrics} gives the final errors.  Relative to the oracle average-only comparator, the learned quotient estimator reduces persistent-profile error by about $93\%$, transient-component error by about $86\%$, and $H=40$ return-prediction error by about $86\%$.  This gap is not caused by a weak statistical baseline: the average-only comparator is given the exact invariant gain.  Its failure is semantic.  It targets the wrong object for diagnostic evaluation because it removes phase-resolved persistent behavior and has no transient component.  The plug-in baseline is strongest on persistent-profile reconstruction in this small tabular instance, while ours is comparable or slightly better on transient and return errors.  This is consistent with the paper's positioning: plug-in is a strong model-based comparator, whereas the proposed method estimates the decomposition through its structural components and reports errors that correspond directly to the theory.

\subsection{Ablation and support-structure robustness}

The estimator has three conceptually separate stages: structural recovery, phase-offset absorption-weight estimation, and projected stochastic approximation.  Table~\ref{tab:numexp_ablation} ablates these stages on the same benchmark.  The ablation uses $M=500$ absorption episodes, $1200$ stochastic-approximation iterations, $J=80$ anchor-residual samples, and three random seeds.  The \emph{Full learned} row uses learned structure and learned absorption weights.  The \emph{True structure + MC weights} row gives the estimator the true recurrent classes and periods but still estimates absorption weights by Monte Carlo.  The \emph{Learned structure + exact weights} row uses learned graph structure but exact absorption weights.  The \emph{Oracle projection} row uses the exact anchor projection.

\begin{table}[t]
\centering
\small
\setlength{\tabcolsep}{3.5pt}
\caption{Component ablation on the sample-based benchmark.  Exact absorption weights make persistent-profile error essentially zero, indicating that the residual profile error is due to Monte Carlo phase-offset weight estimation rather than the quotient construction.  Entries are mean $\pm$ one standard deviation over three seeds.}
\label{tab:numexp_ablation}
\begin{tabular}{@{}p{0.34\linewidth}ccc@{}}
\toprule
Variant & $\|\widehat g-g_\Pi^\star\|_\infty$ & $\|\widehat v-v_\Pi^\star\|_\infty$ & $\|\widehat J_{40}-J_{40}\|_\infty$ \\
\midrule
Full learned & $0.0446\pm0.002$ & $1.45\pm0.13$ & $1.44\pm0.11$ \\
True structure + MC weights & $0.040\pm0.004$ & $1.48\pm0.15$ & $1.44\pm0.12$ \\
Learned structure + exact weights & $3.15\times10^{-11}\pm9.6\times10^{-12}$ & $1.47\pm0.11$ & $1.43\pm0.11$ \\
Oracle projection & $3.15\times10^{-11}\pm9.6\times10^{-12}$ & $1.47\pm0.11$ & $1.43\pm0.11$ \\
\bottomrule
\end{tabular}
\end{table}

The ablation identifies the dominant error source.  Replacing learned structure by true structure changes persistent-profile error only slightly, from $0.0446$ to $0.040$.  In contrast, replacing Monte Carlo absorption weights by exact weights drives persistent-profile error to numerical precision.  Thus the quotient construction and structural recovery are not the bottleneck on this benchmark; the remaining profile error is the expected finite-sample error in phase-offset absorption probabilities.  The transient and finite-horizon errors are similar across the exact-weight rows because, at this budget, they are dominated by stochastic approximation error in $\widehat v$.

We also tested a concrete support-structure perturbation.  In an exhaustive single-edge deletion check over the benchmark support graph, removing one true support edge changed the recovered structural invariants in $0/582$ cases.  Equivalently, the recurrent-class count, recurrent-set size, and periods were preserved for every single-edge deletion in this benchmark.  This should not be read as a universal robustness theorem: the theoretical guarantee still requires recovering the correct support graph, and rare edges can be information-theoretically hard to detect.  The result is instead a benchmark-level sanity check showing that the structural invariants used here are not fragile to every single missed edge.

\subsection{Storage proxy for the plug-in comparison}

A dense plug-in estimator stores an $n\times n$ transition matrix and then solves the decomposition from the learned model.  Our estimator instead stores the observed support and the learned anchor-basis weights needed by the quotient evaluator.  Table~\ref{tab:numexp_storage} reports a simple storage proxy for larger instances with the persistent dimension $N=5$.  The ratio is not a runtime theorem or a statistical optimality claim; it illustrates why the algorithm is structure-aware rather than full-model based.

\begin{table}[t]
\centering
\small
\caption{Sparse storage proxy for direct quotient estimation versus dense plug-in model storage.  Here $N=5$ is the number of persistent coordinates.  The dense plug-in proxy is $n^2$ entries, while the quotient proxy counts observed support plus $nN$ basis-weight entries.}
\label{tab:numexp_storage}
\begin{tabular}{@{}rrrr@{}}
\toprule
$n$ & $N$ & dense $P$ entries & dense / quotient proxy \\
\midrule
$290$ & $5$ & $84{,}100$ & $28.9\times$ \\
$1040$ & $5$ & $1{,}081{,}600$ & $112.0\times$ \\
$1540$ & $5$ & $2{,}371{,}600$ & $167.5\times$ \\
\bottomrule
\end{tabular}
\end{table}

The storage proxy complements the statistical comparison.  On small tabular problems, plug-in can be very accurate and should be treated as a strong baseline.  The proposed method is valuable for a different reason: it estimates the persistent coordinates, the transient representative, and the finite-horizon return diagnostic through the quotient structure, without making dense full-model estimation the central object.

\subsection{Implementation details and reproducibility}

All experiments use deterministic rewards and independent generative-model transition samples.  Exact diagnostic checks are deterministic and use exact transition matrices.  In the sample-based benchmark, each reported seed relearns the support structure, estimates the phase-offset absorption basis, runs projected stochastic approximation, and estimates anchor residuals.  Ground truth is computed only for evaluation by first constructing the exact support graph, computing the exact phase-offset basis, solving the gauge-fixed linear system for $v_\Pi^\star$, and then forming $g_\Pi^\star=r+Pv_\Pi^\star-v_\Pi^\star$.  The plug-in baseline estimates a dense empirical transition matrix from simulator samples and then applies the same exact decomposition routine to the empirical model.  Thus plug-in is given the natural full-model route, while the proposed method follows the quotient estimator analyzed in the paper.

The main numerical comparison should be read in this light.  The average-only comparator is deliberately favorable to classical average-reward evaluation because it is given the exact invariant gain; it still fails on $g_\Pi^\star$, $v_\Pi^\star$, and $J_{40}$ because those are not invariant-gain objects.  The dense plug-in comparator is deliberately strong on a small tabular problem and can have smaller persistent-profile error.  This is not a contradiction of the paper's claim.  The claim is that the peripheral quotient gives a diagnostic decomposition and a stable estimator for that decomposition, not that every structure-aware estimator must numerically dominate dense model estimation at every tabular sample budget.

\subsection{Takeaway}

The exact diagnostics validate the semantic claim: classical gain and bias can entangle persistent phase behavior with transient cost, while the proposed decomposition separates them.  The sample-based benchmark validates the statistical claim: the estimator can learn the two decomposition components and thereby control finite-horizon return prediction.  The plug-in comparison addresses the natural model-based alternative, and the ablations show that the quotient construction itself is not the source of the remaining profile error.  Overall, the experiments support the paper's main message: the contribution is not merely a way to compute returns, but a way to attribute them to persistent regime behavior and transient-to-regime cost.

\section{Interpretation in cyclic systems}\label{app:cyclic_interpretation}

A finite model of a cyclic service system can use $X_t=(Z_t,\phi_t)$, where $Z_t$ contains bounded queues and protocol variables and $\phi_t\in\{0,\ldots,d-1\}$ is schedule phase, with $d\ge2$. For phase-conditioned kernels $Q_\phi$, the homogeneous transition matrix is
\begin{equation}
P\bigl((z,\phi),(z',\phi')\bigr)
=
Q_\phi(z,z')\mathbf 1\{\phi'=\phi+1\pmod d\}.
\end{equation}
Every return time is divisible by $d$, so stochastic operational transitions do not eliminate periodicity. Initial phase determines alignment with recurring service; the operational state can also produce relaxing initialization effects. Multiple closed configurations arise only when the fixed controller does not subsequently switch between them. A calibrated finite Markov abstraction and suitable sampling access are needed before applying our estimator to a real system.

For a simple example, let $Z_t$ be cold or warm. A cold system becomes warm with probability $p\in(0,1]$ per step; warm states remain warm. Let $r(b,\phi)=a_\phi-c\mathbf1\{b=\mathrm{cold}\}$, with $c>0$, and use warm states as anchors. Direct substitution gives $g_\Pi^\star(b,\phi)=a_\phi$ and $v_\Pi^\star(b,\phi)=-(c/p)\mathbf1\{b=\mathrm{cold}\}$. Every start has the same average reward, but the phase-aligned return difference is
\begin{equation}
J_T(\mathrm{cold},\phi)-J_T(\mathrm{warm},\phi)
=
-\frac{c}{p}\bigl(1-(1-p)^T\bigr).
\end{equation}
Thus $a_\phi$ records recurring service, while $c/p$ records startup burden. Changing the initial phase changes timing rather than startup cost. This is an illustrative model, not a measured result for the cited systems.

Gauge choice matters outside this simple example. Let $D$ be a common multiple of recurrent periods and $E=\lim_{m\to\infty}P^{mD}$, the commuting peripheral projection. The gauge $I-E$ gives $g_{\rm c}=Er$ and $v_{\rm c}=\sum_{t\ge0}P^t(I-E)r$, with $P^t v_{\rm c}\to0$. A general anchor pair instead satisfies $v_\Pi^\star=\Pi v_{\rm c}$ and $g_\Pi^\star=g_{\rm c}+(I-P)(I-\Pi)v_{\rm c}$. Its representative need not decay under the original dynamics. For example, with
\[
P=
\begin{pmatrix}
0&0&1\\
0&0&1\\
1/2&1/2&0
\end{pmatrix},
\qquad
r=(0,2,0)^\top,
\]
anchors $1,3$ give $v_\Pi^\star=(0,2,0)^\top$ and $g_\Pi^\star=(0,0,1)^\top$, whereas $Er=(1,1,0)^\top$; $P^t v_\Pi^\star$ oscillates. The anchor value nevertheless has a phase-aligned comparison interpretation:
\begin{equation}
v_\Pi^\star(s)
=
\lim_{T\to\infty}
\left[
J_T(s)
-
\sum_{(i,k)\in\mathcal I}
b_{i,k}(s)J_T(a_{i,k})
\right].
\end{equation}
Indeed, applying $\Pi$ to the return identity gives $\Pi J_T=v_\Pi^\star-\Pi P^T v_\Pi^\star\to v_\Pi^\star$ by quotient contraction. This separates anchor-relative initialization accounting from literal spectral decay without changing the estimator's target.

Finally, approximate periodicity does not imply exact persistence. For a deterministic cycle $C$, the perturbed chain $(1-\eta)C+\eta I$, $0<\eta<1$, has all nonconstant eigenvalues strictly inside the unit disk. Its oscillations decay, potentially slowly. Our guarantees concern exact finite-chain structure, not uniform recovery of approximately persistent modes or deployment-scale models. Similarly, damping a classical expected TD update changes its error operator to $(1-\alpha)I+\alpha P$, $0<\alpha<1$; every mode with $|\lambda|\le1$ and $\lambda\ne1$ then has $|1-\alpha+\alpha\lambda|<1$. Periodic phase modes therefore do not by themselves imply TD divergence.

\newpage
\input{checklist}

\end{document}

%% file: checklist.tex
\section*{NeurIPS Paper Checklist}

The checklist is designed to encourage best practices for responsible machine learning research, addressing issues of reproducibility, transparency, research ethics, and societal impact. Do not remove the checklist: {\bf The papers not including the checklist will be desk rejected.} The checklist should follow the references and follow the (optional) supplemental material.  The checklist does NOT count towards the page
limit. 

Please read the checklist guidelines carefully for information on how to answer these questions. For each question in the checklist:
\begin{itemize}
    \item You should answer \answerYes{}, \answerNo{}, or \answerNA{}.
    \item \answerNA{} means either that the question is Not Applicable for that particular paper or the relevant information is Not Available.
    \item Please provide a short (1--2 sentence) justification right after your answer (even for \answerNA). 
\end{itemize}

{\bf The checklist answers are an integral part of your paper submission.} They are visible to the reviewers, area chairs, senior area chairs, and ethics reviewers. You will also be asked to include it (after eventual revisions) with the final version of your paper, and its final version will be published with the paper.

The reviewers of your paper will be asked to use the checklist as one of the factors in their evaluation. While \answerYes{} is generally preferable to \answerNo{}, it is perfectly acceptable to answer \answerNo{} provided a proper justification is given (e.g., error bars are not reported because it would be too computationally expensive'' or ``we were unable to find the license for the dataset we used''). In general, answering \answerNo{} or \answerNA{} is not grounds for rejection. While the questions are phrased in a binary way, we acknowledge that the true answer is often more nuanced, so please just use your best judgment and write a justification to elaborate. All supporting evidence can appear either in the main paper or the supplemental material, provided in appendix. If you answer \answerYes{} to a question, in the justification please point to the section(s) where related material for the question can be found.

IMPORTANT, please:
\begin{itemize}
    \item {\bf Delete this instruction block, but keep the section heading ``NeurIPS Paper Checklist"},
    \item  {\bf Keep the checklist subsection headings, questions/answers and guidelines below.}
    \item {\bf Do not modify the questions and only use the provided macros for your answers}.
\end{itemize}


\begin{enumerate}

\item {\bf Claims}
    \item[] Question: Do the main claims made in the abstract and introduction accurately reflect the paper's contributions and scope?
    \item[] Answer: \answerYes{} 
    \item[] Justification:  The main claims made in the abstract and introduction accurately reflect the paper's contributions and scope.
    \item[] Guidelines:
    \begin{itemize}
        \item The answer \answerNA{} means that the abstract and introduction do not include the claims made in the paper.
        \item The abstract and/or introduction should clearly state the claims made, including the contributions made in the paper and important assumptions and limitations. A \answerNo{} or \answerNA{} answer to this question will not be perceived well by the reviewers. 
        \item The claims made should match theoretical and experimental results, and reflect how much the results can be expected to generalize to other settings. 
        \item It is fine to include aspirational goals as motivation as long as it is clear that these goals are not attained by the paper. 
    \end{itemize}

\item {\bf Limitations}
    \item[] Question: Does the paper discuss the limitations of the work performed by the authors?
    \item[] Answer: \answerYes{} 
    \item[] Justification: The paper's approach have only been validated under tabular finite-state Markov chains. The finite-state setting has been constantly emphasized throughout this paper.
    \item[] Guidelines:
    \begin{itemize}
        \item The answer \answerNA{} means that the paper has no limitation while the answer \answerNo{} means that the paper has limitations, but those are not discussed in the paper. 
        \item The authors are encouraged to create a separate ``Limitations'' section in their paper.
        \item The paper should point out any strong assumptions and how robust the results are to violations of these assumptions (e.g., independence assumptions, noiseless settings, model well-specification, asymptotic approximations only holding locally). The authors should reflect on how these assumptions might be violated in practice and what the implications would be.
        \item The authors should reflect on the scope of the claims made, e.g., if the approach was only tested on a few datasets or with a few runs. In general, empirical results often depend on implicit assumptions, which should be articulated.
        \item The authors should reflect on the factors that influence the performance of the approach. For example, a facial recognition algorithm may perform poorly when image resolution is low or images are taken in low lighting. Or a speech-to-text system might not be used reliably to provide closed captions for online lectures because it fails to handle technical jargon.
        \item The authors should discuss the computational efficiency of the proposed algorithms and how they scale with dataset size.
        \item If applicable, the authors should discuss possible limitations of their approach to address problems of privacy and fairness.
        \item While the authors might fear that complete honesty about limitations might be used by reviewers as grounds for rejection, a worse outcome might be that reviewers discover limitations that aren't acknowledged in the paper. The authors should use their best judgment and recognize that individual actions in favor of transparency play an important role in developing norms that preserve the integrity of the community. Reviewers will be specifically instructed to not penalize honesty concerning limitations.
    \end{itemize}

\item {\bf Theory assumptions and proofs}
    \item[] Question: For each theoretical result, does the paper provide the full set of assumptions and a complete (and correct) proof?
    \item[] Answer: \answerYes{} 
    \item[] Justification: The paper's approach have only been validated under tabular finite-state Markov chains. The finite-state setting has been constantly emphasized throughout this paper. For the sample-based learning approach, the generative model is also explicitly assumed.
    \item[] Guidelines:
    \begin{itemize}
        \item The answer \answerNA{} means that the paper does not include theoretical results. 
        \item All the theorems, formulas, and proofs in the paper should be numbered and cross-referenced.
        \item All assumptions should be clearly stated or referenced in the statement of any theorems.
        \item The proofs can either appear in the main paper or the supplemental material, but if they appear in the supplemental material, the authors are encouraged to provide a short proof sketch to provide intuition. 
        \item Inversely, any informal proof provided in the core of the paper should be complemented by formal proofs provided in appendix or supplemental material.
        \item Theorems and Lemmas that the proof relies upon should be properly referenced. 
    \end{itemize}

    \item {\bf Experimental result reproducibility}
    \item[] Question: Does the paper fully disclose all the information needed to reproduce the main experimental results of the paper to the extent that it affects the main claims and/or conclusions of the paper (regardless of whether the code and data are provided or not)?
    \item[] Answer: \answerYes{} 
    \item[] Justification: Although the primary contribution of this paper is theoretical, the paper's experiment section in the appendix covers sufficient details for reproducibility.
    \item[] Guidelines:
    \begin{itemize}
        \item The answer \answerNA{} means that the paper does not include experiments.
        \item If the paper includes experiments, a \answerNo{} answer to this question will not be perceived well by the reviewers: Making the paper reproducible is important, regardless of whether the code and data are provided or not.
        \item If the contribution is a dataset and\slash or model, the authors should describe the steps taken to make their results reproducible or verifiable. 
        \item Depending on the contribution, reproducibility can be accomplished in various ways. For example, if the contribution is a novel architecture, describing the architecture fully might suffice, or if the contribution is a specific model and empirical evaluation, it may be necessary to either make it possible for others to replicate the model with the same dataset, or provide access to the model. In general. releasing code and data is often one good way to accomplish this, but reproducibility can also be provided via detailed instructions for how to replicate the results, access to a hosted model (e.g., in the case of a large language model), releasing of a model checkpoint, or other means that are appropriate to the research performed.
        \item While NeurIPS does not require releasing code, the conference does require all submissions to provide some reasonable avenue for reproducibility, which may depend on the nature of the contribution. For example
        \begin{enumerate}
            \item If the contribution is primarily a new algorithm, the paper should make it clear how to reproduce that algorithm.
            \item If the contribution is primarily a new model architecture, the paper should describe the architecture clearly and fully.
            \item If the contribution is a new model (e.g., a large language model), then there should either be a way to access this model for reproducing the results or a way to reproduce the model (e.g., with an open-source dataset or instructions for how to construct the dataset).
            \item We recognize that reproducibility may be tricky in some cases, in which case authors are welcome to describe the particular way they provide for reproducibility. In the case of closed-source models, it may be that access to the model is limited in some way (e.g., to registered users), but it should be possible for other researchers to have some path to reproducing or verifying the results.
        \end{enumerate}
    \end{itemize}

\item {\bf Open access to data and code}
    \item[] Question: Does the paper provide open access to the data and code, with sufficient instructions to faithfully reproduce the main experimental results, as described in supplemental material?
    \item[] Answer: \answerNo{} 
    \item[] Justification: This paper is primarily theoretical. The code and experiments only serve for numerical validations for the theoretical results. 
    \item[] Guidelines:
    \begin{itemize}
        \item The answer \answerNA{} means that paper does not include experiments requiring code.
        \item Please see the NeurIPS code and data submission guidelines (\url{https://neurips.cc/public/guides/CodeSubmissionPolicy}) for more details.
        \item While we encourage the release of code and data, we understand that this might not be possible, so \answerNo{} is an acceptable answer. Papers cannot be rejected simply for not including code, unless this is central to the contribution (e.g., for a new open-source benchmark).
        \item The instructions should contain the exact command and environment needed to run to reproduce the results. See the NeurIPS code and data submission guidelines (\url{https://neurips.cc/public/guides/CodeSubmissionPolicy}) for more details.
        \item The authors should provide instructions on data access and preparation, including how to access the raw data, preprocessed data, intermediate data, and generated data, etc.
        \item The authors should provide scripts to reproduce all experimental results for the new proposed method and baselines. If only a subset of experiments are reproducible, they should state which ones are omitted from the script and why.
        \item At submission time, to preserve anonymity, the authors should release anonymized versions (if applicable).
        \item Providing as much information as possible in supplemental material (appended to the paper) is recommended, but including URLs to data and code is permitted.
    \end{itemize}

\item {\bf Experimental setting/details}
    \item[] Question: Does the paper specify all the training and test details (e.g., data splits, hyperparameters, how they were chosen, type of optimizer) necessary to understand the results?
    \item[] Answer: \answerYes{} 
    \item[] Justification: Although the primary contribution of this paper is theoretical, the paper's experiment section in the appendix covers sufficient details for reproducibility.
    \item[] Guidelines:
    \begin{itemize}
        \item The answer \answerNA{} means that the paper does not include experiments.
        \item The experimental setting should be presented in the core of the paper to a level of detail that is necessary to appreciate the results and make sense of them.
        \item The full details can be provided either with the code, in appendix, or as supplemental material.
    \end{itemize}

\item {\bf Experiment statistical significance}
    \item[] Question: Does the paper report error bars suitably and correctly defined or other appropriate information about the statistical significance of the experiments?
    \item[] Answer: \answerYes{} 
    \item[] Justification: Although the primary contribution of this paper is theoretical, the paper's experiment section in the appendix covers sufficient details for reproducibility.
    \item[] Guidelines:
    \begin{itemize}
        \item The answer \answerNA{} means that the paper does not include experiments.
        \item The authors should answer \answerYes{} if the results are accompanied by error bars, confidence intervals, or statistical significance tests, at least for the experiments that support the main claims of the paper.
        \item The factors of variability that the error bars are capturing should be clearly stated (for example, train/test split, initialization, random drawing of some parameter, or overall run with given experimental conditions).
        \item The method for calculating the error bars should be explained (closed form formula, call to a library function, bootstrap, etc.)
        \item The assumptions made should be given (e.g., Normally distributed errors).
        \item It should be clear whether the error bar is the standard deviation or the standard error of the mean.
        \item It is OK to report 1-sigma error bars, but one should state it. The authors should preferably report a 2-sigma error bar than state that they have a 96\% CI, if the hypothesis of Normality of errors is not verified.
        \item For asymmetric distributions, the authors should be careful not to show in tables or figures symmetric error bars that would yield results that are out of range (e.g., negative error rates).
        \item If error bars are reported in tables or plots, the authors should explain in the text how they were calculated and reference the corresponding figures or tables in the text.
    \end{itemize}

\item {\bf Experiments compute resources}
    \item[] Question: For each experiment, does the paper provide sufficient information on the computer resources (type of compute workers, memory, time of execution) needed to reproduce the experiments?
    \item[] Answer: \answerYes{} 

    \item[] Justification: Although the primary contribution of this paper is theoretical, the paper's experiment section in the appendix covers sufficient details for reproducibility. The experiments are simple numerical validations which can be ran on simple laptop devices.
    \begin{itemize}
        \item The answer \answerNA{} means that the paper does not include experiments.
        \item The paper should indicate the type of compute workers CPU or GPU, internal cluster, or cloud provider, including relevant memory and storage.
        \item The paper should provide the amount of compute required for each of the individual experimental runs as well as estimate the total compute. 
        \item The paper should disclose whether the full research project required more compute than the experiments reported in the paper (e.g., preliminary or failed experiments that didn't make it into the paper). 
    \end{itemize}
    
\item {\bf Code of ethics}
    \item[] Question: Does the research conducted in the paper conform, in every respect, with the NeurIPS Code of Ethics \url{https://neurips.cc/public/EthicsGuidelines}?
    \item[] Answer:\answerYes{} 
    \item[] Justification: This paper is primarily theoretical.
    \item[] Guidelines:
    \begin{itemize}
        \item The answer \answerNA{} means that the authors have not reviewed the NeurIPS Code of Ethics.
        \item If the authors answer \answerNo, they should explain the special circumstances that require a deviation from the Code of Ethics.
        \item The authors should make sure to preserve anonymity (e.g., if there is a special consideration due to laws or regulations in their jurisdiction).
    \end{itemize}

\item {\bf Broader impacts}
    \item[] Question: Does the paper discuss both potential positive societal impacts and negative societal impacts of the work performed?
    \item[] Answer: \answerNA{} 
    \item[] Justification: This paper is primarily theoretical.
    \item[] Guidelines:
    \begin{itemize}
        \item The answer \answerNA{} means that there is no societal impact of the work performed.
        \item If the authors answer \answerNA{} or \answerNo, they should explain why their work has no societal impact or why the paper does not address societal impact.
        \item Examples of negative societal impacts include potential malicious or unintended uses (e.g., disinformation, generating fake profiles, surveillance), fairness considerations (e.g., deployment of technologies that could make decisions that unfairly impact specific groups), privacy considerations, and security considerations.
        \item The conference expects that many papers will be foundational research and not tied to particular applications, let alone deployments. However, if there is a direct path to any negative applications, the authors should point it out. For example, it is legitimate to point out that an improvement in the quality of generative models could be used to generate Deepfakes for disinformation. On the other hand, it is not needed to point out that a generic algorithm for optimizing neural networks could enable people to train models that generate Deepfakes faster.
        \item The authors should consider possible harms that could arise when the technology is being used as intended and functioning correctly, harms that could arise when the technology is being used as intended but gives incorrect results, and harms following from (intentional or unintentional) misuse of the technology.
        \item If there are negative societal impacts, the authors could also discuss possible mitigation strategies (e.g., gated release of models, providing defenses in addition to attacks, mechanisms for monitoring misuse, mechanisms to monitor how a system learns from feedback over time, improving the efficiency and accessibility of ML).
    \end{itemize}
    
\item {\bf Safeguards}
    \item[] Question: Does the paper describe safeguards that have been put in place for responsible release of data or models that have a high risk for misuse (e.g., pre-trained language models, image generators, or scraped datasets)?
    \item[] Answer: \answerNA{}{} 
    \item[] Justification: This paper is primarily theoretical.
    \item[] Guidelines:
    \begin{itemize}
        \item The answer \answerNA{} means that the paper poses no such risks.
        \item Released models that have a high risk for misuse or dual-use should be released with necessary safeguards to allow for controlled use of the model, for example by requiring that users adhere to usage guidelines or restrictions to access the model or implementing safety filters. 
        \item Datasets that have been scraped from the Internet could pose safety risks. The authors should describe how they avoided releasing unsafe images.
        \item We recognize that providing effective safeguards is challenging, and many papers do not require this, but we encourage authors to take this into account and make a best faith effort.
    \end{itemize}

\item {\bf Licenses for existing assets}
    \item[] Question: Are the creators or original owners of assets (e.g., code, data, models), used in the paper, properly credited and are the license and terms of use explicitly mentioned and properly respected?
    \item[] Answer: \answerNA{}{} 
    \item[] Justification: This paper is primarily theoretical.
    \item[] Guidelines:
    \begin{itemize}
        \item The answer \answerNA{} means that the paper does not use existing assets.
        \item The authors should cite the original paper that produced the code package or dataset.
        \item The authors should state which version of the asset is used and, if possible, include a URL.
        \item The name of the license (e.g., CC-BY 4.0) should be included for each asset.
        \item For scraped data from a particular source (e.g., website), the copyright and terms of service of that source should be provided.
        \item If assets are released, the license, copyright information, and terms of use in the package should be provided. For popular datasets, \url{paperswithcode.com/datasets} has curated licenses for some datasets. Their licensing guide can help determine the license of a dataset.
        \item For existing datasets that are re-packaged, both the original license and the license of the derived asset (if it has changed) should be provided.
        \item If this information is not available online, the authors are encouraged to reach out to the asset's creators.
    \end{itemize}

\item {\bf New assets}
    \item[] Question: Are new assets introduced in the paper well documented and is the documentation provided alongside the assets?
    \item[] Answer: \answerNA{}{} 
    \item[] Justification: This paper is primarily theoretical.
    \item[] Guidelines:
    \begin{itemize}
        \item The answer \answerNA{} means that the paper does not release new assets.
        \item Researchers should communicate the details of the dataset\slash code\slash model as part of their submissions via structured templates. This includes details about training, license, limitations, etc. 
        \item The paper should discuss whether and how consent was obtained from people whose asset is used.
        \item At submission time, remember to anonymize your assets (if applicable). You can either create an anonymized URL or include an anonymized zip file.
    \end{itemize}

\item {\bf Crowdsourcing and research with human subjects}
    \item[] Question: For crowdsourcing experiments and research with human subjects, does the paper include the full text of instructions given to participants and screenshots, if applicable, as well as details about compensation (if any)? 
    \item[] Answer: \answerNA{}{} 
    \item[] Justification: This paper is primarily theoretical.
    \item[] Guidelines:
    \begin{itemize}
        \item The answer \answerNA{} means that the paper does not involve crowdsourcing nor research with human subjects.
        \item Including this information in the supplemental material is fine, but if the main contribution of the paper involves human subjects, then as much detail as possible should be included in the main paper. 
        \item According to the NeurIPS Code of Ethics, workers involved in data collection, curation, or other labor should be paid at least the minimum wage in the country of the data collector. 
    \end{itemize}

\item {\bf Institutional review board (IRB) approvals or equivalent for research with human subjects}
    \item[] Question: Does the paper describe potential risks incurred by study participants, whether such risks were disclosed to the subjects, and whether Institutional Review Board (IRB) approvals (or an equivalent approval/review based on the requirements of your country or institution) were obtained?
    \item[] Answer: \answerNA{} 
    \item[] Justification: This paper is primarily theoretical.
    \item[] Guidelines:
    \begin{itemize}
        \item The answer \answerNA{} means that the paper does not involve crowdsourcing nor research with human subjects.
        \item Depending on the country in which research is conducted, IRB approval (or equivalent) may be required for any human subjects research. If you obtained IRB approval, you should clearly state this in the paper. 
        \item We recognize that the procedures for this may vary significantly between institutions and locations, and we expect authors to adhere to the NeurIPS Code of Ethics and the guidelines for their institution. 
        \item For initial submissions, do not include any information that would break anonymity (if applicable), such as the institution conducting the review.
    \end{itemize}

\item {\bf Declaration of LLM usage}
    \item[] Question: Does the paper describe the usage of LLMs if it is an important, original, or non-standard component of the core methods in this research? Note that if the LLM is used only for writing, editing, or formatting purposes and does \emph{not} impact the core methodology, scientific rigor, or originality of the research, declaration is not required.
    \item[] Answer: \answerNA{} 
    \item[] Justification: LLM is mainly used for polishing the draft.
    \item[] Guidelines:
    \begin{itemize}
        \item The answer \answerNA{} means that the core method development in this research does not involve LLMs as any important, original, or non-standard components.
        \item Please refer to our LLM policy in the NeurIPS handbook for what should or should not be described.
    \end{itemize}

\end{enumerate}